\documentclass[11pt]{article}

\usepackage[margin=1in]{geometry}
\usepackage[T1]{fontenc}
\usepackage[utf8]{inputenc}
\usepackage{lmodern}
\usepackage{microtype}
\usepackage{amsmath,amssymb,amsthm,mathtools,bm}
\usepackage{booktabs,tabularx,array,multirow}
\usepackage{siunitx}
\usepackage{enumitem}
\usepackage{graphicx}
\usepackage{float}
\usepackage{xcolor}
\usepackage{hyperref}
\usepackage[nameinlink,capitalise]{cleveref}
\usepackage{url}
\usepackage{tcolorbox}

\definecolor{navy}{RGB}{25,55,95}
\definecolor{lightblue}{RGB}{238,245,252}
\hypersetup{
  colorlinks=true,
  linkcolor=navy,
  citecolor=navy,
  urlcolor=navy,
  pdftitle={Resolution-Aware Experimental Design under Partial Identifiability}
}

\newtheorem{theorem}{Theorem}[section]
\newtheorem{proposition}[theorem]{Proposition}
\newtheorem{lemma}[theorem]{Lemma}
\newtheorem{corollary}[theorem]{Corollary}
\theoremstyle{definition}

\theoremstyle{remark}
\newtheorem{remark}[theorem]{Remark}

\newcommand{\E}{\mathbb{E}}
\newcommand{\Pp}{\mathbb{P}}
\newcommand{\R}{\mathbb{R}}
\newcommand{\ind}{\mathbf{1}}
\newcommand{\TV}{\operatorname{TV}}
\newcommand{\argmin}{\operatorname*{arg\,min}}
\newcommand{\argmax}{\operatorname*{arg\,max}}
\newcommand{\cE}{\mathcal{E}}
\newcommand{\cY}{\mathcal{Y}}
\newcommand{\dd}{\,\mathrm{d}}

\newcommand{\GwaeTailJ}{1.788292}
\newcommand{\GwaeAverageJ}{1.047271}
\newcommand{\GwaeTailMOneWorstFive}{0.013781}
\newcommand{\GwaeAverageMOneWorstFive}{0.433758}
\newcommand{\GwaeTailMTwoWorstFive}{0.002977}
\newcommand{\GwaeAverageMTwoWorstFive}{0.875844}

\newcommand{\MethaneReferenceCost}{0.061716}
\newcommand{\MethaneUcbPL}{0.047649}
\newcommand{\MethaneUcbLH}{0.037982}
\newcommand{\MethaneUcbMVK}{0.047738}
\newcommand{\MethaneFrontierRows}{%
0.00 & 1.86014 & 0.56993 & 0.06250 \\
0.01 & 2.20015 & 0.39993 & 0.00281 \\
0.02 & 1.90511 & 0.54744 & 0.01375 \\
0.05 & 1.75046 & 0.62477 & 0.02475 \\
0.10 & 1.72207 & 0.63897 & 0.02872 \\
0.20 & 1.68874 & 0.65563 & 0.03405 \\
1.00 & 1.56981 & 0.71510 & 0.04343 \\
}

\title{Resolution-Aware Experimental Design under Partial Identifiability}
\author{
Sofianos-Panagiotis Fotias\\
National Technical University of Athens
}
\date{}

\begin{document}
\maketitle

\begin{abstract}
Experimental design is commonly framed as choosing the experiment expected to provide the most information. Under partial identifiability however, persistent nuisance uncertainty can make the same observation carry different structural meanings. We introduce Resolution-Aware Experimental Design (RAED), which selects an experiment by the smallest expected nonempty structural candidate set achievable subject to false-exclusion control. We prove an exact cross-nuisance aliasing separation: an experiment can be preferred by structural and full-latent information gain, average classification, and nuisance-marginalized informativeness while having arbitrarily poorer valid structural resolution. RAED nevertheless preserves the expected ordering under a genuine composite Blackwell comparison. To make this criterion operational, we develop a learned score-based implementation with finite-sample nuisance-average and positive-tail calibration, and characterize a rare-tail sample-complexity obstruction. Under constrained sensing, two subsurface-flow benchmarks exhibit genuine RAED--expected-information-gain (EIG) experiment-selection disagreements, with the clearest and largest held-out resolution differences in WCA. In a fluvial benchmark, tail protection changes the selected physical experiment and replaces hard-region false exclusions primarily with explicit ambiguity. In a mechanistic methane-oxidation benchmark, a prospectively specified 5\% false-exclusion tolerance also yields a nontrivial finite-sample population guarantee for tail-sensitive nuisance risk, with 95\% joint confidence across all three structural families.
\end{abstract}

\section{Introduction}
\label{sec:introduction}

Experimental design asks which intervention, measurement, or control should be chosen before future data are observed. In Bayesian experimental design, a standard principle is to prefer experiments expected to change posterior beliefs about an unknown model or parameter as much as possible. Expected information gain formalizes this idea by measuring the expected divergence between prior and posterior beliefs, equivalently the mutual information between the unknown quantity and the future observation \cite{huan2024,rainforth2024}. This is natural when the scientific objective is posterior concentration. But in many model-discrimination problems, observations are also shaped by persistent nuisance variables whose values can substantially alter the response. In such settings, a more direct objective is to determine which structural explanations can be excluded while controlling the risk of excluding the truth. Suppose the structural target is

$$
M\in\{1,\ldots,K\},
$$

and experiment $e$ produces an observation $Y$. Rather than force a single structural label, the terminal decision may return a nonempty candidate set

$$
\Gamma_e(Y)\subseteq\{1,\ldots,K\}.
$$

A singleton corresponds to complete structural resolution, whereas a larger set records ambiguity that cannot be safely removed. Because the nuisance state $\theta$ is not observed at decision time, the same nuisance-blind rule $\Gamma_e$ must map observations to structural conclusions across all nuisance states covered by the validity requirement. The design objective is therefore to choose the experiment that supports the smallest valid candidate set.

Resolution-Aware Experimental Design (RAED) formalizes this idea. For each structural family $k$, the observation law depends on a persistent nuisance state $\theta\in\Theta_k$. Given a false-exclusion tolerance $\alpha$, RAED controls the probability of excluding the true structural family and evaluates experiment $e$ through the expected candidate-set size
\begin{equation}
J_e=\E|\Gamma_e(Y)|.
\label{eq}
\end{equation}
Validity can be required at every admissible nuisance state, on average across nuisance states, or over the worst $\delta$-fraction of the nuisance distribution, where $\delta\in(0,1]$ specifies the fraction being protected. Smaller values of $\delta$ concentrate the validity requirement on increasingly adverse nuisance regimes, while $\delta=1$ recovers nuisance-average validity. Together, the false-exclusion tolerance $\alpha$ and nuisance-tail level $\delta$ determine the trade-off between structural resolution and the strength of validity protection.

The distinction matters because nuisance averaging can conceal incompatible structural meanings of the same observation. An experiment may be highly informative about $M$ after nuisance averaging, or even identify $M$ perfectly at each fixed nuisance state, while no single nuisance-blind terminal rule can safely translate its observations into a precise structural conclusion. We call this phenomenon cross-nuisance structural aliasing. It is an experiment-level form of partial identifiability: the issue is not simply overlap between averaged predictive distributions, but whether a single structural decision rule can remain valid across the nuisance states that the experiment must tolerate.

Table~\ref{tab:intro-cyclic} gives a concrete four-family example. Let $M\in\{0,1,2,3\}$ denote the structural family and let the nuisance state be $Z\in\{0,1,2,3\}$. Consider an experiment $e_I$ whose observation is
$$
Y=M\oplus Z \pmod 4.
$$
At each fixed $Z$, the observation uniquely determines $M$. However, once $Z$ is unknown this identification breaks down. The same observed value can correspond to different structural families under different nuisance states.

\begin{table}[H]
\centering
\small
\setlength{\tabcolsep}{0.95em}
\renewcommand{\arraystretch}{1.18}
\begin{tabular}{c@{\hspace{0.75em}}|@{\hspace{0.75em}}cccc}
\toprule
$Z\backslash M$ & $0$ & $1$ & $2$ & $3$ \\
\midrule
$0$ & \textbf{0} & $1$ & $2$ & $3$ \\
$1$ & $1$ & $2$ & $3$ & \textbf{0} \\
$2$ & $2$ & $3$ & \textbf{0} & $1$ \\
$3$ & $3$ & \textbf{0} & $1$ & $2$ \\
\bottomrule
\end{tabular}

\vspace{0.35em}
\footnotesize
Entries are observations $Y=M\oplus Z\pmod 4$; bold entries mark the same observation $Y=0$.

\caption{\textbf{Cross-nuisance structural aliasing at $K=4$.}
Each row is a permutation of the structural label, so $e_I$ identifies $M$ perfectly at every fixed $Z$. Across nuisance states, the same observation $Y=0$ is compatible with every $M$. A single nuisance-blind terminal rule therefore cannot assign that observation a unique structural meaning. Theorem~\ref{thm:paper-aliasing} compares this information-favoured experiment with a nuisance-independent experiment $e_R$ that achieves singleton RAED resolution.}
\label{tab:intro-cyclic}
\end{table}

This perspective leads to three main contributions.
\begin{enumerate}[leftmargin=1.5em]
\item \textbf{A resolution-aware design objective.} RAED places a randomized nonempty structural candidate set at the end of the experiment and selects the experiment with the smallest expected candidate-set size subject to a declared false-exclusion tolerance. Its validity criterion can range from protection at every nuisance state to tail-sensitive and nuisance-average guarantees.

\item \textbf{A theory of when information and structural resolution disagree.} We identify cross-nuisance structural aliasing as a mechanism by which an experiment can appear highly informative while still failing to support a precise nuisance-blind structural decision. An exact cyclic construction shows that this disagreement can grow with the number of structural families. We also show that RAED preserves the expected ordering under genuine composite informativeness and that the separation persists beyond exact aliasing.

\item \textbf{A learned implementation with finite-sample calibration.} We develop score-based decision rules with independent calibration for nuisance-average and tail-sensitive validity, turning the population objective into a practical experiment-selection procedure. Across two subsurface-flow benchmarks, RAED and expected information gain select different experiments and produce materially different held-out structural resolution. In a fluvial benchmark, tail protection changes the selected well test and replaces hard-region false exclusions primarily with explicit ambiguity. A mechanistic methane-oxidation study further demonstrates nontrivial finite-sample population certification under a prespecified false-exclusion tolerance.

\end{enumerate}

The remainder of the paper relates RAED to neighboring approaches, develops the formal decision problem and aliasing theory, introduces the learned calibration procedure, and evaluates the resulting experiment choices and finite-sample guarantees.

\section{Related Work}
\label{sec:related}

\paragraph{Bayesian experimental design and model discrimination.}
Bayesian experimental design commonly selects $e$ by maximizing expected information gain about a latent quantity, including structural model indices and continuous parameters \cite{huan2024,rainforth2024}. Model-discrimination design instead uses criteria tailored to separating rival models, including T-optimality, robust discrimination, and classification-based Bayesian design \cite{atkinson1975,atkinson1975multi,dette2009,dette2013,hainy2022}. Expected model-elimination criteria similarly reward experiments that are expected to reject more competing models \cite{alberton2012}. RAED differs in the terminal decision it optimizes: after observing the experiment, it seeks the smallest nonempty structural candidate set that can be reported while controlling false exclusion across persistent nuisance uncertainty.

\paragraph{Set-valued classification and confidence sets.}
Minimum expected set size under coverage is a classical statistical objective. Least-ambiguous set-valued classification minimizes expected prediction-set size subject to classwise coverage constraints \cite{sadinle2019}; related work studies confidence sets and classwise prediction sets with controlled ambiguity \cite{schafer2009,denis2017,shi2024}. At the nuisance-average endpoint, fixed-experiment RAED reduces exactly to this problem when the validity and efficiency laws coincide and the same randomized decision class is used. RAED then places this terminal set-valued decision inside an outer experiment-selection problem and extends the validity requirement to persistent nuisance uncertainty.

\paragraph{Robust and ambiguity-aware design.}
Robust design methods protect experiment utilities against nuisance uncertainty, model misspecification, or distributional ambiguity. Robust expected information gain introduces ambiguity sets around the information objective \cite{go2022}, while maximin Bayesian design frames experiment choice as a game against an adversarial nature \cite{abdulsamad2026}. RAED instead applies robustness to the nuisance-specific false-exclusion loss
$$
L_{k,e}(\theta)=\Pp_{k,\theta}^{e}\{k\notin\Gamma_e(Y)\},
$$
and optimizes structural resolution subject to the resulting validity requirement. Positive values of $\delta$ interpolate between increasingly adverse nuisance tails and nuisance-average protection, while $\delta=0$ is treated separately as the pointwise-validity endpoint.

\paragraph{Comparison of statistical experiments.}
Blackwell's ordering formalizes when one experiment can be simulated from another through a common garbling, while Le Cam--Torgersen deficiency theory quantifies approximate comparisons \cite{blackwell1953,lecam1964,torgersen1991}. These notions provide a natural benchmark for RAED. If the same nuisance-independent garbling maps one experiment to another for every $(k,\theta)$, then the full nuisance-specific decision problem is preserved and RAED respects the resulting informativeness ordering. The separation studied here arises under weaker notions of informativeness, such as comparisons made only after nuisance averaging or through mappings whose interpretation changes with the nuisance state.

\paragraph{Compound channels, confusability, and list decoding.}
Persistent nuisance uncertainty also creates conceptual links to compound channels, symmetrizability, confusability, and list decoding \cite{blackwell1959compound,csiszar1988avc,shannon1956,elias1957}. These literatures study communication rates, adversarial channel states, zero-error confusability, and decoding lists. RAED addresses a different problem: one-step scientific experiment selection with controlled false exclusion and expected candidate-set size as the resolution objective. The common theme is that the same observation may support incompatible interpretations across nuisance states unless a single decoder can resolve them consistently.

\paragraph{Risk control and finite-sample calibration.}
The calibration layer builds on risk-controlling prediction, Learn-then-Test, optimized certainty equivalent (OCE) risk, and concentration bounds for bounded losses \cite{bates2021,angelopoulos2025,huang2026oce,rockafellar2000}. In RAED, these tools provide finite-sample guarantees for the nuisance risk of learned candidate-set rules without changing the underlying population validity criterion. The parameters $\alpha$ and $\delta$ specify the scientific false-exclusion requirement, while $\zeta$ is the allowed probability that the finite-sample certificate itself fails.

Taken together, these connections place RAED at the intersection of experimental design and set-valued decision making under nuisance uncertainty. Its distinguishing feature is the outer design objective: experiments are compared by the smallest expected nonempty structural candidate set they can support subject to a prescribed false-exclusion requirement. The next section formalizes this decision problem and shows how, for a fixed experiment at the nuisance-average endpoint, it reduces to classical set-valued classification under the appropriate conditions.

\section{Resolution-Aware Experimental Design}
\label{sec:formulation}

\subsection{Structural families, nuisance states, and experiments}

Let
\begin{equation}
K\ge2,\qquad 0\le\alpha<1,\qquad 0\le\delta\le1.
\label{eq:global-population-parameters}
\end{equation}

The structural target is
\[ M\in\{1,\ldots,K\}. \]

Conditional on $M=k$, the data also depend on a nuisance state $\theta\in\Theta_k$. The nuisance state is persistent in the sense that, for a realized experiment, the same value of $\theta$ governs the entire observation $Y$ rather than being redrawn across its components.

An experiment $e\in\cE$ induces, for every structural family $k$ and nuisance state $\theta\in\Theta_k$, an observation law
\[ P_{k,\theta}^{e} \]
on an observation space $\cY_e$. We assume these spaces are standard Borel. The experiment and its observation protocol are chosen before $Y$ is observed.

Two nuisance distributions may enter the formulation, and they serve different purposes. The measure $\mu_k$ specifies how nuisance states within family $k$ are weighted when imposing nuisance-average or nuisance-tail validity. The measure $\nu_k$ specifies how nuisance states are weighted when evaluating the expected size of the resulting candidate set. These measures may coincide, but they need not.

Let $\pi_k\ge0$, with $\sum_{k=1}^K\pi_k=1$, denote the structural weights used in the resolution objective. Define
\begin{equation}
P_{k,\nu}^{e}(A) = \int_{\Theta_k} P_{k,\theta}^{e}(A)\,\nu_k(\dd\theta), \qquad Q_e = \sum_{k=1}^K \pi_k P_{k,\nu}^{e}.
\label{eq:efficiency-law}
\end{equation}

Thus $\Theta_k$ specifies which nuisance states are admissible, $\mu_k$ determines how validity is assessed across them, $\nu_k$ determines how resolution is averaged across them, and $\pi_k$ determines how structural families are weighted in the resolution objective.

\subsection{Randomized nonempty candidate-set rules}

After observing $Y=y$, the terminal decision is a nonempty structural candidate set
\[ \Gamma_e(y)\subseteq\{1,\ldots,K\}, \qquad \Gamma_e(y)\neq\varnothing. \]
RAED allows this decision rule to be randomized. Formally, for each observation $y$, a Markov kernel
\[ \Pi_e(S\mid y) \]
assigns probabilities to the nonempty subsets $S\subseteq\{1,\ldots,K\}$. Deterministic candidate-set rules are included as the special case in which $\Pi_e(\cdot\mid y)$ places probability one on a single set.

For each structural family $k$, define the inclusion probability
\begin{equation}
g_{k,e}(y) := \Pr_{\Pi_e}\{k\in\Gamma_e(y)\} = \sum_{S\ni k}\Pi_e(S\mid y).
\label{eq:continuous-inclusion}
\end{equation}
Thus $g_{k,e}(y)$ is the probability, over the rule's internal randomization, that family $k$ is retained after observing $y$. Because the reported candidate set is always nonempty,
\begin{equation}
0\le g_{k,e}(y)\le1, \qquad \sum_{k=1}^K g_{k,e}(y)\ge1.
\label{eq:nonempty-marginals}
\end{equation}

For the expected-cardinality objective and the familywise false-exclusion losses considered below, these inclusion probabilities contain all information needed for optimization. Appendix~\ref{app:marginals} shows that every feasible collection of such marginals can be realized by a randomized rule over nonempty candidate sets.

For structural family $k$ and nuisance state $\theta$, define the retention probability
\begin{equation}
C_{k,e}(\theta;g) := \E_{Y\sim P_{k,\theta}^{e}} \left[g_{k,e}(Y)\right],
\label{eq:coverage}
\end{equation}
and the corresponding false-exclusion loss
\begin{equation}
L_{k,e}(\theta;g) := 1-C_{k,e}(\theta;g).
\label{eq:loss}
\end{equation}
Here $C_{k,e}(\theta;g)$ is the probability that the candidate set retains the true structural family when the data are generated under $(k,\theta)$, while $L_{k,e}(\theta;g)$ is the probability that the true family is excluded.

\subsection{Validity requirements}
\label{sec:validity}

RAED controls false exclusion separately within each structural family. At the pointwise endpoint, validity requires
\begin{equation}
\sup_{\theta\in\Theta_k} L_{k,e}(\theta;g)\le\alpha,\qquad k=1,\ldots,K.
\label{eq:uniform-validity}
\end{equation}
We use $\delta=0$ to denote this pointwise-validity endpoint. No nuisance probability measure is required in this case.

For $0<\delta\le1$, validity is defined through the upper-tail risk of the nuisance-specific false-exclusion loss under $\mu_k$:
\begin{equation}
\rho_{\delta,\mu_k}(L_{k,e}):=\inf_{\eta\in\mathbb R}\left\{\eta+\frac{1}{\delta}\int_{\Theta_k}\bigl(L_{k,e}(\theta;g)-\eta\bigr)_+\,\mu_k(\dd\theta)\right\},
\label{eq:tail-risk}
\end{equation}
where $(x)_+=\max\{x,0\}$. The corresponding Tail-validity requirement is
\begin{equation}
\rho_{\delta,\mu_k}(L_{k,e})\le\alpha,
\qquad k=1,\ldots,K.
\label{eq:tail-validity}
\end{equation}

For $0<\delta<1$, this criterion controls the average false-exclusion loss in the upper $\delta$-tail of the nuisance distribution. Smaller values of $\delta$ therefore place greater emphasis on adverse nuisance states. At $\delta=1$,
\[ \rho_{1,\mu_k}(L_{k,e})= \E_{\theta\sim\mu_k}L_{k,e}(\theta;g),\]
so Tail validity reduces to nuisance-average validity.

For later use, let $\mathcal G_e(\alpha,\delta)$ denote the set of randomized nonempty candidate-set rules satisfying the corresponding familywise validity requirement: \eqref{eq:uniform-validity} when $\delta=0$, and \eqref{eq:tail-validity} when $0<\delta\le1$.

\subsection{Resolution objective and experiment choice}

For experiment $e$ and candidate-set rule $g$, define the expected candidate-set size
\begin{equation}
J_e(g) := \E_{Y\sim Q_e} \left[ \E_{\Pi_e}\!\left( |\Gamma_e(Y)|\mid Y \right) \right] = \E_{Y\sim Q_e} \left[ \sum_{k=1}^K g_{k,e}(Y) \right].
\label{eq:objective}
\end{equation}
The inner expectation averages over any randomization of the terminal rule, while the outer expectation averages over the predictive observation distribution $Q_e$. Thus $J_e(g)$ is the expected number of structural families retained after experiment $e$.

For a fixed experiment, define its optimal valid ambiguity by
\begin{equation}
V_e^\star(\alpha,\delta) := \inf_{g\in\mathcal{G}_e(\alpha,\delta)} J_e(g),
\label{eq:fixed-e-raed}
\end{equation}
where $\mathcal{G}_e(\alpha,\delta)$ denotes the candidate-set rules satisfying the validity requirement defined above. Because every rule returns a nonempty set and the rule that always retains all $K$ families is feasible,
\[ 1\le V_e^\star(\alpha,\delta)\le K. \]
For convenience, we also define the normalized resolution
\begin{equation}
R_e^\star(\alpha,\delta) := \frac{K-V_e^\star(\alpha,\delta)}{K-1} \in[0,1].
\label{eq:normalized-resolution}
\end{equation}
Here $R_e^\star=1$ corresponds to singleton resolution and $R_e^\star=0$ to retaining all $K$ structural families. We use expected candidate-set size as the primary measure and normalized resolution when comparisons across problems with different numbers of structural families are useful.

RAED selects
\begin{equation}
e^\star(\alpha,\delta) \in \argmin_{e\in\cE} V_e^\star(\alpha,\delta).
\label{eq:raed-design}
\end{equation}
Thus experiments are compared by the smallest expected structural ambiguity they can achieve while satisfying the same declared validity requirement.

\subsection{Relationship to least-ambiguous set-valued classification}

At the nuisance-average endpoint, the fixed-experiment RAED problem has a direct connection to classical set-valued classification.

\begin{proposition}[Fixed-experiment nuisance-average reduction]
\label{prop:slw-reduction}
Fix an experiment $e$ and set $\delta=1$. Suppose that $\mu_k=\nu_k$ for every structural family, and define
\[ P_k:=P_{k,\mu}^{e}, \qquad Q_e:=\sum_{k=1}^K \pi_k P_k. \]
If mandatory nonemptiness is removed and the same randomized Markov-kernel decision class is used, then fixed-experiment RAED is exactly the randomized classwise least-ambiguous set-valued classification problem
\begin{equation}
\min_{\Gamma} \E_{Q_e}|\Gamma(Y)| \qquad \text{subject to} \qquad P_k\{k\in\Gamma(Y)\}\ge 1-\alpha, \quad k=1,\ldots,K.
\label{eq:slw-reduction}
\end{equation}
If an optimum of this randomized classical problem can be chosen nonempty $Q_e$-almost surely, then the same optimum value is attainable by RAED. Relative to formulations restricted to deterministic set-valued rules, RAED additionally permits randomization and requires the reported candidate set to be nonempty.
\end{proposition}

\begin{proof}
At $\delta=1$, the nuisance-tail validity criterion reduces to nuisance-average false-exclusion control:
\[\E_{\theta\sim\mu_k} L_{k,e}(\theta;g) \le \alpha. \]
Using the definition of $L_{k,e}$ and Fubini's theorem,
\begin{align}
\E_{\theta\sim\mu_k} L_{k,e}(\theta;g) &=
1- \E_{\theta\sim\mu_k} \E_{Y\sim P_{k,\theta}^{e}} g_{k,e}(Y)\\ &=
1- \E_{Y\sim P_k} g_{k,e}(Y).
\end{align}
Hence the RAED validity constraint is equivalent to
\[\E_{Y\sim P_k}g_{k,e}(Y)\ge1-\alpha, \]
which is precisely the classwise retention constraint
\[ P_k\{k\in\Gamma(Y)\}\ge1-\alpha \]
for the randomized candidate-set rule.

Likewise, for every observation $y$,
\[\E_{\Pi_e}\!\left[|\Gamma_e(y)|\right] = \sum_{k=1}^K g_{k,e}(y)\]
so the RAED objective is
\[J_e(g)=\E_{Y\sim Q_e}\left[\sum_{k=1}^K g_{k,e}(Y)\right]= \E_{Q_e}|\Gamma_e(Y)|.\]
Thus the objective and classwise coverage constraints coincide exactly with those of randomized least-ambiguous set-valued classification \cite{sadinle2019}. The only additional RAED restriction is mandatory nonemptiness. If a classical optimum is nonempty $Q_e$-almost surely, it can therefore be represented as a RAED rule with the same objective value.
\end{proof}

\subsection{Scientific validity and statistical certification}

RAED is defined by two scientific parameters: the false-exclusion tolerance $\alpha$ and the nuisance-tail level $\delta$. Together, they specify the validity requirement that a candidate-set rule must satisfy.

When this requirement is verified from finite data, a separate confidence parameter $\zeta\in(0,1)$ is introduced. It controls the allowed probability that the statistical certification procedure fails to establish the desired validity guarantee. Thus $\alpha$ and $\delta$ specify the scientific target, whereas $\zeta$ specifies the confidence with which that target is certified from finite samples.

This distinction becomes important in \Cref{sec:learned}. Calibration is a statistical procedure for certifying a learned rule against the RAED validity requirement; it does not change the validity requirement itself.

\section{Information and Structural Resolution under Nuisance Uncertainty}
\label{sec:aliasing}

Information about the structural label does not necessarily translate into a precise structural conclusion when nuisance uncertainty persists after the experiment. The central obstruction is that the same observation can acquire different structural meanings under different nuisance states.

\subsection{Cross-nuisance structural aliasing}
\label{sec:paper-cross-nuisance}

Cross-nuisance structural aliasing arises when an experiment distinguishes the structural families at each fixed nuisance state, but the mapping from observations to structural labels changes with that state. Since the terminal rule observes $Y$ but not the nuisance state, it must use a single mapping across all nuisance states covered by the validity requirement. The following construction makes this separation exact.

To construct the nuisance-independent reference experiment used below, we use the $K$-ary symmetric channel. For
\[q\in\left[0,\frac{K-1}{K}\right],\]
let $\mathsf{C}_K(q)$ denote the channel
\[\mathsf{C}_K(q)_{kj}=
\begin{cases}1-q, & j=k,\\\dfrac{q}{K-1}, & j\ne k. \end{cases}\]

Under a uniform structural prior, its mutual information is
\begin{equation}
\mathcal{I}_K(q)=\log K-h(q)-q\log(K-1),
\label{eq:paper-kary-mi}
\end{equation}
where
\[ h(q) = -q\log q-(1-q)\log(1-q) \]
is the binary entropy function. The function $\mathcal{I}_K(q)$ is strictly decreasing for
\[ 0\le q<\frac{K-1}{K}.\]
Fix $K\ge2$ and parameters
\[ 0<\varepsilon<\beta\le\alpha<1-\frac{1}{K}. \]
Let
\[ \mathbb Z_K=\{0,\ldots,K-1\}, \qquad M\sim\mathrm{Uniform}(\mathbb Z_K), \]
and let the persistent nuisance state $Z\in\mathbb Z_K$ be independent of $M$. The validity and efficiency nuisance distributions are taken to coincide, with
\[\mu(Z=0)=\nu(Z=0)=1-\varepsilon,\qquad\mu(Z=z)=\nu(Z=z)=\frac{\varepsilon}{K-1},\quad z\ne0.\]
We compare two experiments. Under the aliased experiment $e_I$,
\[Y=M\oplus Z \pmod K,\]
where $\oplus$ denotes addition modulo $K$. Thus the structural mapping is perfectly distinguishable at each fixed nuisance state, but changes with $Z$.
Under the nuisance-independent reference experiment $e_R$,
\[P(Y=j\mid M=k,Z=z,e_R)=\mathsf{C}_K(\beta)_{kj},\]
so the observation law depends on $M$ but not on the nuisance state.
\begin{theorem}[Cross-nuisance aliasing separation and exact Tail frontier]
\label{thm:paper-aliasing}
For the experiments $e_I$ and $e_R$ defined above, the following statements hold.
\begin{enumerate}[label=(\roman*)]
\item
For every fixed nuisance state $z$, experiment $e_I$ identifies the structural label perfectly:
\[I(M;Y\mid Z=z,e_I)=\log K.\]
At each fixed $z$, its structural channel also Blackwell-dominates the channel of $e_R$.
\item
After nuisance averaging, the structural channels of $e_I$ and $e_R$ are respectively
\[\mathsf{C}_K(\varepsilon)\qquad\text{and}\qquad\mathsf{C}_K(\beta).\]
Since $\varepsilon<\beta$, the former strictly Blackwell-dominates the latter. Consequently,
\[I(M;Y\mid e_I)=\mathcal I_K(\varepsilon)>\mathcal I_K(\beta)=
I(M;Y\mid e_R),\]
and both structural EIG and average Bayes classification strictly prefer $e_I$ to $e_R$.
\item
Full-latent EIG gives the same ordering:
\[I((M,Z);Y\mid e_I)=\log K>\mathcal I_K(\beta)=I((M,Z);Y\mid e_R).\]
\item
At the pointwise-validity endpoint,
\begin{equation}
V_{e_I}^\star(\alpha,0) =K(1-\alpha),\qquad V_{e_R}^\star(\alpha,0)=1.
\label{eq:paper-pointwise-alias-values}
\end{equation}

\item
For every $0<\delta\le1$, the exact Tail-RAED value of the aliased experiment is
\begin{equation}
V_{e_I}^\star(\alpha,\delta)=\max\!\left\{1,\,\min\!\left[K(1-\alpha),\,
K-(K-1)\frac{\alpha\delta}{\varepsilon}\right]\right\},
\label{eq:paper-tail-alias-value}
\end{equation}
while
\[V_{e_R}^\star(\alpha,\delta)=1.\]

\end{enumerate}

Equivalently, define
\[\delta_0=\frac{K\varepsilon}{K-1},\qquad\delta_1=\frac{\varepsilon}{\alpha}.\]
Then
\begin{equation}
V_{e_I}^\star(\alpha,\delta)=\begin{cases}K(1-\alpha),& 0<\delta\le\delta_0,\\[1mm]K-(K-1)\dfrac{\alpha\delta}{\varepsilon},& \delta_0<\delta<\delta_1,\\[2mm]1,& \delta\ge\delta_1.\end{cases}
\label{eq:paper-tail-alias-piecewise}
\end{equation}

In particular,
\begin{equation}
V_{e_I}^\star(\alpha,\delta)>1\quad\Longleftrightarrow\quad\varepsilon>
\alpha\delta.
\label{eq:paper-tail-alias-iff}
\end{equation}

Thus the information-based criteria above strictly prefer $e_I$, while RAED strictly prefers the nuisance-independent experiment $e_R$ whenever $\varepsilon>\alpha\delta$, as well as at the pointwise endpoint.
\end{theorem}

\paragraph{Proof sketch of Theorem~\ref{thm:paper-aliasing}.}
Under $e_I$, $Y=M\oplus Z$. For every fixed $z$, this is a bijection in $M$, while nuisance averaging produces the structural channel $\mathsf{C}_K(\varepsilon)$. Under $e_R$, the structural channel is the nuisance-independent $\mathsf{C}_K(\beta)$. Since $\varepsilon<\beta$, $\mathsf{C}_K(\beta)$ is a garbling of $\mathsf{C}_K(\varepsilon)$, giving the information and classification orderings in parts (i)--(iii).

For pointwise validity under $e_I$, every pair $(k,y)$ is generated with probability one by the nuisance state $z=y\ominus k$. Any valid rule must therefore satisfy $g_k(y)\ge1-\alpha$ for every $k,y$, which gives $J\ge K(1-\alpha)$. The bound is attainable by a randomized nonempty rule.

For positive-tail validity, cyclic symmetrization reduces the optimization to two exclusion levels: $\ell_0$ on the aligned nuisance state $Z=0$, of mass $1-\varepsilon$, and $\ell_A$ on the alias states $Z\ne0$, of total mass $\varepsilon$. An exchange argument allows an optimum with $\ell_A\ge \ell_0$, for which
\[\rho_\delta(L)=\begin{cases}\ell_A, & \delta\le\varepsilon,\\[1mm] \{\varepsilon \ell_A+(\delta-\varepsilon)\ell_0\}/\delta, & \delta>\varepsilon. \end{cases}\]
Minimizing expected candidate-set size is then a two-variable linear program. Its optimizer changes at $\delta_0=K\varepsilon/(K-1)$ and reaches singleton resolution at $\delta_1=\varepsilon/\alpha$, yielding \eqref{eq:paper-tail-alias-piecewise}. The complete proof, including attainability by randomized nonempty rules, is given in \Cref{app:proof-cyclic-aliasing}.

\begin{figure}[H]
\centering
\includegraphics[width=0.88\linewidth]{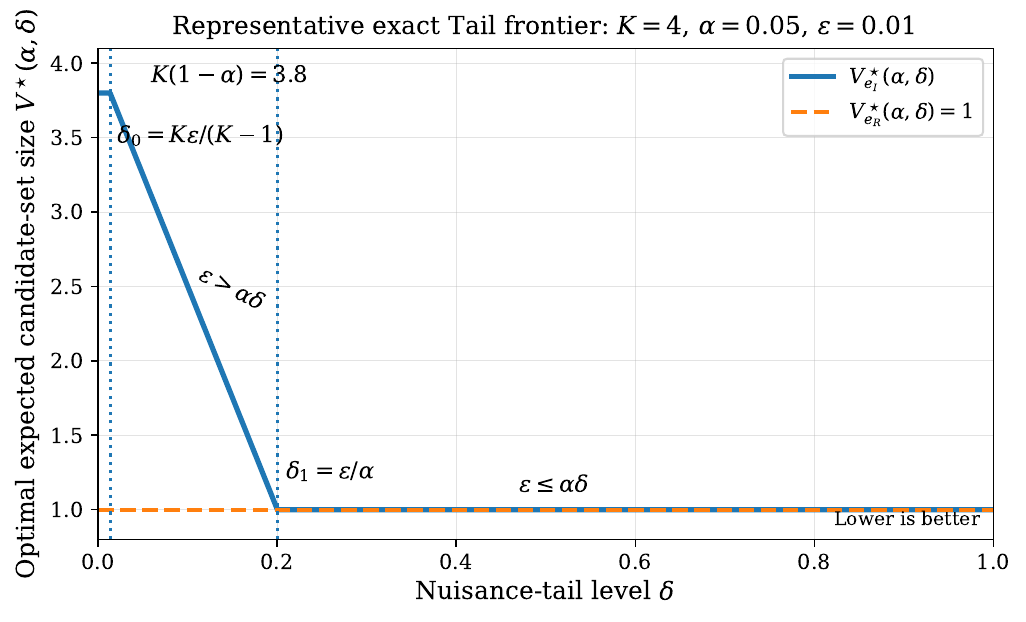}
\caption{\textbf{Exact Tail frontier in the cyclic-aliasing construction.}
For $K=4$, $\alpha=0.05$, and $\varepsilon=0.01$, the aliased experiment $e_I$ remains at $V_{e_I}^\star=3.8$ for the most stringent nuisance tails and then improves with increasing $\delta$, reaching singleton resolution at $\delta_1=\varepsilon/\alpha=0.2$. The nuisance-independent experiment $e_R$ achieves $V_{e_R}^\star=1$ throughout.}
\label{fig:tail-frontier}
\end{figure}

\begin{corollary}[Failure of maximin conditional mutual information]
\label{cor:paper-maximin-conditional-eig}
In the construction of \Cref{thm:paper-aliasing},
\begin{equation}
\inf_{z\in\mathbb Z_K} I(M;Y\mid Z=z,e_I) =\log K >\mathcal I_K(\beta) = \inf_{z\in\mathbb Z_K} I(M;Y\mid Z=z,e_R).
\label{eq:paper-maximin-conditional-eig}
\end{equation}
Hence the maximin criterion $\max_e\inf_z I(M;Y\mid Z=z,e)$ strictly prefers $e_I$, even though RAED strictly prefers $e_R$ whenever $\varepsilon>\alpha\delta$ and at the pointwise endpoint. Replacing nuisance-averaged mutual information by its worst-nuisance conditional value therefore does not resolve the aliasing: every nuisance-specific decoder is perfect under $e_I$, but those decoders are mutually incompatible when $Z$ is unobserved. This statement concerns this particular maximin conditional-MI criterion, not every possible robust information utility.
\end{corollary}

\begin{proof}
For $e_I$, conditioning on $Z=z$ gives $Y=M\oplus z$, a bijection in $M$, so $I(M;Y\mid Z=z,e_I)=\log K$ for every $z$. For $e_R$, the channel is nuisance-independent and equals $\mathsf{C}_K(\beta)$, giving \\$I(M;Y\mid Z=z,e_R)=\mathcal I_K(\beta)<\log K$ because $\beta>0$. The RAED ordering follows from \Cref{thm:paper-aliasing}.
\end{proof}

The separation can also grow with the number of structural families. For every fixed $\bar\delta<1$, the construction can be chosen so that the information-based criteria still prefer $e_I$ while
\begin{equation}
\frac{V_{e_I}^\star(\alpha,\bar\delta)}
     {V_{e_R}^\star(\alpha,\bar\delta)}
\ge K c_{\alpha,\bar\delta}, \qquad c_{\alpha,\bar\delta} =\min\!\left\{1-\alpha,\frac{1-\bar\delta}{1+\bar\delta}\right\}>0.
\label{eq:paper-k-scaling-summary}
\end{equation}
Thus the resolution gap can grow linearly with $K$. The formal statement and proof are given in Corollary~\ref{cor:paper-k-scaling}.

\paragraph{Beyond exact cyclic aliasing.}
Exact label permutations are not required for the same mechanism to arise. Suppose there is an efficiency-relevant response law such that, for every structural family, a nuisance region of $\mu_k$-mass at least $m$ has response laws within average total variation $\tau$ of that common law. The general near-alias result in \Cref{thm:app-near-alias-general} then gives the uniform bound
\begin{equation}
V_e^\star(\alpha,\delta) \ge \max\!\left\{1,\;K\bigl[1-\tau-\alpha\kappa_\delta(m)\bigr]_+\right\}, \qquad \kappa_\delta(m)= \begin{cases} 1, & m\ge\delta,\\ \delta/m, & m<\delta. \end{cases}
\label{eq:paper-near-alias-summary}
\end{equation}
The bound is nontrivial whenever $K[1-\tau-\alpha\kappa_\delta(m)]_+>1$. In the small-mass regime $m<\delta$, the controlling term becomes $\alpha\delta/m$, showing explicitly that increasingly rare near-alias regions can be tolerated as $\delta$ grows. The exact cyclic construction sharpens this general mechanism to the threshold $\varepsilon>\alpha\delta$.

\subsection{Composite informativeness that RAED respects}
\label{sec:common-garbling}

The separation in \Cref{thm:paper-aliasing} does not imply that RAED reverses genuine Blackwell informativeness. There, the information advantage of $e_I$ appears only after nuisance averaging or after interpreting the observation with a nuisance-specific decoder. A stronger comparison is available when one experiment can simulate another through the same post-processing rule for every structural family and nuisance state.

\begin{proposition}[Composite Blackwell monotonicity]
\label{prop:paper-common-garbling}
Let $e_1$ and $e_2$ share the same structural weights, nuisance sets, and validity and efficiency measures. Suppose there exists a Markov kernel $G$, independent of both $k$ and $\theta$, such that
\[ P^{e_2}_{k,\theta} = P^{e_1}_{k,\theta}G \qquad \text{for every }k,\theta. \]
Then, for every $(\alpha,\delta)$,
\begin{equation}
V_{e_1}^\star(\alpha,\delta) \le V_{e_2}^\star(\alpha,\delta).
\label{eq:paper-exact-common-garbling}
\end{equation}
\end{proposition}

The condition means that an observation from $e_1$ can be post-processed through one common kernel $G$ to reproduce the observation law of $e_2$ under every composite state $(k,\theta)$. Hence any valid candidate-set rule for $e_2$ can also be implemented after observing $e_1$: first apply $G$, then apply the $e_2$ rule. This preserves the nuisance-specific false-exclusion losses and the expected candidate-set size. Therefore every resolution level achievable with $e_2$ is also achievable with $e_1$, and $e_1$ cannot have a larger optimal RAED value.

Appendix~\ref{app:proof-common-garbling} extends this result to approximate common garbling. If a single kernel simulates $e_2$ from $e_1$ within uniform total-variation error $\varepsilon$ over all $(k,\theta)$, then each nuisance-specific false-exclusion loss changes by at most $\varepsilon$ and expected candidate-set size by at most $(K-1)\varepsilon$. The appendix also gives an exact-$\alpha$ repair and the corresponding deficiency bounds. Thus the separation in \Cref{sec:paper-cross-nuisance} is not a general conflict between information and structural resolution; it arises when the claimed informativeness does not hold for the full nuisance-indexed experiment through one common comparison kernel.

\section{Learned RAED and Finite-Sample Validity}
\label{sec:learned}

The RAED formulation in \Cref{sec:formulation} optimizes over all measurable randomized nonempty candidate-set rules. In simulator-based problems this unrestricted optimum is generally unavailable, so we work with a learned family of score-based rules. Scores are trained first, experiments are ranked on an independent selection sample, and the selected rule is calibrated on separate data. This separation allows finite-sample validity to be certified without assuming that the learned score family attains the oracle RAED frontier.

\subsection{Score-based candidate sets}
\label{sec:learned-scores}

At the nuisance-average endpoint, when $P_{k,\mu}^{e}\ll Q_e$, the optimal classwise retention rule is governed by the density ratio
\begin{equation}
	r_{k,e}(y) =\frac{\dd P_{k,\mu}^{e}}{\dd Q_e}(y),\qquad
	P_{k,\mu}^{e}(A) =
	\int P_{k,\theta}^{e}(A)\,\mu_k(\dd\theta).
	\label{eq:paper-density-ratio-score}
\end{equation}
Large values of $r_{k,e}(y)$ indicate observations that are relatively more characteristic of structural family $k$ than of the overall predictive mixture $Q_e$. This motivates learning family-specific scores that rank observations by evidence for retaining each family. Under positive-tail validity, the corresponding oracle characterization involves a least-favourable reweighting of the nuisance distribution; the full results are given in \Cref{app:score-geometry,app:additional-theory}.

For each experiment $e$, the learned procedure produces scores
\[s_{k,e}:\cY_e\to\R,\qquad k=1,\ldots,K.\]
Let
\[b_e(y)=\argmax_j s_{j,e}(y),\]
with a fixed deterministic rule for resolving ties. Given family-specific thresholds $t_{k,e}$, define
\begin{equation}
\Gamma_{e,t}(y)=\{b_e(y)\}\cup\{k:s_{k,e}(y)\ge t_{k,e}\}.
\label{eq:paper-learned-rule}
\end{equation}
The first term guarantees that the reported candidate set is never empty, while the thresholds determine which additional structural families are retained.

\subsection{Nested finite-sample calibration}
\label{sec:paper-calibration}

The thresholds in \eqref{eq:paper-learned-rule} control the trade-off between structural resolution and false exclusion. More selective thresholds produce smaller candidate sets but increase the risk of excluding the true family. We therefore choose the deployed thresholds using an independent calibration sample.

For each familywise validity claim $j$, organize the candidate-set rules along a one-dimensional ordered path $\{\Gamma_{j,\lambda}:\lambda\in\Lambda_j\}$, where larger values of $\lambda$ produce larger candidate sets:
\[
\lambda_1\le\lambda_2
\quad\Longrightarrow\quad
\Gamma_{j,\lambda_1}(y)\subseteq\Gamma_{j,\lambda_2}(y).
\]
We allow either a continuous interval
$\Lambda_j=[\underline\lambda_j,\overline\lambda_j]$
or a finite ordered grid
$\Lambda_j=\{\lambda_{j,1}<\cdots<\lambda_{j,m_j}\}$.
In either case, the largest value is a safe endpoint at which the family associated with claim $j$ is always retained. Let $L_{j,\lambda}(\theta)\in[0,1]$ denote its false-exclusion loss at nuisance state $\theta$.

To state the finite-sample guarantee, we condition on everything fixed before the final calibration sample, including score fitting, experiment selection, and construction of the nested path. We denote this pre-calibration information by $\mathcal F_0$.

Along the nested path, let $R_j(\lambda)$ denote the true false-exclusion risk for claim $j$. Because larger candidate sets retain more families, $R_j(\lambda)$ is nonincreasing in $\lambda$. For an interval-valued path we additionally assume that $R_j$ is continuous. At the safe endpoint,
\[
R_j(\overline\lambda_j)=0.
\]

The calibration sample is used to place a one-sided upper confidence bound (UCB) on the unknown risk $R_j(\lambda)$. Let $\gamma_j\in(0,1)$ denote the allowed failure probability for claim $j$. For each fixed $\lambda\in\Lambda_j$, the calibration procedure produces an upper bound $U_j(\lambda;\gamma_j)$ satisfying
\begin{equation}
\Pr\!\left\{R_j(\lambda)\le U_j(\lambda;\gamma_j)\,\middle|\,\mathcal F_0
\right\}\ge 1-\gamma_j.
\label{eq:paper-pointwise-ucb}
\end{equation}
Thus, with probability at least $1-\gamma_j$, the true false-exclusion risk at that fixed rule lies below the reported bound. In the empirical implementation, these bounds are obtained from a bounded-mean Bernoulli--KL concentration inequality; the result below only requires that the stated one-sided coverage guarantee holds.

The goal is to retain as much structural resolution as possible while certifying risk below $\alpha$. We therefore enlarge the candidate set only until the upper confidence bound falls below $\alpha$ and remains below $\alpha$ for every more conservative rule. Define
\begin{equation}
\widehat\lambda_j:=\inf\left\{\lambda\in\Lambda_j:U_j(\lambda';\gamma_j)<\alpha\text{ for every }\lambda'\ge\lambda\right\},
\label{eq:paper-nested-selection}
\end{equation}
using the safe endpoint $\overline\lambda_j$ if no earlier value satisfies this condition.

\begin{theorem}[Finite-sample calibration of a nested RAED path]
\label{thm:paper-nested-calibration}
Under the conditions above, and assuming $\widehat\lambda_j$ is a measurable function of the pre-calibration information and the calibration sample, the selected rule satisfies
\begin{equation}\Pr\!\left\{R_j(\widehat\lambda_j)\le\alpha\,\middle|\,\mathcal F_0\right\}\ge1-\gamma_j.
\label{eq:paper-mean-cal-cert}
\end{equation}
\end{theorem}

The nested structure is what allows the threshold to be chosen from the calibration data without applying a separate multiplicity correction to every candidate value of $\lambda$. For a continuous path, invalid selection forces the upper confidence bound to fail at the population crossing point $R_j(\lambda)=\alpha$. For a finite threshold grid, the corresponding witness is the last invalid threshold on that grid. In either case, the argument requires only one population-determined boundary point. The full proof is given in \Cref{app:proof-nested-mean}.

\subsection{Positive-tail calibration and independent certification}
\label{sec:paper-tail-certification}

For $0<\delta<1$, the risk controlled in \Cref{sec:paper-calibration} is the upper-$\delta$ Tail risk of the nuisance-specific false-exclusion losses. For a familywise claim $j$ associated with structural family $k$, this is
\[R_j(\lambda)=\rho_{\delta,\mu_k}\!\left(L_{j,\lambda}\right).\]
Unlike nuisance-average risk, this quantity is not itself an ordinary mean. We therefore use the variational representation introduced in \eqref{eq:tail-risk} to reduce its calibration to that of a bounded mean.

For each $\lambda$, choose an auxiliary value $\eta_j(\lambda)\in[0,1]$ before the final calibration sample is opened. Then
\begin{equation}
\rho_{\delta,\mu_k}\!\left(L_{j,\lambda}\right)\le\eta_j(\lambda)+
\frac{1}{\delta}\E\!\left[\bigl(L_{j,\lambda}(\theta)-\eta_j(\lambda)\bigr)_+\right].
\label{eq:paper-oce-upper}
\end{equation}
Equality is obtained when $\eta_j(\lambda)$ minimizes the variational expression, but any fixed choice gives a valid upper bound. In the implementation, the complete map $\lambda\mapsto\eta_j(\lambda)$ is constructed from the independent selection sample and fixed before final calibration.

The right-hand side of \eqref{eq:paper-oce-upper} can be calibrated using the same bounded-mean concentration machinery as in \Cref{sec:paper-calibration}. In particular, a one-sided upper confidence bound for the mean of
\[\delta\eta_j(\lambda)+\bigl(L_{j,\lambda}(\theta)-\eta_j(\lambda)\bigr)_+\]
gives, after division by $\delta$, a one-sided upper confidence bound for the Tail risk. Applying \Cref{thm:paper-nested-calibration} along the same nested threshold path therefore gives
\begin{equation}
\Pr\!\left\{\rho_{\delta,\mu_k}\bigl(L_{j,\widehat\lambda_j}\bigr)
\le\alpha\,\middle|\,\mathcal F_0\right\}\ge1-\gamma_j.
\label{eq:paper-tail-cal-cert}
\end{equation}
The full argument is given in \Cref{app:proof-nested-tail}. At $\delta=1$, Tail risk reduces to nuisance-average risk and is covered directly by \Cref{thm:paper-nested-calibration}.

Repeated observation-noise draws at the same nuisance state improve the estimate of that state's false-exclusion loss, but they do not provide additional independent nuisance states. The effective sample size for population Tail certification is therefore the number of independently sampled nuisance states. Appendix~\ref{app:nested-mc} shows that estimating each nuisance-specific loss by inner Monte Carlo remains conservative for the Tail certificate.

Because experiment selection is completed before the final calibration sample is opened, certification only needs to cover the familywise validity claims of the selected experiment. At a fixed $\delta$, assigning
\[\gamma_k=\frac{\zeta}{K},\qquad k=1,\ldots,K,\]
gives joint confidence at least $1-\zeta$ across the $K$ structural families. Independence of the selection and calibration roles avoids an additional multiplicity penalty for the original experiment library, while the nested construction of \Cref{thm:paper-nested-calibration} avoids one over the threshold path. The general selection-aware result is given in \Cref{thm:paper-selected-only}.

The resulting workflow is therefore to train the scores and select the experiment first, fix all choices that precede calibration, calibrate the selected rule on independent nuisance states, and only then use the independent evaluation data. These finite-sample guarantees apply for $0<\delta\le1$. At $\delta=0$, validity is required at every admissible nuisance state rather than over a nuisance population, so the iid calibration argument used here no longer applies.

\subsection{The statistical price of small nuisance tails}
\label{sec:paper-rare-tail}

Nested calibration avoids paying separately for every threshold considered, but it cannot compensate for nuisance states that are too rare to appear in the calibration sample. A simple construction gives a necessary sample-size scale. Consider a population in which the false-exclusion loss is zero almost everywhere but equals one on a nuisance region of probability
\[p=\delta(\alpha+\varepsilon),\qquad \varepsilon>0.\]
Its upper-$\delta$ Tail risk is $\alpha+\varepsilon$, and is therefore just above the allowed level $\alpha$. Yet $n$ independent nuisance draws miss the adverse region entirely with probability $(1-p)^n$. Any distribution-free procedure that would certify after observing no losses must therefore make this event sufficiently unlikely.

\begin{proposition}[Rare-tail nuisance-sampling obstruction]
\label{prop:paper-rare-tail}
Fix $0<\alpha<1$, $0<\delta\le1$, and an allowed false-certification probability $0<\beta<1$. Any deterministic distribution-free procedure that certifies Tail risk at most $\alpha$ after an all-zero calibration sample must satisfy, for every $\varepsilon>0$ with $\alpha+\varepsilon\le1$,
\begin{equation}
\bigl(1-\delta(\alpha+\varepsilon)\bigr)^n\le\beta.
\label{eq:paper-rare-tail-bound}
\end{equation}
Consequently, when $\alpha\delta$ is small,
\begin{equation}
n=\Omega\!\left(
\frac{\log(1/\beta)}{\alpha\delta}
\right).
\label{eq:paper-rare-tail-rate}
\end{equation}
\end{proposition}

The complete two-point argument is given in Appendix~\ref{app:proof-rare-tail}. The result shows why stringent Tail certification can require many independent nuisance states even when the observed calibration losses are small. This limitation motivates using Watt and WCA as computation-limited held-out experiment-selection comparisons, while the separately powered GWAE-Fluvial and methane studies support fixed-positive-$\delta$ population certification.

\section{Benchmark Construction and Experimental Protocol}
\label{sec:protocol}

We evaluate RAED on three reservoir-engineering benchmarks---Watt, WCA, and GWAE-Fluvial---and on a methane-oxidation mechanism-discrimination benchmark. The reservoir studies use active well tests to distinguish competing structural descriptions under persistent geological uncertainty, while methane provides a cheaper mechanistic setting in which finite-sample guarantees for positive-tail false-exclusion risk can be powered directly. \Cref{tab:benchmark-summary} summarizes the structural target, nuisance variables, physical design space, and empirical role of each benchmark. We first describe how the four benchmark problems were constructed and then give the common learned-RAED and comparator protocol.

\begin{table}[t]
\centering
\caption{Overview of the empirical benchmarks.}
\label{tab:benchmark-summary}
\small
\renewcommand{\arraystretch}{1.15}

\begin{tabularx}{\linewidth}{
>{\raggedright\arraybackslash}p{2.0cm}
>{\raggedright\arraybackslash}X
>{\raggedright\arraybackslash}X
>{\raggedright\arraybackslash}p{3.0cm}
}
\hline
Benchmark & Structural target & Experimental design & Main role \\
\hline

Watt &
Fault-model ontologies ($K=3,6,6,12$) &
Active well tests; full or restricted pressure observation &
Ontology and constrained-sensing studies \\
\hline

WCA &
Three depositional architectures &
Directional interference tests with restricted monitoring &
Constrained experiment selection \\
\hline

GWAE-Fluvial &
One versus two spanning channels &
Nine active well tests &
Tail versus nuisance-average design \\
\hline

Methane oxidation &
Three kinetic mechanisms &
81 reactor operating conditions &
Positive-tail certification \\
\hline

\end{tabularx}
\end{table}

\subsection{Watt: distinguishing fault architectures with active well tests}
\label{sec:watt-benchmark}

The Watt field model of Arnold et al.~\cite{arnold2013} is a synthetic reservoir benchmark built around several alternative geological descriptions of the same field. The published benchmark combines three fault models (FM1--FM3), three cutoff alternatives, three grid representations, and three top-structure alternatives, giving $3^4=81$ base scenarios. For the RAED study we retain the three fault-model alternatives while fixing the remaining hierarchy at CO3, G3, and TS2. This gives a common reservoir framework in which the principal structural difference is the interpretation of the fault system.

FM1--FM3 contain the same major named faults, but differ in their traces, spatial extent, segmentation, and therefore the compartment connections they impose on the reservoir. Each fault model is additionally combined with the two supplied property representations, OBJECT and PIXEL, and the two relative-permeability families, RP0 and RP1. Seven transmissibility multipliers further vary the strength of the existing major faults. The resulting ensemble therefore contains several sources of uncertainty around each fault interpretation rather than three deterministic reservoir models.

From this common physical ensemble, we define four alternative model ontologies. For the $K=3$ case, $M\in\{\mathrm{FM1},\mathrm{FM2},\mathrm{FM3}\}$, so that only the fault model forms part of the structural identity. We also consider two six-model descriptions, FM$\times$PROP and FM$\times$RP, and the twelve-model description FM$\times$PROP$\times$RP. Here PROP distinguishes the two supplied property representations, OBJECT and PIXEL, which provide alternative spatial porosity and permeability fields for the same geological scenario, while RP distinguishes the two supplied relative-permeability families, RP0 and RP1. When PROP or RP is not included in the structural identity, it remains part of the within-family nuisance state. In every case, seven continuous fault-transmissibility multipliers also remain nuisance variables. The four ontologies therefore allow the same physical ensemble to be examined at progressively finer distinctions between structural identity and within-model uncertainty.

For each reservoir realization we simulate ten well tests. The experiment library follows established pressure-transient and well-testing practice, combining injector and producer multirate tests, pulse/falloff experiments, periodic rate perturbations, and a producer drawdown/build-up test \cite{bourdet2002,fokker2018,salina2019}. Each test perturbs one designated well while the remaining wells continue under their common background controls. The response records the actuator bottom-hole pressure and realized rate together with pressure at three remote wells. These tests are intended to probe pressure communication through the reservoir and thereby expose differences in fault-controlled connectivity.

The resulting ensemble contains 600 reservoir realizations, each simulated under the ten well tests, giving $6{,}000$ full-physics OPM Flow responses. The simulations use the reconstructed $112\times30\times40$ Watt model and a common equilibrium initial state. Details of the source-model reconstruction, initialization, and ensemble generation are given in \Cref{app:watt-construction}.

The Watt ensemble is used in two complementary ways. First, using the full well response, we analyse all four model ontologies, with $K=3,6,6,$ and $12$. This isolates the effect of model definition: the underlying reservoir realizations and physical experiments are unchanged, while progressively finer distinctions in fault model, property representation, and relative permeability are promoted from nuisance to components of the structural identity.

We then consider a separate constrained-sensing study using the $K=3$ fault-family ontology. Here the structural question is fixed and the available observation is varied instead. The response is truncated at $T\in\{7,14,30,90\}$ days, and pressure is retained from $0$--$3$ remote wells. Actuator pressure and realized rate are always observed. Each combination of observation horizon and number of remote wells therefore defines a common sensing level within which the ten well tests and admissible monitor locations are compared.

The first study examines how achievable resolution changes as the structural distinction becomes finer. The second keeps the structural target fixed and examines how experiment choice changes when observation time and remote pressure measurements are restricted.

\subsection{WCA: depositional-architecture discrimination by pressure interference}
\label{sec:wca-benchmark}

The WCA benchmark is based on the multiple-geological-interpretation study of Park et al.~\cite{park2013}. The source study considers three alternative training images, TI1--TI3, representing different fluvial depositional architectures. TI1 is characterized by relatively thin, low-sinuosity channels, TI2 by thicker low-sinuosity channels, and TI3 by a more sinuous channel system with associated levee architecture. Their overall sand/shale proportions are similar, so the main distinction between the three interpretations lies in the geometry and connectivity of the permeable bodies rather than in gross facies volume.

For each training image we retain ten conditional facies realizations that honor the source hard data. Each realization is a different spatial arrangement of the channel, levee, and shale facies that is compatible with the same parent training image. These realizations are not treated as separate geological hypotheses. Instead, $M\in\{\mathrm{TI1},\mathrm{TI2},\mathrm{TI3}\}$ defines the structural identity. Within each family, the nuisance state consists of the particular conditional facies realization together with its porosity and permeability fields. For a fixed facies geometry, these properties are varied using source-derived distributions and spatially correlated random fields, so reservoirs belonging to the same structural family can differ substantially in both local rock quality and flow response.

Constructing the dynamic benchmark required placing these geological realizations in a common reservoir model. The categorical facies fields were mapped to a $53\times40\times116$ simulation grid chosen to preserve the main geometric and directional connectivity characteristics of the source models. Porosity and permeability were then assigned within the facies using the source-derived property distributions and spatially correlated random fields described above. A common physical scale, fluid model, initial state, and oil--water contact were used across all realizations, together with a fixed $3\times3$ array of nine wells. Wells are completed only in active reservoir cells above the oil--water contact. Full details of the geological reconstruction, petrophysical mapping, and dynamic model are given in \Cref{app:wca-construction}.

The WCA experiment library follows the same well-testing principles as Watt and comprises ten active interference tests on the nine-well layout. Six tests probe reservoir communication along reciprocal east--west, north--south, and diagonal injector--producer directions. The remaining four use multirate excitation, alternating injection, a balanced multiwell interference pattern, and a central pulse/recovery test. Together, the experiments probe different flow paths and timescales through the competing depositional architectures. Each training image contributes 1000 reservoir realizations, giving $3{,}000$ realizations in total. Simulating every realization under the ten well tests produces $30{,}000$ full-physics OPM Flow responses.

As in Watt, the complete responses are used both under rich observation and under a separate constrained-sensing analysis. In the constrained case, the response is truncated at $T\in\{7,14,30,60\}$ days, and pressure is retained from $0$--$3$ remote wells. Actuator pressure and realized rate are always observed. The structural target remains fixed at the three parent training images, so this analysis isolates the effect of reduced observation time and spatial coverage on experiment choice and achievable resolution.

\subsection{GWAE-Fluvial: one- versus two-channel discrimination}
\label{sec:gwae-benchmark}

The GWAE-Fluvial study builds on the single- and double-channel reservoir structures described by Shishaev, Demyanov, and Arnold \cite{shishaev2026} and the associated geological prior information reported in that work. We use the reported channel dimensions and structural ranges to construct a parametric fluvial ensemble on the supplied reservoir model, with one- and two-channel families defined directly in terms of their geological geometry and petrophysical variability.

The structural target is $M=1$ for a reservoir containing one spanning channel and $M=2$ for a reservoir containing two separated channels. Within each family, the nuisance variables describe the particular channel geometry and rock properties. Channels are represented as finite sinusoidal ribbons spanning the reservoir, with nuisance variation in width, thickness, wavelength, phase, lateral and vertical position, and petrophysical quality. In the two-channel family, the two channels share the same broad geometric form but may differ in width, thickness, vertical position, and lateral separation. Exact parameter ranges, geometric constraints, and the porosity/permeability construction are given in \Cref{app:gwae-construction}.

The geological realizations are simulated on the supplied reservoir grid using a common OPM Flow model and well configuration, with west and east injector banks separated by three central producers. Nine well tests are considered, including simultaneous and single-bank injection, cross-reservoir and central-pair tests, staged-rate excitation, pulse/falloff, and alternating-bank injection. Responses are recorded at days $1,3,5,7,$ and $10$ and consist of injector rates and pressures together with producer oil and water production rates.

The ensemble contains $24{,}000$ geological realizations, each simulated under the nine well tests, giving $216{,}000$ full-physics reservoir responses. This population is used to compare nuisance-average RAED at $\delta=1$ with Tail-RAED at $\delta=0.05$, both at $\alpha=0.05$. The earlier nuisance-average analysis in \Cref{sec:gwae-hard-region-results} uses a separate $10{,}000$-realization population, whereas the Tail-versus-average comparison is carried out on the independent $24{,}000$-realization ensemble described here. This allows us to examine whether protecting against a small adverse region of geological uncertainty changes the preferred well test and the resolution obtained after calibration.

\subsection{Methane oxidation: a powered certification benchmark}
\label{sec:methane-protocol}

The methane-oxidation study follows the complete-oxidation discrimination problem of Pankajakshan et al.~\cite{pankajakshan2023}. The structural target consists of three rival kinetic mechanisms, $M\in\{\mathrm{PL},\mathrm{LH},\mathrm{MVK}\}$, corresponding to power-law, Langmuir--Hinshelwood, and Mars--van Krevelen rate models. Within each mechanism, persistent nuisance parameters follow the fitted multivariate Gaussian approximation reported in the source study.

The physical design space is a full $3^4$ factorial over temperature, total feed flow, inlet oxygen-to-methane ratio, and inlet methane fraction, giving 81 candidate reactor conditions. Each experiment observes the outlet mole fractions of methane, oxygen, and carbon dioxide under the source-derived measurement-noise model.

Methane is used primarily to test finite-sample positive-tail calibration. Its steady reactor model is inexpensive enough to generate much larger independent nuisance samples than the reservoir benchmarks. We therefore use $10{,}000$ calibration nuisance states per structural family, together with a separate $5{,}000$-state evaluation set, so that nontrivial positive-tail population guarantees can be assessed directly.

The score-training, experiment-selection, calibration, and evaluation nuisance populations are kept disjoint. At each fixed positive tail level $\delta>0$, the familywise calibration confidence is allocated equally across the three mechanisms, giving $95\%$ joint confidence across PL, LH, and MVK. Exact nuisance distributions, reactor equations, design conditions, and calibration details are given in \Cref{app:methane-construction}.

\subsection{Common learned-RAED protocol}
\label{sec:unified-scorer}

All four benchmarks follow the learned-RAED workflow of \Cref{sec:learned}. Scores are fitted first, experiments are ranked on an independent nuisance sample, the selected experiment is then calibrated on a separate sample, and final performance is measured on independent evaluation states.

Watt, WCA, and GWAE-Fluvial use one-versus-rest logistic scores on a 256-component radial-basis-function (RBF) Nyström representation. Methane uses the same construction with 96 Nystr\"om components. Cross-validation is grouped by nuisance state, so repeated observation-noise realizations of the same physical state remain together.

The nuisance states used for score fitting, experiment selection, calibration, and evaluation are disjoint in every benchmark. Repeated noise draws at a fixed nuisance state are used only to estimate its response distribution and conditional false-exclusion loss, not as additional independent nuisance samples.

After experiment selection, only the selected experiment is calibrated. Watt and WCA use this protocol for held-out empirical comparison. Methane and GWAE-Fluvial use larger calibration populations to support finite-sample positive-tail guarantees. Exact sample sizes, score settings, noise replication, and calibration details are given in Appendix~\ref{app:implementation}.

\subsection{Comparator criteria and downstream evaluation}
\label{sec:comparators}

RAED is compared with model-index expected information gain, full-latent expected information gain, Bayes forced classification, expected model elimination, and minimum ordered-pair predictive discrimination. All comparator criteria are evaluated on the same independent experiment-selection population used for RAED ranking.

The comparison is made at the level of experiment choice. Once a criterion selects an experiment, the selected experiment is evaluated using the same learned candidate-set construction, calibration procedure, and independent evaluation population as the RAED-selected experiment. Differences in the reported structural resolution can therefore be attributed to the selected experiment rather than to a different decision rule.

The EIG criteria are estimated by Monte Carlo and can be numerically indistinguishable near their optimum. We therefore distinguish a different numerical maximizer from a clear selection disagreement: RAED and EIG are counted as selecting differently only when the RAED-selected experiment lies outside the predefined Monte Carlo uncertainty set around the EIG optimum. Exact estimators, sampling counts, and tie rules are given in Appendix~\ref{app:implementation}.

\section{Results}
\label{sec:results}

\subsection{Structural granularity and achievable resolution in Watt}
\label{sec:watt-ontology-results}

\begin{figure}[H]
	\centering
	\includegraphics[width=0.7\linewidth]{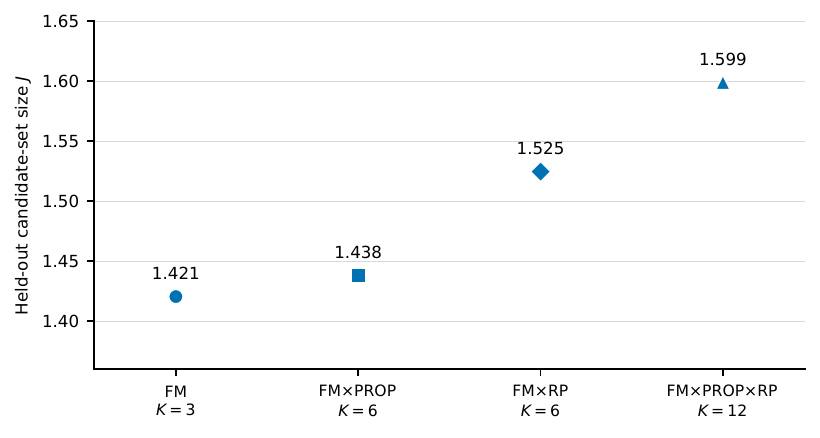}
\caption{\textbf{Achievable structural resolution across Watt model ontologies.}
	Held-out mean candidate-set size $J$ for the four structural targets. The same INJ5 pulse/falloff test is selected in every case at $(\alpha,\delta)=(0.05,1)$.}
	\label{fig:watt-ontology}
\end{figure}

The Watt ensemble provides a direct test of how the definition of the structural target affects achievable resolution, because the same reservoir realizations, well tests, and observations can be evaluated under all four model ontologies. At $(\alpha,\delta)=(0.05,1)$, using the full well response, learned RAED selected the same INJ5 pulse/falloff test for the $K=3$ FM, $K=6$ FM$\times$PROP, $K=6$ FM$\times$RP, and $K=12$ FM$\times$PROP$\times$RP targets. The corresponding held-out mean candidate-set sizes are $J=1.421$, $1.438$, $1.525$, and $1.599$. Thus progressively finer distinctions increase the absolute number of surviving explanations only modestly. Because $K$ changes across the four ontologies, $J$ and the normalized resolution $R$ describe complementary aspects of the result: $J$ gives the expected number of structural explanations that remain, whereas $R$ measures resolution relative to the size of the candidate model set. We emphasize $J$ here because it retains the direct candidate-set interpretation of the RAED objective. Along both refinement paths, normalized resolution also increases. Both quantities are reported in \Cref{tab:watt-ontology-summary}. The corresponding worst-family held-out nuisance-average false-exclusion rates are $0.056$, $0.067$, $0.057$, and $0.061$, respectively. These values are close to, but slightly above, the nominal $\alpha=0.05$ level. The complete resolution and exclusion results are reported in \Cref{tab:watt-ontology-summary}.

When PROP and RP are treated as nuisance in the $K=3$ fault-family problem, RAED must distinguish FM1--FM3 in the presence of the response variability induced by those additional modeling choices. The resulting $J\approx1.42$ therefore represents strong empirical fault-family resolution under substantial within-family variation. When PROP and RP are instead promoted into the structural identity, the candidate model set grows from three to twelve, yet the held-out mean candidate-set size increases only to about $1.60$. The full well response therefore retains substantial discriminatory power even for the finer structural distinctions, although the aggregate value of $J$ alone does not identify which particular model pairs account for the remaining ambiguity.

The fact that the same INJ5 pulse/falloff test is selected across all four ontologies is also informative. Its value is not confined to the coarsest fault-family problem: the same physical experiment remains preferred as property representation and relative-permeability family are progressively incorporated into the structural target. The modest increase in $J$ shows that these additional distinctions do introduce further ambiguity, but much less than might be expected from the fourfold increase in the number of candidate models.

These Watt ontology results serve a different purpose from the constrained-sensing comparisons studied later in this section. Here the physical reservoir ensemble, well tests, and observations are unchanged; only the definition of structural identity is varied. The comparison therefore isolates how achievable resolution changes when attributes that were previously treated as nuisance are instead required to be distinguished explicitly. The reported values are held-out empirical summaries of the nontrivial learned rules rather than finite-sample population certificates. The available Watt calibration sample does not certify these nontrivial rules at the declared $5\%$ level, so the ontology study provides an empirical resolution analysis.

\subsection{Constrained sensing separates RAED and information gain}
\label{sec:restricted-results}

The constrained Watt and WCA studies keep the structural target fixed while progressively reducing the available observation time and number of remote pressure measurements. We focus on the operating point $(\alpha,\delta)=(0.05,0.05)$. Within each sensing level, RAED and model-index EIG rank the same set of admissible well tests and monitor configurations using the same independent selection data. The two criteria therefore differ only in how they select the experiment. Model-index EIG does not itself produce a candidate set: after each criterion selects an experiment, that experiment is evaluated through the same learned-RAED downstream procedure, including score fitting, threshold calibration, and independent held-out evaluation.

In WCA, RAED and model-index EIG are cleanly separated in seven of the sixteen horizon--monitor combinations after accounting for the finite-sample Monte Carlo error in the EIG estimates. In five of these seven cells, the RAED-selected experiment yields a smaller held-out candidate set, with reductions in $J$ between $0.104$ and $0.385$. One cell is essentially tied, with a difference of $-0.004$, and one yields a smaller candidate set for the EIG-selected experiment by $0.118$. Across all sixteen sensing levels, the experiments selected by RAED give a mean held-out candidate-set size of $J=2.520$, compared with $J=2.658$ for those selected by model-index EIG.

These differences in candidate-set size are accompanied by substantial variation in held-out Tail false-exclusion risk. In several WCA cells, the smaller candidate sets obtained after the RAED-selected experiment also have higher worst-family upper-$5\%$-Tail exclusion than those obtained after the EIG-selected experiment. For example, at $(T,B)=(7,2)$ days and remote wells, the candidate-set sizes are $2.378$ and $2.763$ for the RAED- and EIG-selected experiments, respectively, while the corresponding worst-family held-out Tail risks are $0.365$ and $0.111$. The reduction in ambiguity should therefore not be interpreted independently of the associated exclusion risk.

Watt shows the separation more weakly. RAED and model-index EIG are cleanly separated in three of the sixteen sensing levels, and in all three cases the RAED-selected experiment yields a smaller held-out candidate set, with reductions in $J$ between $0.017$ and $0.054$. Averaged over all sixteen sensing levels, mean held-out $J$ is $2.092$ for the RAED selections and $2.111$ for the model-index-EIG selections. In two of the three clean disagreements, the smaller candidate set is accompanied by a lower worst-family held-out Tail risk; in the third, at $(T,B)=(90,0)$, the RAED-selected experiment has the smaller $J$ but the larger Tail risk. Complete resolution and exclusion results for both benchmarks are reported in \Cref{tab:watt-restricted-complete,tab:wca-restricted-complete}.

\begin{figure}[H]
	\centering
	\includegraphics[width=0.96\linewidth]{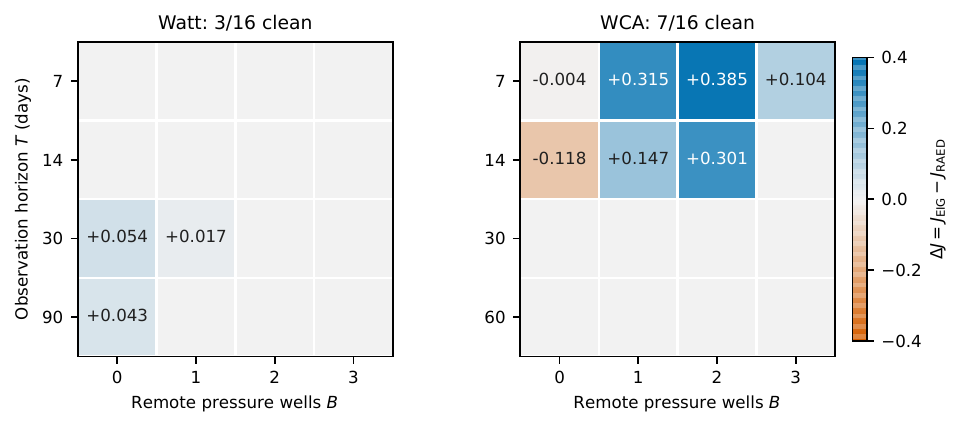}
	\caption{\textbf{Clean RAED--model-index-EIG disagreements under restricted sensing.} Only cells in which the RAED-selected design lies outside the EIG Monte Carlo indistinguishability set are colored and annotated. Positive $\Delta J=J_{\rm EIG}-J_{\rm RAED}$ indicates a smaller held-out candidate set for the RAED-selected experiment. The common color scale makes the stronger WCA separation visible; exact values for all 16 cells per benchmark are reported in \Cref{tab:watt-restricted-complete,tab:wca-restricted-complete}.}
	\label{fig:restricted-deltaJ}
\end{figure}

The Watt and WCA results show that information gain and structural resolution can rank the same set of experiments differently when sensing is restricted. In many cells the two criteria still agree, or are numerically indistinguishable, and WCA also contains a clean case in which the EIG-selected experiment gives the smaller downstream candidate set. The main point is therefore not that RAED dominates EIG, but that the choice of design criterion can change both the selected experiment and the resulting balance between ambiguity and false exclusion.

The learned rules are calibrated on independent data and evaluated on held-out states. They are not accompanied by finite-sample population-validity certificates, so the Tail risks reported here should be interpreted as empirical held-out performance.

The additional comparators give a similar picture in terms of candidate-set size. In WCA, the experiments selected by full-latent EIG, Bayes forced classification, expected elimination, and minimum ordered-pair predictive discrimination all give larger mean held-out $J$ than the RAED selections across the sixteen sensing levels, with ordered-pair discrimination the closest alternative. Several comparator rankings contain numerically indistinguishable optima, so the averages use the representative design defined for each criterion; the full summary is given in \Cref{tab:restricted-comparator-means}.

\subsection{Nuisance-average validity can conceal a geological hard region}
\label{sec:gwae-hard-region-results}

The GWAE-Fluvial benchmark provides a concrete example of a limitation of nuisance-average validity: a small family-average false-exclusion rate need not imply uniformly low risk across the geological nuisance population. In the first GWAE-Fluvial analysis, RAED was applied with $(\alpha,\delta)=(0.05,1)$ and selected the symmetric-bank experiment E01. Selected-only calibration on 1000 nuisance states per family gave risk upper bounds of $0.036541$ for M1 and $0.049826$ for M2, providing 95\% joint confidence across the two familywise population claims. On the independent evaluation population, the corresponding familywise false-exclusion rates were $0.0141$ for M1 and $0.0286$ for M2. The rule also remained highly resolving, with $J=1.1267$. Because $K=2$, approximately $87.3\%$ of observations were resolved to a singleton and $12.7\%$ returned the ambiguous set $\{\mathrm{M1},\mathrm{M2}\}$.

The family averages, however, do not show how the exclusion risk is distributed across geological realizations. For each nuisance state $\theta$, we fixed the deterministic reservoir response and estimated its conditional false-exclusion probability over repeated draws from the assumed observation model. In both structural families, the median and 90th percentile of this statewise risk were zero, indicating that most geological realizations were almost never falsely excluded. The upper tail was very different: the 99th-percentile conditional risk was approximately $0.86$ for M1 and $0.98$ for M2, while the mean risk within the worst $1\%$ of nuisance states was approximately $0.998$ and $0.991$, respectively.

\begin{figure}[H]
	\centering
	\includegraphics[width=0.82\linewidth]{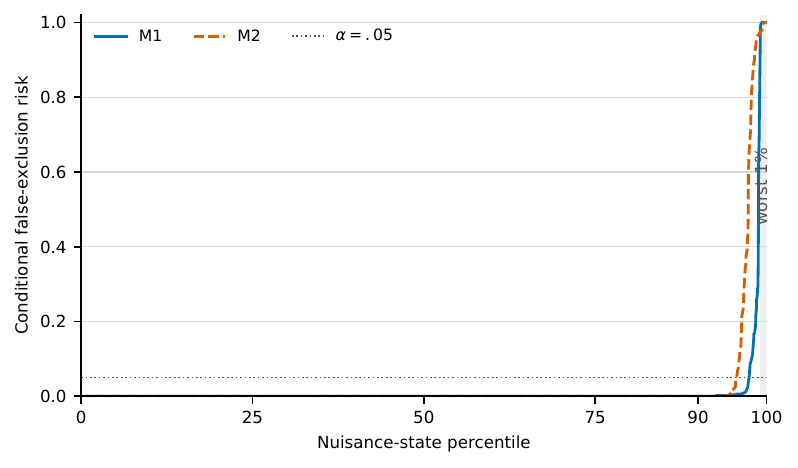}
	\caption{\textbf{Statewise risk concentration in the first GWAE-Fluvial campaign.} Sorted conditional false-exclusion risks from 2000 untouched nuisance states per family show a long zero-risk region and an extreme upper tail. The 99th percentiles are $0.8626$ for M1 and $0.9746$ for M2; within the shaded worst $1\%$, mean risks are $0.9982$ and $0.9907$. These are descriptive evaluations of the independently calibrated nuisance-average rule.}
	\label{fig:gwae-historical-risk}
\end{figure}

The nuisance-average guarantee is therefore compatible with very low exclusion risk over most of the geological population and near-systematic failure on a small subset of hard realizations. Nothing in the $\delta=1$ criterion prevents this concentration, even though the family-average exclusion target is satisfied. The first campaign was designed as a nuisance-average analysis rather than a Tail-RAED sensitivity study; the concentration of risk only became apparent in the subsequent statewise evaluation.

This observation motivated a second independent GWAE-Fluvial study using a new nuisance population, in which nuisance-average RAED was compared directly with Tail-RAED at $\delta=0.05$. Using an independent population separates the discovery of the hard-tail behavior from the subsequent test of whether protecting that tail changes the selected physical experiment and the resulting balance between resolution and false exclusion.

\subsection{Tail protection changes the selected GWAE experiment}
\label{sec:gwae-tail-results}

The second GWAE-Fluvial study compares nuisance-average RAED, $\delta=1$, with Tail-RAED at $\delta=0.05$ on the same newly generated geological population. The two criteria select different physical experiments: nuisance-average RAED selects the symmetric-bank test E01, whereas Tail-RAED narrowly selects the staged-rate test E07. The change nevertheless shows that protecting the adverse nuisance tail can affect experiment selection itself, rather than only the thresholds calibrated after an experiment has been chosen.

\begin{table}[H]
	\centering
	\caption{Deployed GWAE-Fluvial operating points on the common independent evaluation population. Worst-5\% entries are the empirical upper-tail conditional-risk means computed separately for each method.}
	\label{tab:gwae-main-summary}
	\small
	\setlength{\tabcolsep}{5pt}
	\begin{tabular}{llrrr}
		\hline
		Method & Selected well test & $J$ & M1 worst-5\% & M2 worst-5\% \\
		\hline
		Tail, $\delta=.05$ & E07 staged-rate pair & \GwaeTailJ & \GwaeTailMOneWorstFive & \GwaeTailMTwoWorstFive \\
		Average, $\delta=1$ & E01 symmetric banks & \GwaeAverageJ & \GwaeAverageMOneWorstFive & \GwaeAverageMTwoWorstFive \\
		\hline
	\end{tabular}
\end{table}

For both selected experiments, calibration produced nontrivial candidate-set rules with familywise risk upper bounds below $\alpha=0.05$. On the common independent evaluation population, however, the two rules occupy very different resolution--risk operating points. Nuisance-average RAED remains highly resolving, with $J=1.047$, but its empirical upper-tail conditional-risk means are $0.434$ for M1 and $0.876$ for M2. Tail-RAED increases the expected candidate-set size to $J=1.789$, while its corresponding upper-tail risks fall to $0.014$ and $0.003$ respectively. The tail summaries in \Cref{tab:gwae-main-summary} are computed separately for each rule and therefore describe each rule's own hardest $5\%$ of nuisance states.

For $K=2$, the increase in $J$ has a direct interpretation: Tail-RAED returns the ambiguous set $\{\mathrm{M1},\mathrm{M2}\}$ much more often. To compare the two rules on the same difficult geological states, \Cref{fig:gwae-hard-region} defines the hard region using the worst $5\%$ of nuisance states under the nuisance-average rule. Within this common region, Tail-RAED largely replaces false singleton exclusions with explicit ambiguity.

The GWAE comparison therefore shows the cost of tail protection directly. Nuisance-average RAED achieves much smaller candidate sets, but a small subset of geological realizations remains exposed to high false-exclusion risk. Tail-RAED sacrifices part of that resolution in order to reduce exclusion in these difficult states, returning larger candidate sets when the observations do not support a confident structural decision.

\begin{figure}[H]
	\centering
	\includegraphics[width=0.82\linewidth]{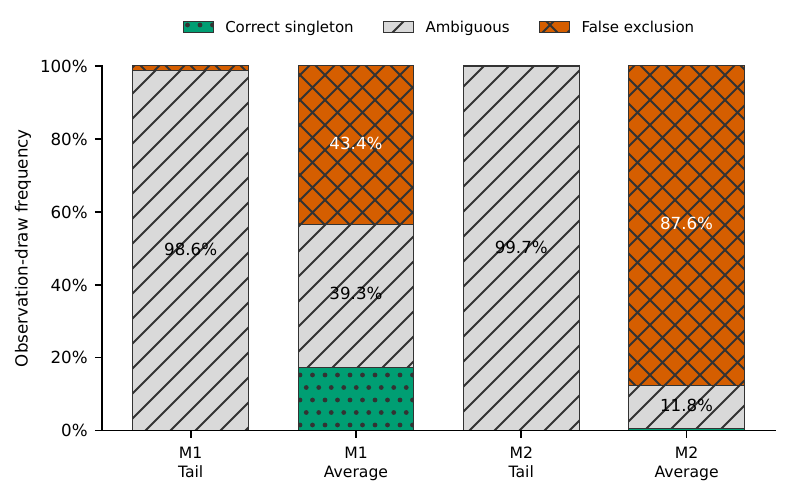}
	\caption{\textbf{Behavior in the nuisance-average rule's worst 5\% geological region.} Frequencies are computed over observation-noise draws within the 250 highest-risk nuisance states per family under the nuisance-average rule. Tail-RAED converts most behavior in this region to explicit ambiguity, whereas nuisance-average RAED produces false singleton exclusions on $43.4\%$ of M1 draws and $87.6\%$ of M2 draws.}
	\label{fig:gwae-hard-region}
\end{figure}

\subsection{Finite-sample Tail certification in methane oxidation}
\label{sec:methane-results}

Methane oxidation provides a case in which the positive-tail validity requirement can be certified with the available nuisance sample. Across the declared design space, learned RAED and model-index EIG select the same reactor condition, so this benchmark is not used to demonstrate disagreement in experiment selection. Instead, it tests whether Tail-RAED can retain nontrivial structural resolution while satisfying the finite-sample validity requirement.

At the reference operating point $\alpha=0.05$ and $\delta=0.05$, calibration uses $10{,}000$ independent nuisance states per mechanism. The resulting rule satisfies the familywise Tail-risk bounds with 95\% joint confidence across PL, LH, and MVK. The three calibration upper bounds are $0.048$, $0.038$, and $0.048$, all below the prescribed false-exclusion level.

The certified rule remains informative, with held-out mean candidate-set size $J=1.750$. On the independent evaluation population, the worst-family empirical 5\% Tail risk is $0.025$. The finite-sample certificate therefore does not force the rule to retain all three mechanisms: positive-tail validity is certified while a nontrivial level of structural resolution is preserved.

\begin{figure}[H]
	\centering
	\includegraphics[width=0.82\linewidth]{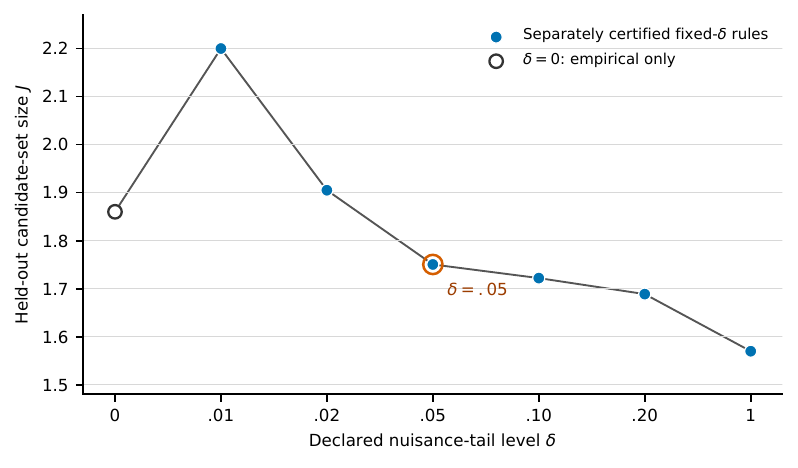}
	\caption{\textbf{Deployed methane-oxidation frontier.} Each filled point is a separately certified fixed-positive-$\delta$ operating point with 95\% joint confidence across PL, LH, and MVK; the guarantees are not simultaneous across the full grid. The open $\delta=0$ point is empirical only. The line is a guide across the ordered declared levels, not a smooth theoretical frontier.}
	\label{fig:methane-deployed-frontier}
\end{figure}

The same pattern holds across the other predeclared positive Tail levels. As $\delta$ decreases, stronger protection of adverse nuisance states generally requires larger candidate sets, whereas larger $\delta$ permits greater resolution. The complete frontier is reported in Appendix~\ref{app:results}.

\section{Discussion}
\label{sec:discussion}

The empirical studies illustrate several different consequences of designing experiments around structural elimination rather than information gain alone. The Watt ontology study shows that the structural question itself matters: when property representation and relative-permeability family are promoted from nuisance variables to components of model identity, the number of candidate models increases substantially while the empirical candidate-set size grows only modestly. The restricted Watt and WCA studies show a different effect. Once observation time and spatial coverage are limited, RAED and information-based criteria can prefer different physical experiments even though they act on the same admissible design set. GWAE then shows that experiment choice can also depend on the form of the validity requirement itself, while methane provides a case in which positive-tail validity can be certified without collapsing to the trivial retain-all rule.

These results also make clear that resolution cannot be interpreted separately from false exclusion. In the Watt ontology study, the nontrivial learned rules remain highly resolving across all four structural targets, but their held-out exclusion rates lie slightly above the nominal level and the available calibration sample does not support a nontrivial population certificate. The restricted-sensing results make the same point more strongly. A RAED-selected experiment can yield a smaller downstream candidate set than an EIG-selected experiment while also producing a larger held-out Tail risk. In other cells the two quantities improve together. The relevant empirical outcome is therefore the pair consisting of remaining ambiguity and false-exclusion behavior, not candidate-set size alone.

GWAE provides the clearest example of why the nuisance distribution matters. Under nuisance-average validity, the first study satisfies its familywise population requirement and achieves high resolution, yet the statewise evaluation reveals a strongly concentrated failure mode: most geological realizations have essentially zero conditional false-exclusion risk, while a small subset is excluded almost systematically. This behavior is entirely compatible with $\delta=1$, because the constraint acts on the nuisance average rather than on the upper tail of the statewise risk distribution. The result gives a concrete geological interpretation to the distinction formalized by the Tail-RAED objective and to the rare-tail argument in Proposition~\ref{prop:paper-rare-tail}.

The second GWAE study asks what changes when that adverse tail is protected explicitly. With $\delta=0.05$, Tail-RAED narrowly selects a different physical experiment and produces a very different operating point. The nuisance-average rule achieves much smaller candidate sets, whereas Tail-RAED returns the ambiguous set far more often and sharply reduces false exclusion in the difficult geological states. This is the intended role of set-valued decisions in RAED: when the observations do not support a reliable exclusion, the method is allowed to retain competing explanations instead of forcing a singleton decision. The cost is substantial, however, and smaller $\delta$ should not be interpreted as automatically preferable. The choice of $(\alpha,\delta)$ encodes how much false exclusion is acceptable and how much of the nuisance population should receive explicit protection.

The change in selected GWAE experiment should also be interpreted with some caution. E07 only narrowly outranks E01 under the Tail criterion, so the ranking itself is sensitive. The stronger result is not the identity of the winning experiment but what happens after calibration and independent evaluation: the nuisance-average and Tail rules occupy very different resolution--risk operating points, and Tail-RAED substantially changes the behavior of the rule on the geological states that were most problematic under nuisance-average validity. More generally, the Watt and WCA $\delta$ sweeps show that the preferred restricted-sensing design can itself depend on the requested degree of nuisance-tail protection.

The four empirical benchmarks carry different statistical weight. Watt and WCA provide independent held-out evaluations of learned rules, but their available nuisance populations are too small to support the stronger positive-tail population certificates used elsewhere in the paper. Their Tail-risk quantities are therefore empirical generalization diagnostics. GWAE and methane use much larger independent calibration populations, allowing fixed-operating-point familywise risk upper bounds to be certified. In GWAE, this supports both the nuisance-average and $\delta=0.05$ Tail analyses; in methane, nontrivial certificates are obtained across the declared positive-$\delta$ operating points. In these cases the formal validity statements come from the calibration upper bounds, while held-out Tail means, statewise quantiles, and hard-region analyses describe how the already calibrated rules behave in practice.

Validity is always relative to the scientific problem that has been declared. The admissible nuisance sets $\Theta_k$ determine which within-family states must be considered, $\mu_k$ determines how positive-tail protection is distributed over those states, and $\nu_k$ determines how resolution is averaged. The GWAE-Fluvial study, for example, uses an explicitly constructed single- and double-channel geological population with controlled geometric and petrophysical variability. Its conclusions therefore apply to that declared geological law. The same is true of the observation models, experiment libraries, and sensing restrictions used throughout the benchmarks: changing any of these changes the decision problem being solved.

Finally, the empirical studies use a restricted learned score family rather than the population oracle. Calibration controls false exclusion for the resulting learned candidate-set rule, while the oracle theory describes the ideal target that this restricted family is attempting to approximate. The present finite-sample population guarantees cover $0<\delta\le1$; the empirical $\delta=0$ points should not be confused with population-uniform pointwise guarantees. The theory is also one-step: one experiment is selected, one response is observed, and one candidate set is returned. A natural extension is to allow the score model, experiment, and exclusion thresholds to be learned more jointly, or to move to sequential designs in which later experiments are chosen specifically to break the aliases that remain after earlier observations.

\section{Conclusion}
\label{sec:conclusion}

Resolution-Aware Experimental Design values an experiment by the structural explanations it can eliminate while controlling false exclusion under persistent nuisance uncertainty. Cross-nuisance aliasing shows that this objective can differ sharply from information gain, while the common-garbling results preserve the expected ordering when one experiment is genuinely less informative than another. The learned formulation makes the same objective usable with restricted score models under nuisance-average and positive-tail validity.

The empirical studies expose different aspects of this problem. In Watt, changing the definition of structural identity alters the attainable resolution while the preferred experiment remains remarkably stable across the declared tail levels. Under constrained sensing, Watt and especially WCA show that RAED and information-based criteria can select different physical experiments, with corresponding changes in both candidate-set size and held-out false-exclusion behavior. GWAE-Fluvial shows why nuisance-average validity can be insufficient when risk is concentrated on a small set of geological realizations: protecting the upper nuisance tail changes the selected experiment and replaces many hard-region false exclusions with explicit ambiguity, at a substantial cost in average resolution. Methane oxidation provides the complementary case in which positive-tail validity can be certified nontrivially with a sufficiently large nuisance-calibration population.

Taken together, the results show that experiment design under partial identifiability is not only a question of acquiring information. It also requires deciding which structural exclusions must remain reliable across nuisance uncertainty and how much unresolved ambiguity is acceptable in exchange. RAED makes that tradeoff explicit through the candidate set and the operating point $(\alpha,\delta)$.

\appendix

\section{Notation and Technical Preliminaries}
\label{app:notation}

For reference, the principal objects used throughout the paper are collected below.

\begin{center}
\small
\begin{tabularx}{\linewidth}{lX}
\toprule
Symbol & Meaning \\
\midrule
$M\in\{1,\ldots,K\}$ & structural family of scientific interest \\
$\theta\in\Theta_k$ & admissible persistent nuisance state within family $k$ \\
$e\in\mathcal E$ & complete precommitted experiment/intervention and observation protocol \\
$P^e_{k,\theta}$ & predictive observation law under $(k,\theta,e)$ \\
$\mu_k$ & validity-mass measure used by Tail-RAED \\
$\nu_k$ & nuisance efficiency measure used in the objective \\
$\pi_k$ & structural efficiency weight, with $\sum_k \pi_k=1$ \\
$\Gamma_e(Y)$ & randomized nonempty structural candidate set returned after observing $Y$ \\
$g_{k,e}(y)$ & marginal probability that family $k$ is included in $\Gamma_e(y)$ \\
$L_{k,e}(\theta)$ & nuisance-specific false-exclusion probability $\Pr\{k\notin\Gamma_e(Y)\mid k,\theta,e\}$ \\
$\alpha$ & scientific false-exclusion tolerance \\
$\delta$ & protected nuisance-tail mass; $0$ is pointwise validity and $1$ nuisance-average validity \\
$\zeta$ & finite-sample certificate failure probability \\
$J_e(g)$ & expected candidate-set size $\E|\Gamma_e(Y)|$ under the efficiency law \\
$V_e^\star(\alpha,\delta)$ & optimal population RAED value for experiment $e$ at the declared validity level \\
$R_e^\star(\alpha,\delta)$ & optimal normalized RAED resolution $(K-V_e^\star(\alpha,\delta))/(K-1)$ \\
\bottomrule
\end{tabularx}
\end{center}

For $0<\delta\le1$, Tail risk is written through the upper-tail/CVaR functional
\[
\rho_{\delta,\mu_k}(L)
=
\inf_{\eta\in\mathbb R}
\left\{
\eta+\frac1\delta\int(L(\theta)-\eta)_+\,\mu_k(\dd\theta)
\right\},
\]
with the equivalent OCE scaling used in \Cref{sec:learned}. At $\delta=0$, validity is defined directly by $\sup_{\theta\in\Theta_k}L_{k,e}(\theta)\le\alpha$ rather than by taking a singular limit of the Tail formula.

\section{Exact Inclusion-Marginal Representation}
\label{app:marginals}

The expected-size objective and all classwise validity constraints depend only on first-order inclusion marginals. This gives an exact compact representation of the current RAED rule problem.

\begin{lemma}[Measurable whole-set lifting]
\label{lem:measurable-whole-set-lifting}
Let $g:\mathcal Y\to\mathcal P_K$ be measurable, where
\[
\mathcal P_K=\left\{u\in[0,1]^K:\sum_{k=1}^K u_k\ge1\right\}.
\]
Then there exists a measurable Markov kernel $\Pi(\cdot\mid y)$ supported on the nonempty subsets of $\{1,\ldots,K\}$ whose inclusion marginals equal $g(y)$ for every $y$.
\end{lemma}
\begin{proof}
Enumerate the $2^K-1$ nonempty subsets as $S_1,\ldots,S_m$ and let $A\in\{0,1\}^{K\times m}$ contain their incidence vectors as columns. By the polytope argument below, for every $u\in\mathcal P_K$ the compact set
\[
\mathcal W(u)=\{p\in\mathbb R_+^m:\mathbf 1^\top p=1,\ Ap=u\}
\]
is nonempty. Its graph is closed in the finite-dimensional product $\mathcal P_K\times\mathbb R^m$, so a standard measurable-selection theorem supplies a Borel measurable selector $p^\star(u)\in\mathcal W(u)$. Set $\Pi(S_j\mid y)=p_j^\star(g(y))$. The resulting kernel is measurable, supported on nonempty subsets, and satisfies $\Pr_{\Pi(\cdot\mid y)}\{k\in\Gamma\}=g_k(y)$.
\end{proof}

\begin{proposition}[Exact inclusion-marginal representation]
\label{prop:app-marginal-representation}
Fix an experiment $e$ and a finite measurable observation partition $\{B_{e,r}\}_{r=1}^{R_e}$. Let
\[
g_{k,r}=\Pr\{k\in\Gamma_e(Y)\mid Y\in B_{e,r}\}.
\]
A vector $g_{\cdot,r}$ is realizable by a randomized nonempty whole-set rule if and only if
\begin{equation}
0\le g_{k,r}\le1,
\qquad
\sum_{k=1}^K g_{k,r}\ge1.
\label{eq:app-marginal-polytope}
\end{equation}
Moreover, for the RAED validity regimes used in this paper--pointwise, strict Tail, and nuisance-average--the expected-cardinality objective and all familywise exclusion constraints depend on the rule only through these marginals. Consequently the whole-subset and marginal formulations have identical feasible loss fields and identical optimal values.
\end{proposition}

\begin{proof}
Let $x_{r,S}$ be the conditional probability of returning nonempty subset $S$ in response cell $r$. Any whole-set rule induces
\[
g_{k,r}=\sum_{S\ni k}x_{r,S},
\]
so $0\le g_{k,r}\le1$. Since every realized subset is nonempty,
\[
\sum_k g_{k,r}
=
\sum_{S\ne\varnothing}|S|x_{r,S}
\ge1.
\]
The same identity shows that expected set size in the cell is exactly $\sum_k g_{k,r}$. If
\[
p^e_{k,\theta,r}:=P^e_{k,\theta}(B_{e,r}),
\]
then family-$k$ retention under nuisance state $\theta$ is
\[
C_{k,e}(\theta;g)=\sum_{r=1}^{R_e}p^e_{k,\theta,r}g_{k,r},
\]
so every nuisance-specific exclusion loss is a function only of the corresponding marginals. Pointwise constraints, Tail/CVaR functionals of this loss field, and nuisance averages therefore all depend only on $g$.

Conversely, consider the polytope
\[
\mathcal P_K=\left\{g\in[0,1]^K:\sum_k g_k\ge1\right\}.
\]
Every extreme point of $\mathcal P_K$ is a nonzero binary vector. Indeed, two fractional coordinates admit an opposite perturbation preserving feasibility; one fractional coordinate with slack in the sum constraint can be perturbed alone; and if the sum constraint is tight, a single fractional coordinate is impossible because the remaining binary coordinates have integer sum. Hence $\mathcal P_K$ is the convex hull of the incidence vectors of the nonempty subsets of $\{1,\ldots,K\}$. Every feasible marginal vector is therefore a convex combination of nonempty subset incidences and can be realized by a randomized whole-set rule, independently in each response cell. This proves exact equivalence.
\end{proof}

\begin{remark}[Whole-set reconstruction]
The marginal representation fixes expected cardinality and all classwise retention probabilities but generally does not identify a unique distribution over returned subsets. Coupling-sensitive summaries such as singleton frequency, the probability of a particular two-family set, or the variance of $|\Gamma|$ therefore require an explicitly declared whole-set reconstruction. By Carath\'eodory's theorem, at most $K+1$ nonempty subsets are needed in a per-cell reconstruction.
\end{remark}

\section{Proofs for RAED Formulation, Aliasing, and Experiment Comparison}
\label{app:theory-proofs}

\subsection{Common and approximate garbling}
\label{app:proof-common-garbling}
The main text uses only the exact common-garbling monotonicity statement. The quantitative transport and deficiency results are recorded here.

\begin{theorem}[Common and approximate garbling of the RAED frontier]
\label{thm:paper-common-garbling}
Fix the structural weights, admissible nuisance sets, efficiency measures, validity measures when applicable, and a scientific operating point $(\alpha,\delta)$. Let $e_1,e_2$ be two experiments. Suppose a Markov kernel $G$, independent of both $k$ and $\theta$, satisfies
\begin{equation}
\sup_{1\le k\le K}\sup_{\theta\in\Theta_k}
\TV\!\left(P^{e_2}_{k,\theta},P^{e_1}_{k,\theta}G\right)
\le \varepsilon,
\qquad 0\le\varepsilon\le1.
\label{eq:paper-approx-garbling}
\end{equation}
Define
\[
d_k(\theta)=\TV\!\left(P^{e_2}_{k,\theta},P^{e_1}_{k,\theta}G\right),
\qquad
q_G=\TV(Q_{e_2},Q_{e_1}G)\le\varepsilon,
\]
and transport an $e_2$ inclusion rule $g^{(2)}$ to $e_1$ by
\[
(T_Gg^{(2)})_k(y_1)=\int g^{(2)}_k(y_2)G(y_1,\dd y_2).
\]
Then the transported rule is a randomized nonempty rule and, for every $k,\theta$,
\begin{align}
\left|L_{k,e_1}(\theta;T_Gg^{(2)})-L_{k,e_2}(\theta;g^{(2)})\right|
&\le d_k(\theta)\le\varepsilon, \label{eq:paper-loss-transport}\\
\left|J_{e_1}(T_Gg^{(2)})-J_{e_2}(g^{(2)})\right|
&\le (K-1)q_G\le(K-1)\varepsilon. \label{eq:paper-objective-transport}
\end{align}
Consequently, if $\varepsilon\le1-\alpha$, then for every $\delta\in[0,1]$,
\begin{equation}
V_{e_1}^\star(\alpha+\varepsilon,\delta)
\le
V_{e_2}^\star(\alpha,\delta)+(K-1)q_G.
\label{eq:paper-bicriteria-garbling}
\end{equation}
For an exact $\alpha$ comparison at arbitrary $\varepsilon$, set
\[
r_\varepsilon=\min\{1,\alpha+\varepsilon\},
\qquad
\lambda_\varepsilon=\left(1-\frac{\alpha}{r_\varepsilon}\right)_+,
\]
with $\lambda_\varepsilon=0$ when $r_\varepsilon=0$. Mixing the transported rule with the full candidate set with probability $\lambda_\varepsilon$ restores the original validity level and gives
\begin{equation}
V_{e_1}^\star(\alpha,\delta)
\le
(1-\lambda_\varepsilon)
\left[V_{e_2}^\star(\alpha,\delta)+(K-1)q_G\right]
+\lambda_\varepsilon K.
\label{eq:paper-exact-repair-garbling}
\end{equation}
In particular, when $\varepsilon=0$,
\[
V_{e_1}^\star(\alpha,\delta)\le V_{e_2}^\star(\alpha,\delta)
\qquad\text{for every }\delta\in[0,1].
\]
\end{theorem}

The proof is a direct transport argument. Markov-kernel composition preserves the nonempty whole-set decision, bounded expectations move by at most total variation, and conditional candidate-set size lies in $[1,K]$, giving the factor $K-1$ in \eqref{eq:paper-objective-transport}. The final mixture uses the full set as a zero-exclusion safe rule. A complete proof is given in \Cref{app:proof-common-garbling}.

The same comparison can be expressed through one-sided deficiency on the composite parameter space
\[
\Omega=\bigcup_{k=1}^K\{k\}\times\Theta_k,
\qquad
\operatorname{def}(e_1\rightsquigarrow e_2)
=
\inf_G\sup_{(k,\theta)\in\Omega}
\TV\!\left(P^{e_2}_{k,\theta},P^{e_1}_{k,\theta}G\right).
\]

\begin{corollary}[Deficiency continuity of the resolution--validity frontier]
\label{cor:paper-deficiency}
Let $d_{12}=\operatorname{def}(e_1\rightsquigarrow e_2)$. The bounds of \Cref{thm:paper-common-garbling} hold for every $\varepsilon>d_{12}$, and at $\varepsilon=d_{12}$ when the infimum is attained. If $\alpha>0$, the exact-repair comparison \eqref{eq:paper-exact-repair-garbling} also holds at a nonattained $d_{12}$ by passage to the limit. For $\varepsilon>0$, if both directed deficiencies are at most $\varepsilon$, then for every $\delta\in[0,1]$ the symmetric bounds below hold. At the zero-error endpoint $\varepsilon=0$, they also hold when $\alpha>0$; if $\alpha=0$, they require attainment of both zero deficiencies (or another condition that supplies the same zero-error continuity). Under these conditions,
\begin{align}
\left|V_{e_1}^\star(\alpha,\delta)-V_{e_2}^\star(\alpha,\delta)\right|
&\le (K-1)\bigl[(1-\lambda_\varepsilon)\varepsilon+\lambda_\varepsilon\bigr],
\label{eq:paper-deficiency-V}\\
\left|R_{e_1}^\star(\alpha,\delta)-R_{e_2}^\star(\alpha,\delta)\right|
&\le (1-\lambda_\varepsilon)\varepsilon+\lambda_\varepsilon.
\label{eq:paper-deficiency-R}
\end{align}
\end{corollary}

\begin{proof}[Proof of \Cref{thm:paper-common-garbling}]
Let $\Pi_{e_2}$ be a randomized nonempty whole-set rule with inclusion marginals $g^{(2)}$. Under $e_1$, observe $Y_1$, generate $\widetilde Y_2\sim G(Y_1,\cdot)$, and then draw the returned set from $\Pi_{e_2}(\cdot\mid\widetilde Y_2)$. This Markov-kernel composition remains supported on nonempty subsets and has inclusion marginals $T_Gg^{(2)}$.

For each $(k,\theta)$, $g_k^{(2)}$ is bounded in $[0,1]$, so the bounded-expectation total-variation inequality gives
\[
\left|
\E_{P^{e_1}_{k,\theta}G}g_k^{(2)}
-
\E_{P^{e_2}_{k,\theta}}g_k^{(2)}
\right|
\le d_k(\theta).
\]
Since exclusion loss is one minus retention probability, this proves \eqref{eq:paper-loss-transport}.

For the efficiency objective define
\[
f(y_2)=\sum_{k=1}^K g_k^{(2)}(y_2).
\]
Nonemptiness and marginal feasibility imply $1\le f\le K$. The transported efficiency expectation is under $Q_{e_1}G$, whereas the original expectation is under $Q_{e_2}$. Applying the oscillation form of the total-variation inequality therefore yields
\[
\left|
\E_{Q_{e_1}G}f-
\E_{Q_{e_2}}f
\right|
\le(K-1)q_G,
\]
which is \eqref{eq:paper-objective-transport}. Convexity of total variation under the common structural and nuisance mixing used to form $Q_e$ gives $q_G\le\varepsilon$.

If the $e_2$ rule is pointwise valid at level $\alpha$, \eqref{eq:paper-loss-transport} makes its transport pointwise valid at level $\alpha+\varepsilon$. For $0<\delta\le1$, the risk-envelope representation implies
\[
\rho_{\delta,\mu_k}(L_{k,e_1})
\le
\rho_{\delta,\mu_k}(L_{k,e_2})
+
\sup_\theta|L_{k,e_1}(\theta)-L_{k,e_2}(\theta)|
\le\alpha+\varepsilon.
\]
Thus the same transfer holds for strict Tail validity and for the nuisance-average endpoint. Taking an infimizing sequence of $e_2$ rules gives \eqref{eq:paper-bicriteria-garbling}.

For exact repair, every transported exclusion risk is bounded by
$r_\varepsilon=\min\{1,\alpha+\varepsilon\}$. Let $g^{\rm all}\equiv(1,\ldots,1)$ and define
\[
\bar g=(1-\lambda_\varepsilon)T_Gg^{(2)}+\lambda_\varepsilon g^{\rm all}.
\]
All nuisance-specific exclusion losses are multiplied by $1-\lambda_\varepsilon$. Pointwise supremum, mean risk, and the Tail functional are positively homogeneous, and
\[
(1-\lambda_\varepsilon)r_\varepsilon\le\alpha.
\]
Hence $\bar g$ is exactly $\alpha$-valid. Linearity of expected set size gives
\[
J_{e_1}(\bar g)
=(1-\lambda_\varepsilon)J_{e_1}(T_Gg^{(2)})+\lambda_\varepsilon K.
\]
Combining this identity with \eqref{eq:paper-objective-transport} and taking an infimizing sequence proves \eqref{eq:paper-exact-repair-garbling}. The exact common-garbling claim is the case $\varepsilon=0$.
\end{proof}

\begin{proof}[Proof of Corollary~\ref{cor:paper-deficiency}]
The first statement follows from the definition of the infimum: for every $\eta>d_{12}$ there exists a common kernel with uniform error at most $\eta$. If the infimum is attained, use the attaining kernel. If $d_{12}=1$, every kernel has total-variation error at most one.

Suppose $\alpha>0$ and the infimum is not attained. Choose errors $\eta_n\downarrow d_{12}$ and apply the exact-repair bound of \Cref{thm:paper-common-garbling}. Because
\[
r_t=\min\{1,\alpha+t\},
\qquad
\lambda_t=\left(1-\frac{\alpha}{r_t}\right)_+,
\]
are continuous for $t\in[0,1]$ when $\alpha>0$, passage to the limit gives the same exact-repair comparison at $d_{12}$.

For the symmetric bound, first note that for $\alpha>0$ the function
\[
B(t):=(1-\lambda_t)t+\lambda_t
=
\begin{cases}
\dfrac{(1+\alpha)t}{\alpha+t},&0\le t\le1-\alpha,\\[2mm]
\alpha t+1-\alpha,&1-\alpha\le t\le1,
\end{cases}
\]
is continuous and nondecreasing. When $\alpha=0$, $B(t)=1$ for every $t>0$ while the declared exact endpoint has $B(0)=0$; hence no nonattained zero-error continuity is asserted. For $\varepsilon>0$, approximation by errors below any value slightly above the directed deficiency and monotonicity of the positive-error bound suffice; at $\varepsilon=0$ with $\alpha=0$, use the assumed attaining zero-error kernels. Apply the exact-repair comparison in the direction $e_1\rightsquigarrow e_2$ and use $V_{e_2}^\star\ge1$:
\begin{align*}
V_{e_1}^\star-V_{e_2}^\star
&\le
(1-\lambda_\varepsilon)(K-1)\varepsilon
+\lambda_\varepsilon(K-V_{e_2}^\star)\\
&\le
(K-1)\bigl[(1-\lambda_\varepsilon)\varepsilon+\lambda_\varepsilon\bigr].
\end{align*}
The same argument with the experiments reversed bounds the opposite difference, proving \eqref{eq:paper-deficiency-V}. Finally,
\[
R_e^\star=\frac{K-V_e^\star}{K-1},
\]
so division by $K-1$ gives \eqref{eq:paper-deficiency-R}.
\end{proof}

\subsection{Exact cyclic aliasing separation}
\label{app:proof-cyclic-aliasing}
\begin{proof}[Proof of \Cref{thm:paper-aliasing}]
Identify structural and nuisance alphabets with the cyclic group $\mathbb Z_K$. Under $e_I$, define
\begin{equation}
Y=M\oplus Z.
\label{eq:app-cyclic-aliasing}
\end{equation}
For every fixed $Z=z$, the map $M\mapsto M\oplus z$ is a bijection, so $M=Y\ominus z$ is recovered exactly and $I(M;Y\mid Z=z,e_I)=\log K$. A deterministic bijective channel can be garbled into the $e_R$ channel defined below.

After nuisance averaging,
\[
P(Y=j\mid M=k,e_I)
=
\begin{cases}
1-\varepsilon,&j=k,\\
\varepsilon/(K-1),&j\ne k,
\end{cases}
\]
so the structural channel is $\mathsf C_K(\varepsilon)$. Under $e_R$, let the observation law be nuisance-independent:
\[
P(Y=j\mid M=k,Z=z,e_R)
=
\begin{cases}
1-\beta,&j=k,\\
\beta/(K-1),&j\ne k.
\end{cases}
\]
Its marginal channel is $\mathsf C_K(\beta)$. Define
\[
\chi(q)=1-\frac{K}{K-1}q.
\]
Composition of two $K$-ary symmetric channels multiplies their nontrivial eigenvalues $\chi(\cdot)$. Since $0<\chi(\beta)<\chi(\varepsilon)$, choose
\[
\eta=\frac{K-1}{K}\left(1-\frac{\chi(\beta)}{\chi(\varepsilon)}\right)
\in\left(0,\frac{K-1}{K}\right).
\]
Then $\mathsf C_K(\beta)=\mathsf C_K(\varepsilon)\mathsf C_K(\eta)$, proving nuisance-marginal Blackwell dominance. Strict model-index EIG preference follows from the monotonicity of \eqref{eq:paper-kary-mi}, and the average Bayes error is $\varepsilon<\beta$.

Under $e_I$, $Y$ is a deterministic function of $(M,Z)$. Since $M$ is uniform and independent of $Z$, $Y=M\oplus Z$ is uniform on $\mathbb Z_K$, hence
\[
I((M,Z);Y\mid e_I)=H(Y)=\log K.
\]
Under $e_R$, $Y$ is conditionally independent of $Z$ given $M$, so
\[
I((M,Z);Y\mid e_R)=I(M;Y\mid e_R)=\mathcal I_K(\beta)<\log K.
\]

For pointwise validity under $e_I$, fix any $k,y\in\mathbb Z_K$. The nuisance state $z=y\ominus k$ is admissible and makes $Y=y$ with probability one when $M=k$. Thus every pointwise-valid randomized rule must satisfy
\[
g_k(y)\ge1-\alpha\qquad\forall k,y.
\]
Summing over $k$ yields
\[
\sum_{k=1}^K g_k(y)\ge K(1-\alpha),
\]
so $J\ge K(1-\alpha)$ under any efficiency distribution on $Y$. To attain the bound, let $t=K(1-\alpha)$, $m=\lfloor t\rfloor$, and $r=t-m$. Independently of $Y$, return a uniformly random $m$-subset with probability $1-r$ and a uniformly random $(m+1)$-subset with probability $r$. Each label is included with probability
\[
\frac{(1-r)m+r(m+1)}{K}=1-\alpha,
\]
and expected set size is $t$. Hence $V_{e_I}^\star(\alpha,0)=K(1-\alpha)$. Under $e_R$, the singleton rule $\Gamma(Y)=\{Y\}$ has exclusion risk $\beta\le\alpha$ at every nuisance state and achieves $J=1$, proving \eqref{eq:paper-pointwise-alias-values}.

It remains to solve the strict Tail problem under $e_I$. Represent a marginal rule by the matrix $g=(g_k(y))_{k,y\in\mathbb Z_K}$. For $u\in\mathbb Z_K$, define the simultaneous translation
\[
g^{(u)}_k(y)=g_{k\oplus u}(y\oplus u).
\]
The feasible marginal set is convex, and the structural prior, nuisance law, objective, nonemptiness constraint, and each family Tail-risk constraint are invariant under this action. Averaging over $u$ therefore preserves expected size and nonemptiness and, by convexity of $\rho_\delta$, cannot increase Tail risk. The averaged matrix has the form $g_k(k\oplus z)=h_z$. Permuting the $K-1$ nonzero offsets and averaging again yields an optimal rule with a common exclusion probability on all alias states. Write
\[
\ell_0=1-g_k(k),
\qquad
\ell_A=1-g_k(k\oplus z),\quad z\ne0.
\]
Thus $\ell_0$ is the false-exclusion level on the aligned nuisance state and $\ell_A$ is the common false-exclusion level on the alias states. The expected candidate-set size is
\[
K-\bigl[(K-1)\ell_A+\ell_0\bigr],
\]
and mandatory nonemptiness imposes $(K-1)\ell_A+\ell_0\le K-1$.

An optimum may be taken with $\ell_A\ge \ell_0$. Indeed, if $\ell_0>\ell_A$, decrease $\ell_0$ by $t$ and increase each alias exclusion by $t/(K-1)$ until the two values meet. This preserves $(K-1)\ell_A+\ell_0$. While $\ell_0\ge \ell_A$, the aligned atom is the upper loss. If $\delta\le1-\varepsilon$, the upper-tail value decreases at unit rate. If $\delta>1-\varepsilon$, the derivative of the upper-tail numerator is at most
\[
-(1-\varepsilon)+\frac{\varepsilon}{K-1}<0,
\]
because $\varepsilon<(K-1)/K$. Thus the exchange weakly decreases Tail risk without changing the objective.

For $\ell_A\ge \ell_0$, the loss distribution has value $\ell_A$ on total nuisance mass $\varepsilon$ and value $\ell_0$ on mass $1-\varepsilon$, so
\[
\rho_\delta(L)=
\begin{cases}
\ell_A,&\delta\le\varepsilon,\\[1mm]
\dfrac{\varepsilon \ell_A+(\delta-\varepsilon)\ell_0}{\delta},&\delta>\varepsilon.
\end{cases}
\]
When $\delta\le\varepsilon$, validity requires $\ell_A\le\alpha$ and hence $\ell_0\le\alpha$, so total exclusion is at most $K\alpha$ and the pointwise value $K(1-\alpha)$ remains optimal.

For $\delta>\varepsilon$, minimizing candidate-set size is equivalent to maximizing
\[
(K-1)\ell_A+\ell_0
\]
subject to
\[
\varepsilon \ell_A+(\delta-\varepsilon)\ell_0\le\alpha\delta,
\qquad
0\le \ell_0\le \ell_A\le1,
\]
and the nonemptiness cap. Ignoring the order constraint, the objective gained per unit Tail-risk budget is $(K-1)/\varepsilon$ for $\ell_A$ and $1/(\delta-\varepsilon)$ for $\ell_0$. These are equal at
\[
\delta_0=\frac{K\varepsilon}{K-1}.
\]
For $\delta<\delta_0$, the order constraint forces the balanced boundary $\ell_A=\ell_0=\alpha$, giving total exclusion $K\alpha$. For $\delta\ge\delta_0$, exclusion is optimally allocated to the alias states, so $\ell_0=0$ and
\[
\ell_A=\min\!\left(1,\frac{\alpha\delta}{\varepsilon}\right).
\]
The nonemptiness cap produces the floor at one, giving \eqref{eq:paper-tail-alias-value} and \eqref{eq:paper-tail-alias-piecewise}. Singleton resolution is reached exactly when $\alpha\delta/\varepsilon\ge1$, which proves \eqref{eq:paper-tail-alias-iff}. Under $e_R$, the singleton rule has constant exclusion risk $\beta\le\alpha$, so $V_{e_R}^\star(\alpha,\delta)=1$ for all $\delta$.
\end{proof}

\subsection{Scaling consequence}
\begin{corollary}[Linear-in-$K$ separation for every strict Tail level]
\label{cor:paper-k-scaling}
Fix $\alpha\in(0,1)$ and any $\bar\delta\in[0,1)$. For all sufficiently large $K$ satisfying $\alpha<1-1/K$, there exists a two-experiment problem in which nuisance-marginalized structural information, full-latent information, average Bayes classification, and nuisance-marginal Blackwell informativeness all strictly rank $e_I$ above $e_R$, while
\begin{equation}
\frac{V_{e_I}^\star(\alpha,\bar\delta)}
{\min_{e\in\{e_I,e_R\}}V_e^\star(\alpha,\bar\delta)}
\ge
Kc_{\alpha,\bar\delta},
\qquad
c_{\alpha,\bar\delta}
=
\min\!\left\{1-\alpha,\frac{1-\bar\delta}{1+\bar\delta}\right\}>0.
\label{eq:paper-k-scaling}
\end{equation}
\end{corollary}

\begin{proof}[Proof of Corollary~\ref{cor:paper-k-scaling}]
If $\bar\delta=0$, apply the pointwise part of \Cref{thm:paper-aliasing} with, for example, $\varepsilon=\alpha/2$ and $\beta=3\alpha/4$. For all sufficiently large $K$, the theorem's parameter restriction holds and
\[
V_{e_I}^\star(\alpha,0)=K(1-\alpha),
\qquad
V_{e_R}^\star(\alpha,0)=1,
\]
which is \eqref{eq:paper-k-scaling} because $c_{\alpha,0}=1-\alpha$.

Now take $0<\bar\delta<1$ and choose
\[
\varepsilon=\frac{\alpha(1+\bar\delta)}{2},
\qquad
\beta=\frac{\alpha+\varepsilon}{2}.
\]
Then
\[
0<\alpha\bar\delta<\varepsilon<\beta<\alpha.
\]
For sufficiently large $K$, all hypotheses of \Cref{thm:paper-aliasing} hold, all information/classification comparisons in the corollary strictly rank $e_I$ above $e_R$, and $V_{e_R}^\star=1$. The exact Tail formula gives
\[
V_{e_I}^\star(\alpha,\bar\delta)
=
\max\!\left\{1,
\min\!\left[K(1-\alpha),
K-(K-1)\frac{2\bar\delta}{1+\bar\delta}
\right]
\right\}.
\]
Moreover,
\[
K-(K-1)\frac{2\bar\delta}{1+\bar\delta}
\ge
K\frac{1-\bar\delta}{1+\bar\delta}.
\]
Therefore
\[
V_{e_I}^\star(\alpha,\bar\delta)
\ge
K\min\!\left\{1-\alpha,\frac{1-\bar\delta}{1+\bar\delta}\right\},
\]
which proves the result.
\end{proof}

\subsection{Diffuse near-aliases}
\label{app:proof-near-alias}
\begin{theorem}[General family-specific near-alias bound]
\label{thm:app-near-alias-general}
Fix $e$ and $0<\delta\le1$, and write
\[
Q_e=\sum_{r=1}^R\lambda_rQ_r,
\qquad
\lambda_r\ge0,
\qquad
\sum_r\lambda_r=1.
\]
For every family $k$ and component $r$, suppose there is a measurable $A_{k,r}\subseteq\Theta_k$ with
\[
m_{k,r}=\mu_k(A_{k,r})>0,
\qquad
\tau_{k,r}
=\frac1{m_{k,r}}
\int_{A_{k,r}}\TV(P^e_{k,\theta},Q_r)\,\mu_k(\dd\theta).
\]
Then every Tail-RAED feasible rule satisfies
\begin{equation}
\E_{Q_r}|\Gamma_e(Y)|
\ge
\max\!\left\{1,
\sum_{k=1}^K
\left[1-\tau_{k,r}-\alpha\kappa_\delta(m_{k,r})\right]_+
\right\}
\label{eq:app-near-alias-local}
\end{equation}
for every $r$, and hence
\begin{equation}
V_e^\star(\alpha,\delta)
\ge
\sum_{r=1}^R\lambda_r
\max\!\left\{1,
\sum_{k=1}^K
\left[1-\tau_{k,r}-\alpha\kappa_\delta(m_{k,r})\right]_+
\right\}.
\label{eq:app-near-alias-global}
\end{equation}
\end{theorem}

\begin{proof}
Fix $k,r$ and abbreviate $A=A_{k,r}$, $m=m_{k,r}$, and $L(\theta)=L_{k,e}(\theta)$. Tail validity gives $\rho_{\delta,\mu_k}(L)\le\alpha$.

If $m\ge\delta$, the weight $w=\ind_A/m$ is feasible in the Tail risk envelope, because $1/m\le1/\delta$. Hence
\[
\frac1m\int_A L(\theta)\,\mu_k(\dd\theta)\le\alpha.
\]
If $m<\delta$, choose
\[
w(\theta)=\frac1\delta\ind_A(\theta)
+\frac{\delta-m}{\delta(1-m)}\ind_{A^c}(\theta).
\]
This is a feasible envelope weight. Since $L\ge0$,
\[
\alpha\ge\rho_{\delta,\mu_k}(L)
\ge\frac1\delta\int_A L(\theta)\,\mu_k(\dd\theta),
\]
so
\[
\frac1m\int_A L(\theta)\,\mu_k(\dd\theta)
\le\frac{\alpha\delta}{m}.
\]
Both cases are summarized by
\begin{equation}
\frac1m\int_A L(\theta)\,\mu_k(\dd\theta)
\le\alpha\kappa_\delta(m).
\label{eq:app-alias-region-loss}
\end{equation}

Let $g_k(y)=\Pr\{k\in\Gamma_e(Y)\mid Y=y\}\in[0,1]$. Since $1-L(\theta)=\E_{P^e_{k,\theta}}g_k(Y)$, the bounded-expectation total-variation inequality gives
\[
\E_{Q_r}g_k(Y)
\ge
\E_{P^e_{k,\theta}}g_k(Y)-\TV(P^e_{k,\theta},Q_r).
\]
Average over $A$ and use \eqref{eq:app-alias-region-loss} to obtain
\[
\E_{Q_r}g_k(Y)
\ge
1-\alpha\kappa_\delta(m_{k,r})-\tau_{k,r}.
\]
Since $g_k\ge0$, the right-hand side may be replaced by its positive part. Summing over $k$ gives
\[
\E_{Q_r}|\Gamma_e(Y)|
=\sum_k\E_{Q_r}g_k(Y)
\ge
\sum_k\left[1-\tau_{k,r}-\alpha\kappa_\delta(m_{k,r})\right]_+.
\]
Mandatory nonemptiness supplies the independent lower bound one, proving \eqref{eq:app-near-alias-local}. Averaging over the mixture weights $\lambda_r$ proves \eqref{eq:app-near-alias-global}. The uniform specialization reported in \eqref{eq:paper-near-alias-summary} follows because $\kappa_\delta(m)$ is nonincreasing in $m$ and each $\tau_{k,r}\le\tau$.
\end{proof}

\section{Learned RAED and Calibration Proofs}
\label{app:calibration-proofs}

\subsection{Population score geometry supporting the learned family}
\label{app:score-geometry}
Fix $e$ and let $P_{k,\mu}^{e}$ and $Q_e$ be as defined in \eqref{eq:paper-density-ratio-score} and \eqref{eq:efficiency-law}. Assume $P_{k,\mu}^{e}\ll Q_e$ and write $r_{k,e}=\dd P_{k,\mu}^{e}/\dd Q_e$.

\begin{proposition}[Classwise nuisance-average density-ratio rule]
\label{prop:app-density-ratio}
For the classwise relaxation
\begin{equation}
\min_{0\le g_k\le1}\ \E_{Q_e}g_k(Y)
\quad\text{subject to}\quad
\E_{P_{k,\mu}^{e}}g_k(Y)\ge1-\alpha,
\label{eq:app-classwise-np}
\end{equation}
there is a threshold $t_k\in[0,\infty]$ and boundary randomization $\gamma_k(y)\in[0,1]$ supported on $\{r_{k,e}=t_k\}$ such that an optimum is
\begin{equation}
g_k^\star(y)
=\ind\{r_{k,e}(y)>t_k\}
+\gamma_k(y)\ind\{r_{k,e}(y)=t_k\}.
\label{eq:app-np-rule}
\end{equation}
\end{proposition}
\begin{proof}
The objective is the $Q_e$-mass of the retained region and the constraint requires at least $1-\alpha$ mass under $P_{k,\mu}^{e}$. Introducing a nonnegative multiplier $\lambda$ gives the pointwise Lagrangian integrand
\[
[1-\lambda r_{k,e}(y)]g_k(y)
\]
under $Q_e$. It is minimized by $g_k=1$ where $\lambda r_{k,e}>1$, by $g_k=0$ where $\lambda r_{k,e}<1$, and by arbitrary randomization on equality. The multiplier and boundary randomization are chosen so that the coverage constraint is met at equality unless it is slack at the unconstrained optimum. This is the randomized Neyman--Pearson construction.
\end{proof}

For strict Tail validity, let
\begin{equation}
\mathcal W_{k,\delta}
:=\left\{w\in L^\infty(\mu_k):
0\le w\le1/\delta,\ \int w\,\dd\mu_k=1\right\},
\label{eq:app-tail-envelope-weights}
\end{equation}
and define the reweighted predictive law
\begin{equation}
P_{k,w}^{e}(A)
:=\int w(\theta)P_{k,\theta}^{e}(A)\,\mu_k(\dd\theta).
\label{eq:app-reweighted-law}
\end{equation}

\begin{proposition}[Tail validity as robust family retention]
\label{prop:app-tail-robust-retention}
For every measurable marginal rule and every $0<\delta\le1$,
\begin{equation}
\rho_{\delta,\mu_k}(L_{k,e}(\cdot;g))\le\alpha
\quad\Longleftrightarrow\quad
\E_{P_{k,w}^{e}}g_k(Y)\ge1-\alpha
\quad\forall w\in\mathcal W_{k,\delta}.
\label{eq:app-tail-robust-equivalence}
\end{equation}
\end{proposition}
\begin{proof}
For any $w\in\mathcal W_{k,\delta}$,
\[
\int w(\theta)L_{k,e}(\theta;g)\,\mu_k(\dd\theta)
=1-\int g_k(y)P_{k,w}^{e}(\dd y).
\]
Taking the supremum over $w$ on the left is equivalent to taking the infimum of reweighted retention on the right. The risk-envelope definition \eqref{eq:tail-risk} gives the result.
\end{proof}

These two propositions justify the score interpretation used in \Cref{sec:learned-scores}. Under the additional compactness, domination, $L^1(Q_e)$-continuity, $0<\alpha<1$, and strong-duality conditions of \Cref{thm:app-tail-least-favourable-score}, the robust problem attains least-favourable $w_k^\star$ and nonnegative dual multipliers $\lambda_k^\star$, yielding the pointwise coupled oracle
\begin{equation}
g^\star(y)\in
\argmin_{u\in[0,1]^K,\ \sum_k u_k\ge1}
\sum_{k=1}^K\bigl[1-\lambda_k^\star r_{k,w_k^\star}(y)\bigr]u_k,
\label{eq:app-least-favourable-score}
\end{equation}
where $r_{k,w}=\dd P_{k,w}^{e}/\dd Q_e$. The full functional-analytic attainment theorem is retained with its complete conditions in Appendix~\ref{app:additional-theory}; it is supporting oracle geometry rather than a main-paper result.

\subsection{Proof of nested nuisance-average calibration}
\label{app:proof-nested-mean}

\begin{proof}[Proof of the mean clause of
	\Cref{thm:paper-nested-calibration}]
	Condition on $\mathcal F_0$. The trained scores, architecture,
	hyperparameters, nested rule family, threshold domain, and population risk
	curve are then fixed, while the final calibration sample retains its declared
	law.
	
	If every threshold is population-valid, the conclusion is immediate. If the
	calibrated passing set is empty, the procedure deploys the safe endpoint and
	is again valid. It remains to consider the event on which an invalid threshold
	is selected.
	
	First suppose that $\Lambda_j$ is a continuous interval. Define
	\[
	\lambda_j^\star
	:=
	\inf\{\lambda\in\Lambda_j:R_j(\lambda)\le\alpha\}.
	\]
	The safe endpoint makes this set nonempty. Whenever invalid thresholds exist,
	continuity and monotonicity give
	\[
	R_j(\lambda_j^\star)=\alpha.
	\]
	If $\widehat\lambda_j$ were invalid, monotonicity would imply
	$\widehat\lambda_j<\lambda_j^\star$. Let $\mathcal P_j$ denote the passing set
	inside the infimum in \eqref{eq:paper-nested-selection}. Since
	$\widehat\lambda_j=\inf\mathcal P_j<\lambda_j^\star$, there exists
	$\lambda\in\mathcal P_j$ with $\lambda<\lambda_j^\star$. The all-larger
	requirement defining $\mathcal P_j$ therefore forces
	\[
	U_j(\lambda_j^\star;\gamma_j)<\alpha
	=
	R_j(\lambda_j^\star).
	\]
	Thus invalid selection implies failure of the pointwise UCB at the single
	$\mathcal F_0$-measurable population crossing point $\lambda_j^\star$.
	
	Now suppose instead that
	\[
	\Lambda_j=\{\lambda_{j,1}<\cdots<\lambda_{j,m_j}\}
	\]
	is a finite ordered threshold grid. If any grid point is invalid, define the
	last invalid threshold
	\[
	\lambda_j^-
	:=
	\max\{\lambda\in\Lambda_j:R_j(\lambda)>\alpha\}.
	\]
	This quantity is fixed conditional on $\mathcal F_0$. Because the passing set
	is finite, a nonempty passing set has a smallest element and
	$\widehat\lambda_j\in\mathcal P_j$. If $\widehat\lambda_j$ were invalid, then
	monotonicity gives
	\[
	\widehat\lambda_j\le\lambda_j^-.
	\]
	Since membership of $\widehat\lambda_j$ in $\mathcal P_j$ requires the UCB to
	be below $\alpha$ at every grid point at least as conservative as
	$\widehat\lambda_j$, it follows in particular that
	\[
	U_j(\lambda_j^-;\gamma_j)<\alpha
	<
	R_j(\lambda_j^-).
	\]
	Hence invalid selection again implies failure of the pointwise UCB at one
	fixed, population-determined threshold.
	
	In either case,
	\[
	\{R_j(\widehat\lambda_j)>\alpha\}
	\]
	is contained in the failure event of a single pointwise UCB whose conditional
	failure probability is at most $\gamma_j$. Therefore
	\[
	\Pr\!\left\{
	R_j(\widehat\lambda_j)\le\alpha
	\,\middle|\,
	\mathcal F_0
	\right\}
	\ge 1-\gamma_j,
	\]
	which proves \eqref{eq:paper-mean-cal-cert}. For a finite or countable claim
	set, a union bound over $j$ with $\sum_j\gamma_j\le\zeta$ gives the
	simultaneous statement. No union bound over the thresholds along an individual
	nested path is required.
\end{proof}

\subsection{Positive-tail specialization of nested calibration}
\label{app:proof-nested-tail}
\begin{proof}[Proof of the positive-tail specialization]
Define
\[H_{j,\lambda}(\theta):=\delta\eta_j(\lambda)+\bigl(L_{j,\lambda}(\theta)-\eta_j(\lambda)\bigr)_+.\]
Condition again on $\mathcal F_0$. Because the entire map $\lambda\mapsto\eta_j(\lambda)$ is fixed and measurable before the calibration sample is used, $H_{j,\lambda}$ is a fixed bounded loss for every $\mathcal F_0$-measurable threshold. The Rockafellar--Uryasev/OCE representation gives, for every fixed $\lambda$,
\begin{align*}
\rho_{\delta,\mu_k}(L_{j,\lambda})
&=\inf_{\eta\in\R}
\left\{\eta+\frac1\delta\E(L_{j,\lambda}-\eta)_+\right\}\\
&\le
\eta_j(\lambda)
+\frac1\delta\E(L_{j,\lambda}-\eta_j(\lambda))_+\\
&=\frac{\E H_{j,\lambda}}{\delta}.
\end{align*}
Let $U_j^H(\lambda;\gamma_j)$ denote a pointwise upper confidence bound for
$\E H_{j,\lambda}$. Because $0\le L\le1$, $0\le\eta\le1$, and $0<\delta<1$, the transformed loss $H_{j,\lambda}$ lies in $[0,1]$. Hence $U_j^H/\delta$ is a pointwise upper confidence bound for the true Tail risk at every $\mathcal F_0$-measurable fixed threshold. Candidate-set enlargement decreases $L_{j,\lambda}$ pointwise and therefore decreases its Tail risk. By the assumed continuity of that risk path, the proof of the mean clause applies verbatim with $R_j(\lambda)$ replaced by $\rho_{\delta,\mu_k}(L_{j,\lambda})$ and $U_j$ by $U_j^H/\delta$. This proves \eqref{eq:paper-tail-cal-cert}; the simultaneous statement follows from the same countable union bound.
\end{proof}

\subsection{Nested Monte Carlo for nuisance-state loss}
\label{app:nested-mc}
\begin{corollary}[Jensen-conservative nested Monte Carlo]
\label{cor:paper-jensen-mc}
Fix $0<\delta<1$, $\lambda$, and the independently fixed $\eta_j(\lambda)$. Suppose that conditional on nuisance state $\theta$,
\begin{equation}
\widehat L_{j,\lambda}(\theta)
=\frac1R\sum_{r=1}^R I_{r,j,\lambda}(\theta),
\qquad
\E\!\left[\widehat L_{j,\lambda}(\theta)\mid\theta\right]
=L_{j,\lambda}(\theta),
\label{eq:paper-inner-mc}
\end{equation}
where the $I_r\in[0,1]$ are generated from the declared observation-noise law. Define
\begin{equation}
\widehat H_{j,\lambda}(\theta)
:=\delta\eta_j(\lambda)
+\bigl(\widehat L_{j,\lambda}(\theta)-\eta_j(\lambda)\bigr)_+.
\label{eq:paper-inner-h}
\end{equation}
Then
\begin{equation}
\E\!\left[\widehat H_{j,\lambda}(\theta)\mid\theta\right]
\ge H_{j,\lambda}(\theta).
\label{eq:paper-jensen-domination}
\end{equation}
Consequently a valid upper confidence bound on the outer mean of independent nuisance-state observations $\widehat H_i$ is conservative for $\E H$, and therefore for the Tail risk through \eqref{eq:paper-oce-upper}. The nuisance states, rather than the $R$ inner noise replicates, remain the outer calibration sample size.
\end{corollary}

\begin{proof}[Proof of Corollary~\ref{cor:paper-jensen-mc}]
For fixed $\eta\in[0,1]$, the function
\[
h_\eta(z):=\delta\eta+(z-\eta)_+
\]
is convex in $z$. Conditional on $\theta$, Jensen's inequality and unbiasedness of \eqref{eq:paper-inner-mc} give
\[
\E[h_\eta(\widehat L)\mid\theta]
\ge
h_\eta(\E[\widehat L\mid\theta])
=h_\eta(L(\theta)).
\]
This is \eqref{eq:paper-jensen-domination}. Integrating over $\theta\sim\mu_k$ yields
\[
\E_{\theta,\mathrm{noise}}\widehat H_{j,\lambda}(\theta)
\ge
\E_{\theta}H_{j,\lambda}(\theta).
\]
Therefore any valid upper confidence bound $U$ for the mean of the outer bounded observations $\widehat H_i$ also satisfies, on the same event,
\[
\E H_{j,\lambda}\le \E\widehat H_{j,\lambda}\le U.
\]
Combining this with \eqref{eq:paper-oce-upper} gives
\[
\rho_{\delta,\mu_k}(L_{j,\lambda})\le U/\delta.
\]
Independence is required across the outer nuisance-state observations used by the chosen bounded-mean UCB. Multiple inner draws contribute only to one $\widehat H_i$ and therefore do not increase the outer nuisance sample size.
\end{proof}

\subsection{Selection-aware independent certification}
\begin{theorem}[Certification after an independent shortlist]
\label{thm:paper-selected-only}
Let $\mathcal F_0$ contain all randomness used to train rules, choose unprotected hyperparameters, and construct a nonempty random shortlist $\mathcal S\subseteq\cE$. Let $\mathcal F_{\rm cal}$ be generated by a final calibration sample independent of $\mathcal F_0$, and let $\mathcal F_{\rm rank}$ contain any additional randomness used to choose among certified rules. For each realized $\mathcal S$, let $\mathcal I(\mathcal S)$ be a finite or countable, measurably enumerated set of claims that must hold simultaneously. Suppose each claim $j\in\mathcal I(\mathcal S)$ has a conditional failure event $B_j$ satisfying
\begin{equation}
\Pr(B_j\mid\mathcal F_0)\le\gamma_j(\mathcal S)
\quad\text{a.s.},
\qquad
\sum_{j\in\mathcal I(\mathcal S)}\gamma_j(\mathcal S)\le\zeta.
\label{eq:paper-conditional-claims}
\end{equation}
Then, with probability at least $1-\zeta$, every claim in the realized shortlist is valid simultaneously. Hence any measurable deployed experiment $\widehat e\in\mathcal S$ whose final rule is chosen from those certified rules is population-valid, even if the final choice depends on the calibration and ranking data.
\end{theorem}

\begin{proof}[Proof of \Cref{thm:paper-selected-only}]
Condition on $\mathcal F_0$. The realized shortlist $\mathcal S$, the fixed rule families, the finite or countable measurable enumeration of $\mathcal I(\mathcal S)$, and the failure allocations are now fixed. Let
\[
\mathcal A_{\mathcal S}
:=\bigcap_{j\in\mathcal I(\mathcal S)}B_j^c.
\]
The conditional union bound gives
\[
\Pr(\mathcal A_{\mathcal S}^c\mid\mathcal F_0)
\le
\sum_{j\in\mathcal I(\mathcal S)}
\Pr(B_j\mid\mathcal F_0)
\le
\sum_{j\in\mathcal I(\mathcal S)}\gamma_j(\mathcal S)
\le\zeta.
\]
Thus $\Pr(\mathcal A_{\mathcal S}\mid\mathcal F_0)\ge1-\zeta$ almost surely. On $\mathcal A_{\mathcal S}$ every rule from which the later selector can choose is population-valid, so any measurable selection among those rules remains valid even if the choice uses $\mathcal F_{\rm cal}$ or $\mathcal F_{\rm rank}$. Finally the tower property gives
\[
\Pr(\mathcal A_{\mathcal S})
=\E\!\left[\Pr(\mathcal A_{\mathcal S}\mid\mathcal F_0)\right]
\ge1-\zeta.
\]
For selected-only calibration, $\mathcal S$ is a singleton already fixed by $\mathcal F_0$, so the claim set contains only the deployed experiment's still-open family/tail claims. This is a validity statement conditional on the independent selection stage, not a guarantee that the selected experiment is population-optimal. Tail-risk monotonicity gives any larger-$\delta$ claims that are logically implied by a certified smaller-$\delta$ statement without another union-bound term.
\end{proof}

\subsection{Proof of the rare-tail obstruction}
\label{app:proof-rare-tail}
\begin{proof}[Proof of Proposition~\ref{prop:paper-rare-tail}]
Compare the null distribution $P_0(Z=0)=1$ with
\[
P_1(Z=1)=p,\qquad P_1(Z=0)=1-p,
\qquad p=\delta(\alpha+\varepsilon).
\]
Because $p\le\delta$, the worst-$\delta$ Tail risk of $P_1$ is
\[
\rho_\delta(P_1)=\frac{p}{\delta}=\alpha+\varepsilon>\alpha.
\]
Under $P_1$, however, an all-zero calibration sample occurs with probability $(1-p)^n$. Any deterministic procedure that certifies the all-zero sample therefore falsely certifies this just-invalid alternative with probability at least $(1-p)^n$. Requiring that probability to be no larger than $\beta$ proves \eqref{eq:paper-rare-tail-bound}. Taking logarithms gives
\[
n\ge
\frac{\log(1/\beta)}{-\log(1-\delta(\alpha+\varepsilon))},
\]
and $-\log(1-x)\sim x$ as $x\downarrow0$ yields \eqref{eq:paper-rare-tail-rate} as $\varepsilon\downarrow0$ in the small-$\alpha\delta$ regime.
\end{proof}

\section{Additional Theoretical Results}
\label{app:additional-theory}

\subsection{Attained least-favourable Tail-RAED score}
The main text uses only the score interpretation of the robust Tail constraint. For completeness, the population attainment result that supports that interpretation is recorded here.

\begin{theorem}[Attained least-favourable Tail-RAED score under dominated compact predictive families]
\label{thm:app-tail-least-favourable-score}
Fix $0<\delta\le1$ and experiment $e$. Assume for each family $k$ that: (i) $\Theta_k$ is a compact metric space and $\mu_k$ is a Borel probability measure; (ii) $\cY_e$ is standard Borel; (iii) $P_{k,\theta}^e\ll Q_e$ for $\mu_k$-almost every $\theta$; (iv), writing
\[
r_{k,\theta}=\frac{\dd P_{k,\theta}^e}{\dd Q_e},
\]
the map $\theta\mapsto r_{k,\theta}$ is Borel measurable and continuous into $L^1(Q_e)$; and (v) $0<\alpha<1$. Let
\begin{equation}
\mathcal W_{k,\delta}
=\left\{w\in L^\infty(\mu_k):0\le w\le1/\delta,\ \int w\,\dd\mu_k=1\right\},
\label{eq:app-lf-weight-set}
\end{equation}
and define
\begin{equation}
r_{k,w}(y)
=\int w(\theta)r_{k,\theta}(y)\,\mu_k(\dd\theta).
\label{eq:app-lf-ratio}
\end{equation}
Then the full nonempty population Tail-RAED problem has an attained primal optimum, zero Lagrange-duality gap, attained nonnegative family multipliers $\lambda_k^\star$, and attained least-favourable weights $w_k^\star\in\mathcal W_{k,\delta}$. An oracle optimum may be chosen so that, for $Q_e$-almost every $y$,
\begin{equation}
\boxed{
g^\star(y)\in
\argmin_{u\in[0,1]^K,\ \sum_k u_k\ge1}
\sum_{k=1}^K
\left[1-\lambda_k^\star r_{k,w_k^\star}(y)\right]u_k.}
\label{eq:app-tail-pointwise-dual-rule}
\end{equation}
Thus every class with $\lambda_k^\star r_{k,w_k^\star}(y)>1$ is retained; if no coefficient is negative, total inclusion mass one is assigned among maximizers of $\lambda_k^\star r_{k,w_k^\star}(y)$; equality sets permit randomization. Whenever $\lambda_k^\star>0$, $w_k^\star$ may be chosen as a nuisance reweighting that is least favourable for the optimal family-$k$ retention constraint.
\end{theorem}

\begin{proof}
Equip
\[
\mathcal G
=\left\{g\in L^\infty(Q_e)^K:
0\le g_k\le1,\ \sum_k g_k\ge1\right\}
\]
with the product weak-* topology. It is convex, weak-* closed and bounded, hence compact by Banach--Alaoglu. The objective
\[
J_e(g)=\sum_k\int g_k\,\dd Q_e
\]
is weak-* continuous. Each $\mathcal W_{k,\delta}$ is convex and weak-* compact in $L^\infty(\mu_k)$ because positivity, the upper bound, and the normalization constraint are weak-* closed.

For fixed $g_k$, write $h_{k,g}(\theta)=\E_{P_{k,\theta}^e}g_k$. The stated $L^1(Q_e)$ continuity of $r_{k,\theta}$ implies continuity of $\theta\mapsto h_{k,g}(\theta)$, while boundedness gives measurability. Consequently
\[
(w,g_k)\longmapsto
\int w(\theta)h_{k,g}(\theta)\,\mu_k(\dd\theta)
=\E_{Q_e}[r_{k,w}(Y)g_k(Y)]
\]
is separately continuous in the relevant weak-* topologies. The robust retention functional
\[
C_k(g_k)
:=\inf_{w\in\mathcal W_{k,\delta}}
\E_{Q_e}[r_{k,w}g_k]
\]
is therefore concave and upper semicontinuous. The sets $\{g:C_k(g_k)\ge1-\alpha\}$ are closed, so their intersection with compact $\mathcal G$ is a nonempty compact feasible set; the full-set rule is feasible. Hence the primal optimum is attained.

Each constraint function $(1-\alpha)-C_k(g_k)$ is convex and $1$-Lipschitz in the sup norm. The full-set rule is strictly feasible because $C_k(1)=1>1-\alpha$. Standard Slater duality for the finite family of continuous convex inequalities therefore gives zero Lagrange-duality gap and nonnegative optimal multipliers. Multiplier attainment may be restricted to a compact set: the dual value at $\lambda=0$ is at least one since every nonempty rule has $J_e(g)\ge1$, whereas evaluating the Lagrangian at the full-set rule gives $K-\alpha\sum_k\lambda_k$. Thus no maximizing multiplier has $\sum_k\lambda_k>(K-1)/\alpha$, and upper semicontinuity gives attainment on the resulting compact multiplier set.

For fixed $\lambda$, the robust constraints yield the saddle expression
\[
\sum_k\lambda_k(1-\alpha)
+\inf_{g\in\mathcal G}
\sup_{w_1\in\mathcal W_{1,\delta},\ldots,w_K\in\mathcal W_{K,\delta}}
\left\{
J_e(g)-\sum_k\lambda_k\E_{Q_e}[r_{k,w_k}g_k]
\right\}.
\]
The rule and reweighting sets are convex compact and the bracketed functional is bilinear and separately continuous, so Sion's minimax theorem permits interchange of infimum and supremum \cite{sion1958}. The resulting upper-semicontinuous function on the compact product reweighting set attains its supremum at $(w_1^\star,\ldots,w_K^\star)$. With $(\lambda^\star,w^\star)$ fixed, the Lagrangian is the integral of the pointwise linear expression in \eqref{eq:app-tail-pointwise-dual-rule}. Pointwise minimization over the nonempty marginal polytope gives exactly the stated rule, with measurable tie-breaking or randomization on equality sets. Complementary slackness identifies active least-favourable reweightings \cite{rudloff2010}.
\end{proof}

This characterization explains the score target used in \Cref{sec:learned-scores}: strict-tail efficiency is governed by a least-favourable reweighted predictive mixture, while independent calibration can certify any prospectively fixed score family against the desired Tail risk.

\section{Statistical and Computational Protocol}
\label{app:implementation}

\subsection{Learned score systems}
Watt and WCA use one-versus-rest logistic scores on a 256-component RBF Nystr\"om map. The common hyperparameter grid is $C\in\{0.1,1,10\}$ and $\gamma_{\mathrm{RBF}}/\gamma_{\mathrm{RBF},0}\in\{0.5,1,2\}$. Five-fold stratified group cross-validation groups on the outer nuisance-state identity. Training uses exact 50/50 positive/negative balancing, with negatives sampled uniformly across rival families and nuisance states. The RBF bandwidth reference $\gamma_{\mathrm{RBF},0}$ is computed from a deterministic pairwise-distance sample capped at 768 rows. Hyperparameter ties are resolved by mean log loss, then by $C$, then by bandwidth multiplier.

GWAE-Fluvial uses the same 256-component RBF-Nystr\"om architecture for its $K=2$ target. Methane oxidation uses 96 Nystr\"om features with fixed $C=1$, a fixed median-bandwidth multiplier, and a grouped three-fold cross validation. In every benchmark, all observation-noise replicas belonging to one outer nuisance state remain in the same statistical fold.

\subsection{Comparator estimators and tie handling}
For Watt and WCA, the final implementation evaluates every comparator on the independent experiment-selection pool. In whitened observation coordinates, let $z_{ki}$ be selection-pool nuisance state $i$ in family $k$, let $m_{ki,e}$ be its stored response under $e$, and define
\[
q_{ki,e}(y)=\phi_d(y;m_{ki,e},I_d),\qquad
p_{k,e}(y)=\frac{1}{n_k}\sum_{i=1}^{n_k}q_{ki,e}(y),\qquad
p_e(y)=\frac{1}{K}\sum_{k=1}^K p_{k,e}(y).
\]
The supports are balanced and the comparator uses the same uniform structural weighting as the paper experiments. Define the finite source-component weight
\[
\omega_{ki}:=\frac{1}{K n_k}.
\]
For each source state the implementation draws $Y_{kir}\sim q_{ki,e}$ using a criterion-specific random seed fixed before comparison. With $R=32$, model-index and full-latent EIG are respectively
\begin{align}
\widehat I_M(e)
&=\sum_{k,i}\omega_{ki}\frac1R\sum_{r=1}^R
  \left\{\log p_{k,e}(Y_{kir})-\log p_e(Y_{kir})\right\},\\
\widehat I_{M,\Theta}(e)
&=\sum_{k,i}\omega_{ki}\frac1R\sum_{r=1}^R
  \left\{\log q_{ki,e}(Y_{kir})-\log p_e(Y_{kir})\right\}.
\end{align}
Thus the full-latent calculation treats each finite selection-pool nuisance state as a latent mixture component. The forced-classification criterion minimizes
\[
\widehat L_{\rm Bayes}(e)=
\sum_{k,i}\omega_{ki}\frac1R\sum_{r=1}^R
\ind\!\left\{\argmax_{\ell}\,p_{\ell,e}(Y_{kir})\ne k\right\}.
\]
For ordered families $k\ne\ell$, the predictive-discrimination estimate is
\[
\widehat D_{k\to\ell}(e)=
\frac1{n_k}\sum_{i=1}^{n_k}\frac1R\sum_{r=1}^R
\left\{\log p_{k,e}(Y_{kir})-\log p_{\ell,e}(Y_{kir})\right\},
\]
and minimum ordered-pair predictive discrimination maximizes $\min_{k\ne\ell}\widehat D_{k\to\ell}(e)$. Expected elimination uses $R_{\rm elim}=16$ fresh draws per selection state. With $c_d=\chi^2_{d,0.95}$, family $\ell$ is rejected for $y$ when
\[
A_{\ell,e}(y)=\ind\!\left\{\min_j\|y-m_{\ell j,e}\|_2^2>c_d\right\},
\]
and the expected-elimination score is
\[
\widehat N_{\rm elim}(e)=
\sum_{k,i}\omega_{ki}\frac1{R_{\rm elim}}\sum_{r=1}^{R_{\rm elim}}
\left\{\sum_{\ell=1}^K A_{\ell,e}(Y_{kir})-A_{k,e}(Y_{kir})\right\}.
\]
The subtraction prevents exclusion of the true family from being counted as elimination of a rival.

The selection support has 40 states per family for Watt and 100 per family for WCA. The EIG standard error is obtained from the per-source sample variances and finite-mixture weights. If $e_{\widehat I}^\star$ is the literal numerical optimum of an information estimate $\widehat I$, the prospectively defined Monte Carlo indistinguishability set is
\[
\mathcal O_{\rm EIG}^{\rm MC}=\left\{e:\ |\widehat I(e)-\widehat I(e_{\widehat I}^\star)|
\le \max\!\left(10^{-12},2\sqrt{\widehat{\rm SE}(e)^2+\widehat{\rm SE}(e_{\widehat I}^\star)^2}\right)\right\}.
\]
For Bayes error, expected elimination, and minimum ordered-pair discrimination the corresponding absolute-score tolerance is $10^{-12}$. A scalar downstream $J(e_{\rm EIG})$ is evaluated at the numerical optimum, with experiment identifier used only to break an exact score tie.

Methane uses the same comparator definitions on 256 selection nuisance states per family, with eight predictive draws per state. The second GWAE-Fluvial campaign compares $\delta=0.05$ Tail-RAED with an independently ranked $\delta=1$ nuisance-average RAED design on the same new nuisance population; the external comparator suite is not used in this comparison.

\subsection{Threshold calibration and Tail evaluation}
To keep the implementation notation compact, the claim index $j$ is suppressed in $\eta(\lambda)$, and both $j$ and the current threshold $\lambda$ are suppressed in $\widehat L_i$ and $\widehat H_i$. For each fixed score system, family thresholds are nested from the all-family safe endpoint toward increasingly selective rules. For $0<\delta<1$, the complete auxiliary OCE map $\eta(\lambda)$ is constructed only from the independent selection role and fixed before final calibration. At each threshold candidate with selection-role nuisance losses $\ell_1,\ldots,\ell_n$, the implementation minimizes
\[
\eta+\frac{1}{\delta n}\sum_{i=1}^n(\ell_i-\eta)_+
\]
over $\{0,1,\ell_1,\ldots,\ell_n\}$; ties are resolved by the smallest $\eta$, and the exact safe endpoint is assigned $\eta=0$.

For methane and the GWAE positive-tail calibration, each independent nuisance state contributes the transformed bounded loss
\[
\widehat H_i=\delta\eta+(\widehat L_i-\eta)_+.
\]
With $\bar H=n^{-1}\sum_i\widehat H_i$, the bounded-mean Bernoulli--KL upper confidence bound is
\begin{equation}
U_{\rm KL}(\bar H,n,\gamma)
=\sup\left\{q\in[\bar H,1]:
\operatorname{kl}(\bar H\|q)\le\frac{\log(1/\gamma)}{n}\right\},
\label{eq:app-kl-ucb}
\end{equation}
where $\operatorname{kl}$ is binary relative entropy. The Jensen result in Corollary~\ref{cor:paper-jensen-mc} makes calibration based on inner Monte Carlo conservative for the underlying nuisance-specific Tail loss. Methane uses $\gamma=0.05/3$; the GWAE Tail protocol uses $\gamma=0.05/2$.

Watt and WCA use the same nested threshold construction for their held-out empirical Tail analyses, but their late-role nuisance counts are not interpreted as supporting a small-tail population certificate.

\subsection{Selection and held-out firewall}
Score fitting and experiment ranking occur before the final calibration pool is accessed. Once an experiment is selected, its score system and nested threshold path are fixed. Calibration then chooses the deployed threshold on an independent nuisance pool, and the untouched evaluation pool is used only after these decisions are complete. This is the selected-only architecture of \Cref{thm:paper-selected-only}.

The principal role sizes are:
\begin{center}
\scriptsize
\setlength{\tabcolsep}{3pt}
\begin{tabular}{lcccc}
\toprule
& Score training & Ranking/selection & Final calibration & Untouched evaluation \\
\midrule
Watt states/family (draws/state) & 80 (12) & 40 (256) & 40 (256) & 40 (512) \\
WCA states/family (draws/state) & 600 (12) & 100 (256) & 100 (256) & 200 (512) \\
GWAE Tail--average campaign states/family (draws/state) & 1000 (12) & 1000 (256) & 5000 (256) & 5000 (512) \\
Methane states/family (draws/state) & 128 (3) & 256 (64) & 10000 (64) & 5000 (64) \\
\bottomrule
\end{tabular}
\end{center}
The GWAE row refers to the independent second campaign used for the Tail-versus-average comparison. The earlier nuisance-average study in \Cref{sec:gwae-hard-region-results} uses separate populations, including 1000 calibration and 2000 evaluation nuisance states per family. The selection-draw entries for Watt, WCA, and GWAE refer to the learned RAED conditional-loss calculation. Comparator-specific Monte Carlo counts are stated above. Pools are disjoint across the four columns; common random numbers are confined to a role.

\section{Reservoir Benchmark Construction and Reproducibility}
\label{app:benchmarks}

\subsection{Watt}
\label{app:watt-construction}
The Watt benchmark is derived from the hierarchical reservoir-characterization case study of Arnold et al.~\cite{arnold2013}. The source hierarchy combines three fault models (FM1--FM3), three cutoff alternatives, three grid representations, and three top-structure alternatives. The RAED benchmark fixes CO3, G3, and TS2 and retains FM1--FM3 together with the two supplied property representations, OBJECT and PIXEL, and the two relative-permeability groups, RP$_0$ and RP$_1$. The three fault models contain the same named major faults but differ in their traces, extents, face counts, and resulting compartment connections.

\paragraph{Model ontologies and nuisance.}
The resulting ensemble supports four definitions of structural identity: FM ($K=3$), FM$\times$PROP ($K=6$), FM$\times$RP ($K=6$), and FM$\times$PROP$\times$RP ($K=12$). PROP distinguishes the OBJECT and PIXEL porosity/permeability representations, while RP distinguishes the two relative-permeability groups. Variables not included in a given structural identity remain part of the within-family nuisance state. In all four ontologies, seven continuous multipliers vary the transmissibility of the existing named faults,
\[
T_2,T_3,T_5,T_6,T_7,T_8,T_9\sim U[0,1].
\]
These multipliers change fault transmissibility without modifying the underlying fault geometry.

Fifty shared seven-dimensional Latin-hypercube coordinates are crossed with the three fault models, two property representations, and two relative-permeability groups, yielding
\[
50\times3\times2\times2=600
\]
reservoir realizations. The same physical ensemble is therefore reused across the four model ontologies; only the definition of which attributes belong to $M$ and which remain in $\theta$ changes.

\paragraph{Grid, properties, and initialization.}
For G3, the source deck and property arrays define a $112\times30\times40$ model, whereas the geometry files include an additional 41st layer. We therefore retain the common 40-layer domain, giving 134,400 cells. The omitted layer contains only ten active cells, represents approximately $2.7\times10^{-5}$ of the active geometric volume, and contains no wells or retained fault content. The OBJECT permeability field, which is absent from the standard scenario directory, was recovered from the official Watt builder archive.

SATNUM is used as the saturation-function region index, with zero-valued entries assigned to the first valid region. Each relative-permeability group retains its three source SWOF tables. Reservoir realizations are initialized directly from EQUIL using a datum depth of 5200 ft, a datum pressure of 2800 psia, an oil--water contact at 5326 ft, and the source constant dissolved-gas relation. Fault segments extending beyond the retained 40-layer grid are truncated at the model boundary.

\paragraph{Wells and physical controls.}
The common well system contains INJ1--INJ6, B, WELL1, WELL4, WELL6, and WELL9. Ten active well tests are defined in \Cref{tab:watt-control-library}. All tests begin from the same equilibrium state and perturb one designated actuator while the remaining wells continue under their common background controls.

\begin{table}[H]
\centering
\caption{Watt active well-test library.}
\label{tab:watt-control-library}
\scriptsize
\setlength{\tabcolsep}{4pt}
\renewcommand{\arraystretch}{1.12}

\begin{tabularx}{\textwidth}{l l l X l}
\hline
Test & Actuator & Baseline control & Perturbation & Remote pressure wells \\
\hline
INJ1 multirate
& INJ1
& RATE 12500 m$^3$/d
& alternating $\pm15\%$ in four 7-d stages
& WELL1, WELL4, WELL9 \\

INJ4 multirate
& INJ4
& RATE 15000 m$^3$/d
& alternating $\pm15\%$ in four 7-d stages
& WELL1, WELL4, WELL6 \\

INJ4 pulse/falloff
& INJ4
& RATE 15000 m$^3$/d
& $+20\%$ for 10 d; shut 5 d; baseline thereafter
& WELL1, WELL4, WELL6 \\

INJ5 pulse/falloff
& INJ5
& RATE 15000 m$^3$/d
& $+20\%$ for 10 d; shut 5 d; baseline thereafter
& WELL1, WELL4, WELL9 \\

INJ6 periodic
& INJ6
& RATE 15000 m$^3$/d
& alternating $\pm15\%$ in four 7-d stages
& WELL6, WELL4, WELL9 \\

B periodic
& B
& RATE 10000 m$^3$/d
& alternating $\pm15\%$ in four 7-d stages
& WELL4, WELL6, WELL9 \\

WELL1 multirate
& WELL1
& LRAT 8000 m$^3$/d
& alternating $\pm15\%$ in four 7-d stages
& WELL4, WELL6, INJ1 \\

WELL4 multirate
& WELL4
& LRAT 8000 m$^3$/d
& alternating $\pm15\%$ in four 7-d stages
& WELL1, WELL6, B \\

WELL6 multirate
& WELL6
& LRAT 8000 m$^3$/d
& alternating $\pm15\%$ in four 7-d stages
& WELL4, WELL9, INJ6 \\

WELL9 drawdown/build-up
& WELL9
& LRAT 6500 m$^3$/d
& $+15\%$ for 7 d; shut 7 d; baseline thereafter
& WELL4, WELL1, INJ5 \\
\hline
\end{tabularx}
\end{table}

\paragraph{Response and observation noise.}
Each well test produces actuator BHP, realized actuator rate, and BHP at three remote wells at days
\[
1,2,3,5,7,10,14,21,30,45,60,75,90,
\]
giving a 65-component response vector. The constrained-observation studies are obtained from these complete responses by truncating the time horizon and retaining only the selected remote pressure traces.

Observation noise is modeled as independent Gaussian perturbations. For pressure, the standard deviation is the larger of $0.5\%$ of the median absolute response and 5 psia. For rate, it is the larger of $1\%$ of the median absolute response and 50 STB/day. The full ensemble comprises 600 reservoir realizations under ten well tests, giving 6000 OPM Flow responses to day 90.

\subsection{WCA}
\label{app:wca-construction}

The WCA benchmark is based on the multiple-geological-interpretation study of Park et al.~\cite{park2013}. The source study considers three parent training images, TI1--TI3, representing alternative fluvial depositional architectures. We take the parent training image as the structural identity,
$M\in\{\mathrm{TI1},\mathrm{TI2},\mathrm{TI3}\}$.
Each parent defines a family of conditional geological realizations, which are treated as within-family nuisance rather than as separate structural hypotheses.

\paragraph{Conditional geometry and petrophysics.}
Each training image contributes ten conditional facies realizations that honor the source hard data. These categorical models are mapped to a $53\times40\times116$ simulation grid, preserving the full native vertical resolution and the main facies geometry and directional connectivity of the source realizations. The mapping uses overlap-volume categorical remapping. The selected grid was chosen from a set of reduced candidates based on preservation of the facies masks and anisotropy.

For each facies realization, porosity and permeability are varied through spatially correlated random fields mapped to source-derived property distributions. The permeability field is coupled to the porosity field through
\[
Z_K=0.85Z_\phi+\sqrt{1-0.85^2}\,\epsilon,
\]
after which both fields are mapped through their corresponding empirical distributions. Horizontal permeability is isotropic, $K_x=K_y$, with vertical permeability set to $K_z=0.1K_x$. Each of the ten conditional geometries is combined with 100 matched petrophysical realizations, giving 1000 nuisance states per parent training image.

\paragraph{Physical model and wells.}
The simulation domain spans 1609.344 m by 804.672 m and 243.84 m vertically, with the top of the model at 1000 m depth. Facies code 0 is treated as inactive, while codes 1--3 form the active reservoir. All realizations share the same oil--water fluid model and an oil--water contact at 1058 m. Nine wells are placed on a common $3\times3$ lattice; their grid coordinates $(I,J)$ are

\begin{center}
\small
\begin{tabular}{lccc}
\hline
 & West & Centre & East \\
\hline
North  & W01 $(11,8)$  & W02 $(27,8)$  & W03 $(43,8)$ \\
Middle & W04 $(11,20)$ & W05 $(27,20)$ & W06 $(43,20)$ \\
South  & W07 $(11,32)$ & W08 $(27,32)$ & W09 $(43,32)$ \\
\hline
\end{tabular}
\end{center}

Wells are completed only in active cells above the oil--water contact, with every well retaining at least one active completion in every reservoir realization.

\paragraph{Physical controls.}
The ten well tests are listed in \Cref{tab:wca-control-library}. The library includes reciprocal injector--producer pairs along the principal directions, diagonal interference tests, multirate excitation, alternating injection, a balanced multiwell pattern, and a central pulse/recovery test. Unless stated otherwise, the pair tests use an injector rate of 1000 m$^3$/day and a producer controlled at 90 bar BHP.

\begin{table}[H]
\centering
\caption{WCA active well-test library.}
\label{tab:wca-control-library}
\scriptsize
\setlength{\tabcolsep}{4pt}
\renewcommand{\arraystretch}{1.12}

\begin{tabularx}{\textwidth}{l l X}
\hline
ID & Test & Well-test schedule \\
\hline

E01 & West--east
& W04 injects at 1000 m$^3$/d; W06 produces at 90 bar BHP \\

E02 & East--west
& W06 injects at 1000 m$^3$/d; W04 produces at 90 bar BHP \\

E03 & North--south
& W02 injects at 1000 m$^3$/d; W08 produces at 90 bar BHP \\

E04 & South--north
& W08 injects at 1000 m$^3$/d; W02 produces at 90 bar BHP \\

E05 & NW--SE diagonal
& W01 injects at 1000 m$^3$/d; W09 produces at 90 bar BHP \\

E06 & SW--NE diagonal
& W07 injects at 1000 m$^3$/d; W03 produces at 90 bar BHP \\

E07 & West--east multirate
& W04 injects at 500/1000/500 m$^3$/d over three 20-d stages; W06 remains at 90 bar BHP \\

E08 & Alternating interference
& W04 and W06 alternate as 1000 m$^3$/d injectors every 10 d; W05 remains at 90 bar BHP \\

E09 & Multiwell interference
& W01/W07 inject 500 m$^3$/d each; W03/W09 produce 500 m$^3$/d each \\

E10 & Central pulse/recovery
& W05 injects 1000 m$^3$/d for 20 d against four 250 m$^3$/d corner producers, followed by near-shut-in monitoring \\
\hline
\end{tabularx}
\end{table}

\paragraph{Observation and simulation ensemble.}
Each well test is represented by actuator BHP, realized actuator rate, and BHP at three remote wells, sampled at days
\[
1,2,3,5,7,10,14,21,30,40,50,60,
\]
giving a 60-component response vector. The constrained-sensing analysis truncates these responses at 7, 14, 30, or 60 days and retains pressure from $0$--$3$ remote wells. Observation noise is modeled as independent Gaussian perturbations with standard deviations of 0.25 bar for pressure and 20 m$^3$/day for realized rate.

The ensemble contains 3000 reservoir realizations, each simulated under the ten well tests, giving 30,000 OPM Flow responses to day 60. Different observation horizons and monitor selections are obtained directly from these complete responses without rerunning the flow simulations.

\subsection{GWAE-Fluvial}
\label{app:gwae-construction}

The GWAE-Fluvial benchmark builds on the single- and double-channel scenarios of Shishaev, Demyanov, and Arnold~\cite{shishaev2026} and the associated geological prior information reported in that work. We use the supplied reservoir grid and forward-model assets together with the reported channel dimensions and structural ranges to construct a prospective parametric fluvial ensemble. Geological realizations are generated directly from interpretable channel-geometry and petrophysical parameters.

\paragraph{Structural prior and geometry.}
The $K=2$ target distinguishes a reservoir containing one connected spanning channel from one containing two separated spanning channels. In local centerline coordinates, each channel follows
\[
v(u)=v_0+\frac{A}{2}\sin\!\left(\frac{2\pi u}{\lambda}+\phi\right),
\]
where $A$ denotes the total peak-to-trough excursion. The geometric parameter ranges are

\begin{center}
\small
\begin{tabular}{lcc}
\hline
Parameter & M1 & M2 \\
\hline
Width & 300--500 m & 300--500 m per channel \\
Thickness & 10--20 m & 10--20 m per channel \\
Wavelength & 1000--2000 m & 500--1000 m \\
Amplitude & 0--300 m & 500--900 m \\
Orientation & $90^\circ$ & $120^\circ$ \\
\hline
\end{tabular}
\end{center}

For M2, the two companion channels share wavelength, amplitude, orientation, and phase, while their width, thickness, vertical placement, and lateral position may differ. Their relative placement is constrained so that the two channel bodies remain spatially distinct.

Each channel body is required to span at least 80\% of the longitudinal domain, have an elongation of at least 1.5, form a single connected component, contain no closed hole in plan view, and occupy at least 1\% of the reservoir bulk volume. For M2, the two channels must remain separate, with no plan-view overlap and at least 300 m separation between occupied cell centers. The resulting ensemble is therefore the parametric proposal restricted to geometrically feasible one- and two-channel realizations.

\paragraph{Petrophysics and forward model.}
The supplied reservoir grid contains $16\times12\times10=1920$ cells over approximately 2400 m by 1800 m, with depths of about 2416--2478 m. Within active channel cells, porosity and permeability are varied together through a rock-quality variable $q\sim U[0,1]$, mapped to source-derived quantiles of the supplied property fields. The same petrophysical law is used for M1 and M2, while cells outside the channel bodies are inactive.

The geological realizations are simulated with OPM Flow 2026.04 using the supplied grid, fluid model, and well configuration. Horizontal permeability is isotropic, $K_y=K_x$, with $K_z=0.15K_x$. All realizations start from the same equilibrium state. The well layout consists of west and east injector banks separated by three central producers, retaining the supplied well locations and completions.

\paragraph{Physical experiment library.}
All nine well tests last ten days and are observed at the same five report times. The experiment library is listed in \Cref{tab:gwae-control-library}.

\begin{table}[H]
\centering
\caption{GWAE-Fluvial active well-test library. BHP values are absolute pressure; RATE values are surface-volume rates.}
\label{tab:gwae-control-library}
\scriptsize
\setlength{\tabcolsep}{4pt}
\renewcommand{\arraystretch}{1.12}

\begin{tabularx}{\textwidth}{l l X}
\hline
ID & Test & Well-test schedule \\
\hline

E01 & Symmetric banks
& I1--I6 at 260 bar; P1--P3 at 220 bar for 10 d \\

E02 & West bank
& I1--I3 at 260 bar; P1--P3 at 220 bar for 10 d \\

E03 & East bank
& I4--I6 at 260 bar; P1--P3 at 220 bar for 10 d \\

E04 & Cross I1--I6
& I1 and I6 at 260 bar; P1 and P3 at 220 bar \\

E05 & Cross I3--I4
& I3 and I4 at 260 bar; P1 and P3 at 220 bar \\

E06 & Midline pair
& I2 and I5 at 260 bar; P2 at 220 bar \\

E07 & Staged-rate pair
& I2/I5 at RATE 6 for days 0--5 and RATE 12 for days 5--10, with 270-bar limit; P1--P3 at 220 bar \\

E08 & Pulse/falloff
& I2/I5 at 270 bar for 3 d, then shut; P1--P3 remain at 220 bar \\

E09 & Alternating banks
& I1--I3 at 260 bar for days 0--5, followed by I4--I6 at 260 bar for days 5--10; P1--P3 at 220 bar \\
\hline
\end{tabularx}
\end{table}

\paragraph{Response, noise, and statistical roles.}
Each well test produces ten signals at days 1, 3, 5, 7, and 10, giving a 50-component response vector. The recorded quantities are total injection rate for the west and east banks, BHP at I2 and I5, and oil- and water-production rates at P1--P3. Independent Gaussian observation noise with standard deviation 0.5 in the corresponding pressure or rate units is added to each component, and responses are whitened by these fixed standard deviations.

The ensemble contains 12,000 nuisance states per structural family, each simulated under all nine well tests, giving 216,000 OPM Flow responses. Within each family, 1000 states are used for score training, 1000 for independent experiment ranking, 5000 for selected-only calibration, and 5000 for untouched evaluation. The same newly generated population is used to compare Tail-RAED at $\delta=0.05$ with nuisance-average RAED at $\delta=1$. The two criteria are ranked independently, calibrated on their respective selected experiments, and evaluated on the common untouched population. Results are reported in \Cref{sec:gwae-tail-results,app:gwae-results}.

\section{Methane-Oxidation Benchmark Construction and Certification Protocol}
\label{app:methane-construction}

The methane benchmark follows the complete-oxidation model-discrimination study of Pankajakshan et al.~\cite{pankajakshan2023}. The three competing mechanisms are power law (PL), Langmuir--Hinshelwood (LH), and Mars--van Krevelen (MVK). Within each mechanism, the kinetic parameters are sampled from the multivariate Gaussian approximation defined by the fitted parameter vector and covariance matrix reported in the source study, without truncation.

\paragraph{Mechanistic models.}
The three candidate rate laws are
\[
r_{\rm PL}=k_1P\,y_{\rm CH_4},
\]
\[
r_{\rm LH}=
\frac{k_rK_{\rm CH_4}(P y_{\rm CH_4})\sqrt{K_{\rm O_2}P y_{\rm O_2}}}
{\left(1+K_{\rm CH_4}P y_{\rm CH_4}+\sqrt{K_{\rm O_2}P y_{\rm O_2}}\right)^2},
\]
and
\[
r_{\rm MVK}=
\frac{k_1k_2P^2y_{\rm CH_4}y_{\rm O_2}}
{k_1Py_{\rm O_2}+2k_2Py_{\rm CH_4}+(k_1k_2/k_3)P^2y_{\rm CH_4}y_{\rm O_2}}.
\]
The forward model is a steady isothermal plug-flow reactor for
$\mathrm{CH_4}+2\mathrm{O_2}\rightarrow\mathrm{CO_2}+2\mathrm{H_2O}$,
using the pressure-drop, Arrhenius, and adsorption relations of the source model.

\paragraph{Design and observation model.}
The candidate design space is a full $3^4$ factorial over four reactor operating variables:

\begin{center}
\small
\begin{tabular}{ll}
\hline
Variable & Levels \\
\hline
Temperature & 250, 300, 350$^\circ$C \\
Normal feed flow & 20, 25, 30 Nml/min \\
$\mathrm{O_2}/\mathrm{CH_4}$ ratio & 2, 3, 4 \\
Inlet methane fraction & 0.005, 0.015, 0.025 \\
\hline
\end{tabular}
\end{center}

This gives 81 candidate reactor conditions. The observation is the outlet mole-fraction vector
$(y_{\rm CH_4},y_{\rm O_2},y_{\rm CO_2})$
with independent Gaussian noise,
\[
\Sigma_y=
\operatorname{diag}
(0.00043^2,0.00202^2,0.00051^2).
\]

\paragraph{Statistical roles and certification.}
For each structural family, 128 nuisance states are used for score training, 256 for independent experiment selection, 10,000 for selected-only calibration, and 5000 for untouched evaluation. The corresponding numbers of observation-noise draws per nuisance state are 3, 64, 64, and 64. Nuisance states and noise draws are independent across the four statistical roles.

For $0<\delta<1$, the selection sample is also used to fix the auxiliary map $\eta(\lambda)$. For each threshold, $\eta$ is chosen from the observed conditional-loss values together with 0 and 1 to minimize
\[
\eta+\frac{1}{\delta}\operatorname{mean}(L-\eta)_+.
\]
Final calibration then uses the transformed loss
\[
H=\delta\eta+(\widehat L-\eta)_+
\]
together with the Bernoulli--KL upper bound in \eqref{eq:app-kl-ucb}. The familywise failure probability is set to
\[
\gamma_k=\frac{0.05}{3},
\]
giving 95\% joint confidence across PL, LH, and MVK at each fixed positive $\delta$.

The declared operating points are
\[
\delta\in\{0,0.01,0.02,0.05,0.10,0.20,1\}.
\]
The $\delta=0$ analysis is empirical. The guarantees at positive $\delta$ apply separately at each fixed operating point rather than simultaneously across the full $\delta$ grid.

\section{Additional Empirical Results}
\label{app:results}

\subsection{Watt ontology summary}

The Watt ontology analysis was evaluated at
$\alpha=0.05$ over the full declared grid
\[
\delta\in\{0,0.01,0.02,0.05,0.10,0.20,1\}.
\]
The same INJ5 pulse/falloff experiment is selected for all four structural
ontologies at every value of $\delta$. The tables below report the held-out
mean candidate-set size $J$, normalized resolution
$R=(K-J)/(K-1)$, and worst-family held-out exclusion risk
$\widehat\rho_{\delta,\max}$.

For $\delta=0$, $\widehat\rho_{\delta,\max}$ is the largest conditional
false-exclusion rate over the held-out nuisance states. For
$0<\delta<1$, it is the largest familywise empirical upper-$\delta$ Tail
false-exclusion risk, and for $\delta=1$ it is the largest familywise
nuisance-average false-exclusion rate. These are held-out empirical summaries
of the nontrivial learned rules.

\begin{table}[H]
	\centering
	\caption{Held-out Watt ontology results at $(\alpha,\delta)=(0.05,0)$.}
	\label{tab:watt-ontology-delta0}
	\small
	\setlength{\tabcolsep}{7pt}
	\renewcommand{\arraystretch}{1.08}
	\begin{tabular}{lrrrr}
		\hline
		Structural target & $K$ & $J$ & $R$ & $\widehat\rho_{0,\max}$ \\
		\hline
		FM & 3 & 1.80783 & 0.59609 & 0.28125 \\
		FM$\times$PROP & 6 & 1.75653 & 0.84869 & 0.16797 \\
		FM$\times$RP & 6 & 1.84614 & 0.83077 & 0.09766 \\
		FM$\times$PROP$\times$RP & 12 & 1.82743 & 0.92478 & 0.09375 \\
		\hline
	\end{tabular}
\end{table}

\begin{table}[H]
	\centering
	\caption{Held-out Watt ontology results at $(\alpha,\delta)=(0.05,0.01)$.}
	\label{tab:watt-ontology-delta001}
	\small
	\setlength{\tabcolsep}{7pt}
	\renewcommand{\arraystretch}{1.08}
	\begin{tabular}{lrrrr}
		\hline
		Structural target & $K$ & $J$ & $R$ & $\widehat\rho_{0.01,\max}$ \\
		\hline
		FM & 3 & 1.80783 & 0.59609 & 0.28125 \\
		FM$\times$PROP & 6 & 1.75653 & 0.84869 & 0.16797 \\
		FM$\times$RP & 6 & 1.84614 & 0.83077 & 0.09766 \\
		FM$\times$PROP$\times$RP & 12 & 1.82743 & 0.92478 & 0.09375 \\
		\hline
	\end{tabular}
\end{table}

\begin{table}[H]
	\centering
	\caption{Held-out Watt ontology results at $(\alpha,\delta)=(0.05,0.02)$.}
	\label{tab:watt-ontology-delta002}
	\small
	\setlength{\tabcolsep}{7pt}
	\renewcommand{\arraystretch}{1.08}
	\begin{tabular}{lrrrr}
		\hline
		Structural target & $K$ & $J$ & $R$ & $\widehat\rho_{0.02,\max}$ \\
		\hline
		FM & 3 & 1.80783 & 0.59609 & 0.28125 \\
		FM$\times$PROP & 6 & 1.75653 & 0.84869 & 0.16797 \\
		FM$\times$RP & 6 & 1.84614 & 0.83077 & 0.09766 \\
		FM$\times$PROP$\times$RP & 12 & 1.82743 & 0.92478 & 0.09375 \\
		\hline
	\end{tabular}
\end{table}

\begin{table}[H]
	\centering
	\caption{Held-out Watt ontology results at $(\alpha,\delta)=(0.05,0.05)$.}
	\label{tab:watt-ontology-delta005}
	\small
	\setlength{\tabcolsep}{7pt}
	\renewcommand{\arraystretch}{1.08}
	\begin{tabular}{lrrrr}
		\hline
		Structural target & $K$ & $J$ & $R$ & $\widehat\rho_{0.05,\max}$ \\
		\hline
		FM & 3 & 1.76025 & 0.61987 & 0.17969 \\
		FM$\times$PROP & 6 & 1.75653 & 0.84869 & 0.16797 \\
		FM$\times$RP & 6 & 1.84614 & 0.83077 & 0.09766 \\
		FM$\times$PROP$\times$RP & 12 & 1.82743 & 0.92478 & 0.09375 \\
		\hline
	\end{tabular}
\end{table}

\begin{table}[H]
	\centering
	\caption{Held-out Watt ontology results at $(\alpha,\delta)=(0.05,0.10)$.}
	\label{tab:watt-ontology-delta010}
	\small
	\setlength{\tabcolsep}{7pt}
	\renewcommand{\arraystretch}{1.08}
	\begin{tabular}{lrrrr}
		\hline
		Structural target & $K$ & $J$ & $R$ & $\widehat\rho_{0.10,\max}$ \\
		\hline
		FM & 3 & 1.73477 & 0.63262 & 0.11230 \\
		FM$\times$PROP & 6 & 1.72793 & 0.85441 & 0.13672 \\
		FM$\times$RP & 6 & 1.82606 & 0.83479 & 0.08789 \\
		FM$\times$PROP$\times$RP & 12 & 1.82743 & 0.92478 & 0.09375 \\
		\hline
	\end{tabular}
\end{table}

\begin{table}[H]
	\centering
	\caption{Held-out Watt ontology results at $(\alpha,\delta)=(0.05,0.20)$.}
	\label{tab:watt-ontology-delta020}
	\small
	\setlength{\tabcolsep}{7pt}
	\renewcommand{\arraystretch}{1.08}
	\begin{tabular}{lrrrr}
		\hline
		Structural target & $K$ & $J$ & $R$ & $\widehat\rho_{0.20,\max}$ \\
		\hline
		FM & 3 & 1.68870 & 0.65565 & 0.08032 \\
		FM$\times$PROP & 6 & 1.67109 & 0.86578 & 0.09473 \\
		FM$\times$RP & 6 & 1.77010 & 0.84598 & 0.07568 \\
		FM$\times$PROP$\times$RP & 12 & 1.77288 & 0.92974 & 0.07520 \\
		\hline
	\end{tabular}
\end{table}

\begin{table}[H]
	\centering
	\caption{Held-out Watt ontology results at $(\alpha,\delta)=(0.05,1)$.}
	\label{tab:watt-ontology-summary}
	\small
	\setlength{\tabcolsep}{7pt}
	\renewcommand{\arraystretch}{1.08}
	\begin{tabular}{lrrrr}
		\hline
		Structural target & $K$ & $J$ & $R$ & $\widehat\rho_{1,\max}$ \\
		\hline
		FM & 3 & 1.42052 & 0.78974 & 0.05586 \\
		FM$\times$PROP & 6 & 1.43800 & 0.91240 & 0.06709 \\
		FM$\times$RP & 6 & 1.52469 & 0.89506 & 0.05723 \\
		FM$\times$PROP$\times$RP & 12 & 1.59907 & 0.94554 & 0.06055 \\
		\hline
	\end{tabular}
\end{table}

For every positive $\delta$, the available Watt calibration sample was
insufficient to certify a nontrivial population-valid rule, so these tables
describe the held-out empirical frontier rather than the final safe-endpoint
deployment. The $\delta=0$ results are empirical pointwise diagnostics.
Across the full empirical frontier, INJ5 pulse/falloff remains the selected
experiment and the same ontology ordering in normalized resolution is retained:
FM$\times$PROP$\times$RP has the highest $R$, followed by FM$\times$PROP,
FM$\times$RP, and FM. Tightening the nuisance-validity requirement reduces
or leaves unchanged the empirical resolution, but the finer ontologies retain
high normalized resolution throughout the declared $\delta$ range.

\subsection{Complete restricted-observation results at \texorpdfstring{$\delta=0.05$}{delta=0.05}}

\Cref{tab:watt-restricted-complete,tab:wca-restricted-complete} report all 16 observation-horizon and remote-monitor combinations used in the constrained-sensing studies at $\alpha=\delta=0.05$. For each sensing level, $|\mathcal O_{\rm EIG}^{\rm MC}|$ denotes the size of the Monte Carlo uncertainty set around the model-index-EIG optimum. A clean disagreement is recorded when the RAED-selected experiment lies outside this set. The reported EIG result corresponds to the deterministic representative defined in Appendix~\ref{app:implementation}.

For each sensing level, the RAED- and EIG-selected experiments are evaluated using the same score-fitting, threshold-calibration, and independent held-out evaluation procedure. We report
\[
\Delta J = J_{\rm EIG}-J_{\rm RAED},
\]
so that positive values indicate a smaller held-out candidate set for the RAED-selected experiment. We also report the worst-family held-out upper-$5\%$-Tail false-exclusion risks, denoted $\widehat\rho_{\rm R}$ and $\widehat\rho_{\rm E}$ for the RAED- and EIG-selected experiments, respectively. These are empirical held-out risks rather than population-validity certificates.

\begin{table}[H]
	\centering
	\caption{Complete Watt restricted-observation results at $\alpha=\delta=0.05$. Positive $\Delta J$ indicates a smaller held-out candidate set for the RAED-selected experiment. The final two columns report the worst-family held-out upper-$5\%$-Tail false-exclusion risk for the RAED- and EIG-selected experiments.}
	\label{tab:watt-restricted-complete}
	\scriptsize
	\setlength{\tabcolsep}{3.5pt}
	\renewcommand{\arraystretch}{1.08}
	
	\begin{tabular}{rrrcrrrrr}
		\hline
		$T$ & Remote wells & $|\mathcal O_{\rm EIG}^{\rm MC}|$
		& Clean & $J_{\rm RAED}$ & $J_{\rm EIG}$ & $\Delta J$
		& $\widehat\rho_{\rm R}$ & $\widehat\rho_{\rm E}$ \\
		\hline
		7  & 0 & 1 & --  & 2.34551 & 2.34551 &  0.00000 & 0.05957 & 0.05957 \\
		7  & 1 & 2 & --  & 2.26751 & 2.24946 & -0.01805 & 0.05176 & 0.04785 \\
		7  & 2 & 1 & --  & 2.22135 & 2.22135 &  0.00000 & 0.04590 & 0.04590 \\
		7  & 3 & 2 & --  & 2.26528 & 2.28561 &  0.02033 & 0.05371 & 0.06055 \\
		14 & 0 & 1 & --  & 2.35125 & 2.35125 &  0.00000 & 0.05957 & 0.05957 \\
		14 & 1 & 1 & --  & 2.16966 & 2.16966 &  0.00000 & 0.04395 & 0.04395 \\
		14 & 2 & 3 & --  & 2.02651 & 2.02651 &  0.00000 & 0.05176 & 0.05176 \\
		14 & 3 & 2 & --  & 2.13758 & 2.13758 &  0.00000 & 0.05469 & 0.05469 \\
		30 & 0 & 1 & yes & 2.31328 & 2.36720 &  0.05392 & 0.05664 & 0.06445 \\
		30 & 1 & 1 & yes & 2.02734 & 2.04469 &  0.01735 & 0.06934 & 0.08105 \\
		30 & 2 & 2 & --  & 1.90103 & 1.90103 &  0.00000 & 0.05762 & 0.05762 \\
		30 & 3 & 1 & --  & 1.99173 & 1.99173 &  0.00000 & 0.08301 & 0.08301 \\
		90 & 0 & 2 & yes & 2.21790 & 2.26066 &  0.04276 & 0.17383 & 0.08398 \\
		90 & 1 & 3 & --  & 1.78530 & 1.97607 &  0.19077 & 0.09180 & 0.12402 \\
		90 & 2 & 1 & --  & 1.69167 & 1.69167 &  0.00000 & 0.15820 & 0.15820 \\
		90 & 3 & 1 & --  & 1.75649 & 1.75649 &  0.00000 & 0.08594 & 0.08594 \\
		\hline
	\end{tabular}
\end{table}

\begin{table}[H]
	\centering
	\caption{Complete WCA restricted-observation results at $\alpha=\delta=0.05$. Positive $\Delta J$ indicates a smaller held-out candidate set for the RAED-selected experiment. The final two columns report the worst-family held-out upper-$5\%$-Tail false-exclusion risk for the RAED- and EIG-selected experiments.}
	\label{tab:wca-restricted-complete}
	\scriptsize
	\setlength{\tabcolsep}{3.5pt}
	\renewcommand{\arraystretch}{1.08}
	
	\begin{tabular}{rrrcrrrrr}
		\hline
		$T$ & Remote wells & $|\mathcal O_{\rm EIG}^{\rm MC}|$
		& Clean & $J_{\rm RAED}$ & $J_{\rm EIG}$ & $\Delta J$
		& $\widehat\rho_{\rm R}$ & $\widehat\rho_{\rm E}$ \\
		\hline
		7  & 0 & 2  & yes & 2.81695 & 2.81291 & -0.00405 & 0.28770 & 0.16934 \\
		7  & 1 & 4  & yes & 2.56748 & 2.88200 &  0.31452 & 0.20586 & 0.04941 \\
		7  & 2 & 5  & yes & 2.37828 & 2.76330 &  0.38503 & 0.36484 & 0.11133 \\
		7  & 3 & 3  & yes & 2.36770 & 2.47161 &  0.10391 & 0.28301 & 0.21719 \\
		14 & 0 & 5  & yes & 2.79314 & 2.67491 & -0.11823 & 0.20176 & 0.23594 \\
		14 & 1 & 16 & yes & 2.60173 & 2.74852 &  0.14680 & 0.21016 & 0.20898 \\
		14 & 2 & 16 & yes & 2.30893 & 2.60945 &  0.30052 & 0.48906 & 0.21758 \\
		14 & 3 & 6  & --  & 2.34177 & 2.61700 &  0.27522 & 0.25176 & 0.16523 \\
		30 & 0 & 5  & --  & 2.77794 & 2.74348 & -0.03446 & 0.36934 & 0.54375 \\
		30 & 1 & 19 & --  & 2.57873 & 2.59372 &  0.01499 & 0.38984 & 0.14863 \\
		30 & 2 & 21 & --  & 2.52367 & 2.59551 &  0.07184 & 0.24688 & 0.16719 \\
		30 & 3 & 9  & --  & 2.26447 & 2.56853 &  0.30406 & 0.31504 & 0.16367 \\
		60 & 0 & 7  & --  & 2.77431 & 2.83377 &  0.05946 & 0.53281 & 0.13105 \\
		60 & 1 & 26 & --  & 2.54613 & 2.71471 &  0.16858 & 0.43008 & 0.11172 \\
		60 & 2 & 26 & --  & 2.35621 & 2.55004 &  0.19383 & 0.17773 & 0.24863 \\
		60 & 3 & 9  & --  & 2.32793 & 2.34423 &  0.01630 & 0.20723 & 0.27480 \\
		\hline
	\end{tabular}
\end{table}

The held-out Tail risks show that differences in candidate-set size should not be interpreted independently of false exclusion. In Watt, the two clean disagreements at $(T,B)=(30,0)$ and $(30,1)$ combine smaller candidate sets under RAED with lower worst-family held-out Tail risk, whereas the clean disagreement at $(90,0)$ combines a smaller candidate set with higher Tail risk. WCA exhibits a wider range of resolution--risk tradeoffs: several RAED-selected experiments produce substantially smaller candidate sets but also higher held-out Tail exclusion than the corresponding EIG-selected experiment. The restricted-sensing comparison therefore reports both quantities rather than interpreting $\Delta J$ alone as a measure of overall superiority.

None of the restricted-sensing rules in these tables carries a finite-sample population-validity certificate. The thresholds are independently empirically calibrated and the reported risks are measured on the held-out evaluation population.

\begin{table}[H]
	\centering
	\caption{Mean held-out candidate-set size across the 16 restricted-sensing cells at $\alpha=\delta=0.05$. Each comparator entry evaluates the deterministic representative of each criterion under the common learned-RAED downstream procedure.}
	\label{tab:restricted-comparator-means}
	\small
	\setlength{\tabcolsep}{5pt}
	\begin{tabular}{lrr}
		\hline
		Selection criterion & Watt mean $J$ & WCA mean $J$ \\
		\hline
		RAED & 2.09184 & 2.52034 \\
		Model-index EIG & 2.11103 & 2.65773 \\
		Full-latent EIG & 2.32000 & 2.70372 \\
		Bayes forced classification & 2.10709 & 2.69110 \\
		Expected elimination & 2.13803 & 2.68784 \\
		Minimum ordered-pair discrimination & 2.29097 & 2.58491 \\
		\hline
	\end{tabular}
\end{table}

We do not report paired 95\% confidence intervals for $\Delta J$ because the evaluation outputs contain aggregate candidate-set means rather than the nuisance-state-level paired contributions required for such an interval. Recovering those pairs would require rerunning the relevant model-fitting and evaluation procedure. Repeated observation-noise draws within a nuisance state are not independent nuisance samples and therefore cannot substitute for the missing outer-state pairs.

\subsection{Dependence of the selected restricted design on \texorpdfstring{$\delta$}{delta}}

Changing the Tail level $\delta$ can affect not only the calibrated candidate-set rule but also the physical experiment selected by RAED. Across the seven declared operating points
$\delta\in\{0,0.01,0.02,0.05,0.10,0.20,1\}$,
the selected design changes in three of the sixteen Watt sensing levels and ten of the sixteen WCA sensing levels. The affected horizon--monitor combinations are listed below.

\begin{center}
	\small
	\begin{tabularx}{0.96\linewidth}{lX}
		\hline
		Benchmark & Sensing levels with more than one selected design across $\delta$ \\
		\hline
		Watt &
		$(7,1)$, $(7,3)$, $(90,0)$ \\
		
		WCA &
		$(7,1)$, $(7,2)$, $(7,3)$,
		$(14,1)$, $(14,2)$, $(14,3)$,
		$(30,1)$, $(30,2)$,
		$(60,1)$, $(60,2)$ \\
		\hline
	\end{tabularx}
\end{center}

Each pair gives the observation horizon in days and the number of retained remote pressure wells. The stronger dependence in WCA indicates that the preferred sensing strategy is more sensitive there to the amount of nuisance-tail protection imposed by the design criterion. A change in selected design may reflect a different physical well test, a different monitor subset, or both.

\subsection{GWAE-Fluvial Tail-RAED comparison}
\label{app:gwae-results}

The GWAE-Fluvial study compares Tail-RAED at $\delta=0.05$ with nuisance-average RAED at $\delta=1$ on the same independently generated geological population. Both criteria use the same score-training and experiment-ranking populations. Each selected experiment is then calibrated separately and evaluated on the same independent evaluation population.

\paragraph{Independent experiment ranking.}
\Cref{tab:gwae-ranking-complete} reports the ranking of all nine well tests. The displayed $J$ values are computed on the independent ranking population before final selected-only calibration and therefore differ from the deployed evaluation values reported later in \Cref{tab:gwae-evaluation-complete}. Tail-RAED selects E07, the staged-rate pair, by a narrow margin over E01, while nuisance-average RAED selects E01, the symmetric-bank test.

\begin{table}[H]
\centering
\caption{GWAE-Fluvial experiment ranking for Tail-RAED at $\delta=0.05$ and nuisance-average RAED at $\delta=1$.}
\label{tab:gwae-ranking-complete}
\small
\setlength{\tabcolsep}{4pt}
\renewcommand{\arraystretch}{1.08}

\begin{tabular}{llrrrr}
\hline
ID & Test & Tail rank & Tail $J$ & Average rank & Average $J$ \\
\hline
E01 & Symmetric banks    & 2 & 1.557934 & 1 & 1.019932 \\
E02 & West bank          & 5 & 1.631516 & 6 & 1.187480 \\
E03 & East bank          & 3 & 1.612088 & 5 & 1.152689 \\
E04 & Cross I1--I6       & 8 & 1.803926 & 7 & 1.242840 \\
E05 & Cross I3--I4       & 7 & 1.751137 & 8 & 1.255127 \\
E06 & Midline pair       & 9 & 1.969910 & 9 & 1.381551 \\
E07 & Staged-rate pair   & 1 & 1.557348 & 4 & 1.048691 \\
E08 & Pulse/falloff      & 6 & 1.634051 & 3 & 1.047578 \\
E09 & Alternating banks  & 4 & 1.625998 & 2 & 1.025992 \\
\hline
\end{tabular}
\end{table}

\paragraph{Selected-experiment calibration.}
Calibration uses 5000 independent nuisance states per structural family for the experiment selected by each criterion. With familywise failure probability $\gamma_k=0.025$, the resulting fixed-operating-point statement has 95\% joint confidence across M1 and M2. For both selected experiments, calibration produces a nontrivial candidate-set rule:

\begin{center}
\small
\begin{tabular}{llcc}
\hline
Method & Selected experiment & M1 UCB & M2 UCB \\
\hline
Tail-RAED, $\delta=0.05$
& E07 staged-rate pair
& 0.048330 & 0.049103 \\

Nuisance-average, $\delta=1$
& E01 symmetric banks
& 0.025999 & 0.049341 \\
\hline
\end{tabular}
\end{center}

For Tail-RAED, the corresponding empirical calibration Tail losses are $0.015922$ for M1 and $0.016391$ for M2; the certification statement uses the upper confidence bounds reported above.

\paragraph{Common evaluation population.}
\Cref{tab:gwae-evaluation-complete} reports the principal diagnostics on the common 5000-state-per-family evaluation population. Conditional false-exclusion risk is estimated separately for each nuisance state using repeated observation-noise draws, and the worst-5\% statistic averages the 250 highest-risk states within each family.

\begin{table}[H]
\centering
\caption{GWAE-Fluvial evaluation results on the common untouched nuisance population. Worst-5\% is the empirical mean conditional false-exclusion risk among the 5\% highest-risk nuisance states in each family.}
\label{tab:gwae-evaluation-complete}
\small
\setlength{\tabcolsep}{4pt}
\renewcommand{\arraystretch}{1.08}

\begin{tabular}{llrrrrr}
\hline
Family & Method & Mean risk & Worst-5\% & P99 & Correct singleton & Ambiguous \\
\hline
M1 & Tail    & 0.000689 & 0.013781 & 0.000000 & 0.093231 & 0.906080 \\
M1 & Average & 0.022152 & 0.433758 & 0.994160 & 0.907827 & 0.070021 \\
M2 & Tail    & 0.000149 & 0.002977 & 0.000000 & 0.329347 & 0.670504 \\
M2 & Average & 0.046046 & 0.875844 & 1.000000 & 0.929433 & 0.024521 \\
\hline
\end{tabular}
\end{table}

With equal weighting of the two structural families, Tail-RAED gives
$J=1.788292$ and $R=0.211708$, compared with $J=1.047271$ and
$R=0.952729$ for nuisance-average RAED. Because $K=2$, the corresponding
overall ambiguity frequencies are $0.788292$ and $0.047271$. The pooled
empirical false-exclusion risks are $0.000419$ and $0.034099$, respectively.
Tail protection therefore increases expected candidate-set size by
$\Delta J=0.741021$ while sharply reducing false exclusion in the adverse
geological tail.

\paragraph{Hard-region behavior.}
As a separate statewise diagnostic, we first call a nuisance state difficult for a rule when its conditional false-exclusion risk satisfies $r_k(\theta)>0.05$. Under this definition, the two rules divide the evaluation population as follows:

\begin{center}
\small
\setlength{\tabcolsep}{3.5pt}
\begin{tabular}{lrrrr}
\hline
Family & Easy for both & Average hard, Tail protected & Tail hard, Average protected & Hard for both \\
\hline
M1 & 4825 & 170 & 0 & 5 \\
M2 & 4686 & 310 & 0 & 4 \\
\hline
\end{tabular}
\end{center}

Thus almost every state that is difficult for Tail-RAED is also difficult for the nuisance-average rule, whereas a substantially larger set of states is difficult only for the nuisance-average rule.

Within the nuisance-average rule's own worst 5\% of geological states, the difference is especially pronounced. For M1, Tail-RAED returns ambiguity on $98.62\%$ of observation draws and falsely excludes the true family on $1.38\%$, compared with $39.28\%$ ambiguity and $43.38\%$ false exclusion for nuisance-average RAED. For M2, Tail-RAED returns ambiguity on $99.71\%$ of draws and falsely excludes on $0.29\%$, compared with $11.75\%$ ambiguity and $87.58\%$ false exclusion for nuisance-average RAED. This is the behavior summarized in Figure~\ref{fig:gwae-hard-region}.

\subsection{Methane frontier}

The same reactor condition, \texttt{MOX048}, is selected at every declared value of $\delta$. The $\delta=0$ point is an empirical diagnostic only. For each fixed positive $\delta$, calibration gives a familywise positive-tail population guarantee with 95\% joint confidence across PL, LH, and MVK. These guarantees apply separately at each operating point and are not simultaneous across the full $\delta$ grid.

\begin{center}
	\small
	\begin{tabular}{cccc}
		\hline
		$\delta$ & $J$ & $R$ & Worst-family held-out $\widehat{\rho}_{\delta}$ \\
		\hline
		\MethaneFrontierRows
		\hline
	\end{tabular}
\end{center}

Here $R=(K-J)/(K-1)$ is the normalized resolution of the resulting learned rule. At the reference operating point $\delta=0.05$, the familywise calibration upper bounds are $\MethaneUcbPL$, $\MethaneUcbLH$, and $\MethaneUcbMVK$ for PL, LH, and MVK, respectively. Requiring the finite-sample certificate increases the held-out mean candidate-set size by $\MethaneReferenceCost$ relative to the corresponding empirically calibrated rule.

\bibliographystyle{unsrt}
\bibliography{raed_references}

\end{document}